%% file: main.tex
\documentclass{article}

\PassOptionsToPackage{numbers, compress}{natbib}

\usepackage{lifetime}

\renewcommand{\hat}{\widehat}
\renewcommand{\tilde}{\widetilde}
\newcommand{\lossell}{\ell}

    \usepackage[final]{neurips_2021}

\usepackage{adjustbox}
\usepackage[utf8]{inputenc} 
\usepackage[T1]{fontenc}    
\usepackage{hyperref}       
\usepackage{url}            
\usepackage{booktabs}       
\usepackage{amsfonts}       
\usepackage{nicefrac}       
\usepackage{microtype}      
\usepackage{xcolor}         
\usepackage{tabto}
\usepackage{wrapfig}
\usepackage{algorithm}
\usepackage{algorithmic}
\newcommand{\update}{\color{black}}

\title{Mixture Proportion Estimation and PU Learning:\\
A Modern Approach
}

\author{%
  Saurabh Garg$^1$,\, Yifan Wu$^1$,\, Alex Smola$^2$,\, Sivaraman Balakrishnan$^1$,\, Zachary C. Lipton$^1$ \\
  $^1$Carnegie Mellon University \\
  $^2$Amazon Web Services \\
}
\usepackage[colorinlistoftodos,textwidth=3.0cm,textsize=tiny]{todonotes}

\begin{document}

\maketitle

\begin{abstract}
\input{sections/00_abstract}
\end{abstract}

\section{Introduction}
\input{sections/01_introduction}

\vspace{-5pt}
\section{Related Work}
\vspace{-3pt}
\input{sections/02_related}
\section{Problem Setup}
\input{sections/03_setup}
\section{Mixture Proportion Estimation}\label{sec:MPE}
\input{sections/04_MPE}
\section{PU-Learning} \label{sec:classification}
\input{sections/05_classification}

\textbf{\texorpdfstring{(TED)$^n$}: Integrating BBE and CVIR {} {} } 
\input{sections/051_unifying}

\section{Experiments} \label{sec:exp}

\input{sections/06_exp}

\vspace{-3pt}
\section{Conclusion and Future Work}
\vspace{-5pt}
\input{sections/07_conclusion}

\section*{Acknowledgements}
\vspace{-5pt}

We thank anonymous reviewers for their feedback during NeurIPS 2021 review process.  
This material is based on research sponsored by Air Force Research Laboratory (AFRL) 
under agreement number FA8750-19-1-1000. The U.S. Government is authorized to reproduce 
and distribute reprints for Government purposes notwithstanding any copyright notation therein. 
The views and conclusions contained herein are those of the authors 
and should not be interpreted as necessarily representing the official 
policies or endorsements, either expressed or implied, of Air Force Laboratory, DARPA or the U.S. Government. 
SB acknowledges funding from the NSF grants DMS-1713003, DMS-2113684 and CIF-1763734, as well as Amazon AI and a Google Research Scholar Award. 
ZL acknowledges Amazon AI, Salesforce Research, Facebook, UPMC, Abridge, the PwC Center, the Block Center, the Center for Machine Learning and Health, and the CMU Software Engineering Institute (SEI) via Department of Defense contract FA8702-15-D-0002,
for their generous support of ACMI Lab's research on machine learning under distribution shift.
\bibliographystyle{abbrvnat}
\bibliography{lifetime}

\newpage
\appendix

\input{sections/appendix}

\end{document}

%% file: sections/00_abstract.tex
Given only positive examples and unlabeled examples
(from both positive and negative classes),
we might hope nevertheless to estimate 
an accurate positive-versus-negative classifier.
Formally, this task is broken down into two subtasks:
(i) \emph{Mixture Proportion Estimation} (MPE)---determining
the fraction of positive examples in the unlabeled data;
and (ii) \emph{PU-learning}---given such an estimate,
learning the desired positive-versus-negative classifier.
Unfortunately, classical methods for both problems
break down in high-dimensional settings.
Meanwhile, recently proposed heuristics
lack theoretical coherence 
and depend precariously on hyperparameter tuning.
In this paper, we propose two simple techniques:
\emph{Best Bin Estimation} (BBE) (for MPE);
and \emph{Conditional Value Ignoring Risk} (CVIR),
a simple objective for PU-learning. 
Both methods dominate previous approaches empirically,
and for BBE, we establish formal guarantees 
that hold whenever we can train a model
to cleanly separate out a small subset of positive examples.
Our final algorithm (TED)$^n$, 
alternates between the two procedures,
significantly improving both our 
mixture proportion estimator and classifier\footnote{Code 
is available at 
\href{https://github.com/acmi-lab/PU\_learning}
{https://github.com/acmi-lab/PU\_learning}}.

%% file: sections/01_introduction.tex
When deploying $k$-way classifiers in the wild,
what can we do when confronted 
with data from a previously unseen class ($k+1$)?
Theory dictates that 
learning under distribution shift 
is impossible absent assumptions.
And yet people appear 
to exhibit this capability routinely.
Faced with new surprising symptoms,
doctors can recognize the presence
of a previously unseen ailment
and attempt to estimate its prevalence. 
Similarly, naturalists can 
discover new species, 
estimate their range and population,
and recognize them reliably going forward.

To begin making this problem tractable,
we might make the label shift assumption
\citep{saerens2002adjusting, storkey2009training, lipton2018detecting},
i.e., that while the class balance $p(y)$ can change,
the class conditional distributions $p(x|y)$ do not. 
Moreover, we might begin by focusing on the base case, 
where only one class has been seen previously, i.e., $k=1$.
Here, we possess (labeled) positive data from the source distribution,
and (unlabeled) data from the target distribution,
consisting of both positive and negative instances. 
This problem has been studied in the literature
as \emph{learning from positive and unlabeled data}
\cite{de1999positive,letouzey2000learning}
and has typically been broken down into two subtasks:
(i) {Mixture Proportion Estimation} (MPE) 
where we estimate $\alpha$, 
the fraction of positives
among the unlabeled examples;
and (ii) {PU-learning} where this estimate
is incorporated into a scheme for learning
a Positive-versus-Negative (PvN) binary classifier.

Traditionally, MPE and PU-learning have been motivated
by settings involving large databases
where unlabeled examples are abundant
and a small fraction of the total positives
have been extracted. 
For example, medical records might be annotated
indicating the presence of certain diagnoses,
but the unmarked passages 
are not necessarily negative. 
This setup has also been motivated
by protein and gene identification 
\cite{elkan2008learning}. 
Databases in molecular biology 
often contain lists of molecules 
known to exhibit some characteristic of interest.
However, many other molecules may
exhibit the desired characteristic,
even if this remains unknown to science.

Many methods have been proposed for both MPE 
\cite{elkan2008learning,du2014class,scott2015rate,ramaswamy2016mixture,jain2016nonparametric,bekker2018estimating,pmlr-v98-reeve19a,ivanov2019dedpul}
and PU-learning~\citep{du2014analysis,du2015convex,kiryo2017positive}. 
However, classical MPE methods 
break down in high-dimensional settings~\citep{ramaswamy2016mixture}
or yield estimators whose accuracy depends 
on restrictive conditions~\citep{du2014class,scott2015rate}.
On the other hand,
most recent proposals either
lack theoretical coherence,
rely on heroic assumptions, 
or depend precariously 
on tuning hyperparameters 
that are, by the very problem setting, untunable.
For PU learning, \citet{elkan2008learning}
suggest training a classifier 
to distinguish positive from unlabeled data 
followed by a rescaling procedure.
\citet{du2015convex} suggest an unbiased 
risk estimation framework for PU learning. 
However, these methods fail badly when applied 
with model classes capable of overfitting
and thus implementations on high-dimensional 
datasets rely on extensive hyperparameter tuning
and additional ad-hoc heuristics   
that do not transport effectively across datasets.

\begin{figure*}[t!]
    \centering 
    \subfigure{\includegraphics[width=0.46\linewidth]{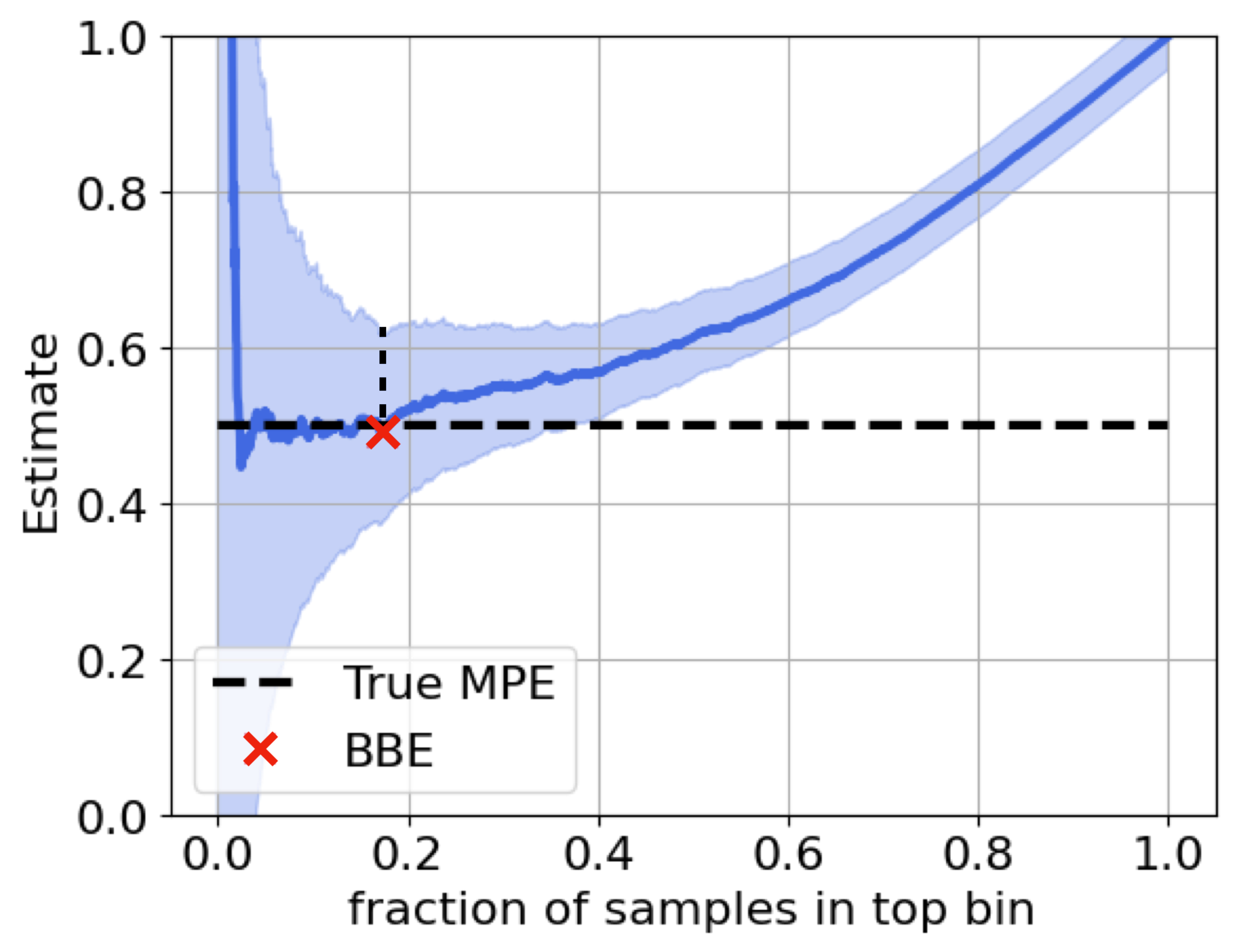}}\hfill
    \subfigure{\includegraphics[width=0.50\linewidth]{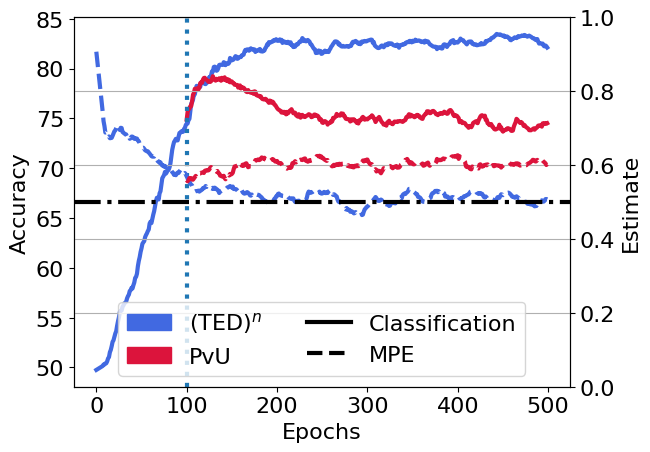}}
    \caption{ \emph{Illustration of proposed methods.} 
    \textbf{(left)} Estimate of $\bf{\alpha}$ with varying fraction 
    of unlabeled examples in the top bin. 
    The shaded region highlights the upper and lower confidence bounds. 
    BBE selects the top bin that minimizes the upper confidence bound. 
    \textbf{(right)} Accuracy and MPE estimate as training proceeds.  
    Till $100$-th epoch (vertical line), we perform PvU training, i.e.,
    warm start for (TED)$^n$. 
    Post $100$-th epoch, we continue 
    with both (TED)$^n$ and PvU training. 
    Note that (TED)$^n$ improves 
    both classification accuracy
    and MPE compared to PvU training. 
    Results with Resnet-18 on binary-CIFAR. 
    For details and comparisons with other methods,
    see \secref{sec:exp}.
    }
    \label{fig:intro}
  \end{figure*}

In this paper, we propose 
(i) Best Bin Estimation (BBE), 
an effective technique for MPE 
that produces consistent estimates $\widehat{\alpha}$
under mild assumptions
and admits finite-sample statistical guarantees 
achieving the desired $O(1/\sqrt{n})$ rates;
and (ii) learning with the 
Conditional Value Ignoring Risk (CVIR) objective,
which discards the highest loss $\hat{\alpha}$
fraction of examples on each training epoch,
removing the incentive to overfit to
the unlabeled positive examples. 
Both methods are simple to implement,
compatible with arbitrary hypothesis classes
(including deep networks),
and dominate existing methods
in our experimental evaluation. 
Finally, we combine the two in an iterated
Transform-Estimate-Discard (TED)$^n$
framework that significantly improves
both MPE estimation error and classifier error. 

We build on label shift methods
\cite{lipton2018detecting, azizzadenesheli2019regularized, alexandari2019adapting, rabanser2018failing, garg2020unified},
that leverage black-box classifiers
to reduce dimensionality,
estimating the target label distribution
as a functional of source and target
push-forward distributions.
While label shift methods rely on classifiers 
trained to separate previously seen classes,
BBE is able to exploit a Positive-versus-Unlabeled (PvU) target classifier,
which gives each input a score indicating 
how likely it is to be a positive sample.
In particular, BBE identifies a threshold
such that by estimating the ratio between
the fractions of positive and unlabeled 
points receiving scores above the threshold,
we obtain
the mixture proportion $\alpha$.

BBE works because in practice, 
for many datasets, PvU classifiers,
even when uncalibrated, produce outputs
with near monotonic calibration diagrams.
Higher scores correspond 
to a higher proportion of positives,
and when the positive data 
contains a separable sub-domain, i.e.,
a region of the input space 
where only the positive distribution has support, 
classifiers often exhibit a threshold
above which the \emph{top bin}
contains mostly positive examples.
We show that the existence 
of a (nearly) pure top bin
is sufficient for BBE to produce
a (nearly) consistent estimate $\hat{\alpha}$,
whose finite sample convergence rates
depend on the fraction of examples in the bin
and whose bias depends on the \emph{purity} of the bin.
Crucially, we can estimate 
the optimal threshold from data.

We conduct a battery of experiments
both to empirically validate our claim
that BBE's assumptions are mild
and frequently hold in practice,
and to establish the outperformance 
of BBE, CVIR, and (TED)$^n$ over 
the previous state of the art.
We first motivate BBE by demonstrating
that in practice PvU classifiers 
tend to isolate a reasonably large,
reasonably pure top bin.
We then conduct extensive experiments 
on semi-synthetic data, adapting 
a variety of binary classification 
datasets to the PU learning setup
and demonstrating the superior 
performance of BBE and PU-learning
with the CVIR objective.
Moreover, we show that (TED)$^n$,
which combines the two in an iterative fashion,
improves significantly over previous methods 
across several architectures 
on a range of image and text datasets.

%% file: sections/02_related.tex
Research on MPE and PU learning date to \citep{denis1998pac,de1999positive,letouzey2000learning}
(see review by \cite{bekker2020learning}). 
\citet{elkan2008learning} first proposed 
to leverage a PvU classifier
%
to estimate the mixture proportion.
\citet{du2014semi} 
propose a different method 
for estimating the mixture 
coefficient based on Pearson divergence minimization. 
While they do not require a PvU classifier,
they suffer the same shortcoming. 
Both methods require that the positive 
and negative examples have disjoint support.
Our requirements are considerably milder.
%
\citet{blanchard2010semi} observe
that without assumptions on the underlying 
positive and negative distributions,
the mixture proportion is not identifiable.
Furthermore, \cite{blanchard2010semi} 
provide an \emph{irreducibility} condition
that identifies $\alpha$ 
and propose an estimator that converges 
to the true $\alpha$.
While their estimator can converge arbitrarily slowly, 
\citet{scott2015rate} showed faster convergence ($\calO(1/\sqrt{n})$)
under stronger conditions.
Unfortunately, despite its appealing theoretical properties
\citet{blanchard2010semi}'s estimator
is computationally infeasible. 
Building on \citet{blanchard2010semi},
\citet{sanderson2014class} and \citet{scott2015rate}
proposed estimating the mixture proportion 
from a ROC curve constructed 
for the PvU classifier. 
However, when the PvU classifier is not perfect, 
these methods are not clearly understood.  
%
%
\citet{ramaswamy2016mixture} proposed
the first computationally feasible 
algorithm for MPE
with convergence guarantees to the true proportion.   
Their method KM, requires 
embedding distributions onto an RKHS. 
However, their estimator underperforms 
on high dimensional datasets and 
scales poorly with large datasets. 
\citet{bekker2018estimating} proposed TIcE,
hoping to identify a positive subdomain 
in the input space using decision tree induction. 
This method also underperforms
in high-dimensional settings. 

In the most similar works,
\citet{jain2016nonparametric} and \citet{ivanov2019dedpul}
explore dimensionality reduction using a PvU classifier. 
Both methods estimate $\alpha$ 
through a procedure operating
on the PvU classifier's output.
However, neither methods 
has provided theoretical backing. 
\cite{ivanov2019dedpul} concede
that their method often fails 
and returns a zero estimate, 
requiring that they fall back 
to a different estimator.
Moreover while both papers state that their method 
require the Bayes-optimal PvU classifier 
to identify $\alpha$ in the transformed space, 
we prove that even when hypothesis class 
is well specified for PvN learning, 
PvU training can fail to recover
the Bayes-optimal scoring function.  
Furthermore, we also show that the heuristic estimator 
in \citet{scott2015rate} can be thought of as using 
PvU classifier for dimensionality reduction. 
While this heuristic is similar to 
our estimator in spirit, we show 
that the functional form of their estimator 
is different from ours and note 
that their heuristic 
enjoys no theoretical guarantee. 
By contrast, our estimator BBE 
is theoretically coherent under mild conditions 
and outperforms all of these methods empirically. 

Given $\alpha$, \citet{elkan2008learning} propose
a transformation via Bayes rule 
to obtain the PvN classifier. 
They also propose 
a weighted objective, 
with weights given 
by the PvU classifier. 
%
Other propose unbiased risk estimators~\citep{du2014analysis,du2015convex} 
which require the mixture proportion $\alpha$.
\citet{du2014analysis} proposed an unbiased 
estimator with non-convex loss functions
satisfying a specific symmetric condition, 
and subsequently \citet{du2015convex} 
generalized it to convex loss functions
(denoted uPU in our experiments). 
in our experiments. 
Noting the problem of overfitting 
in modern overparameterized models, 
\citet{kiryo2017positive} propose a 
regularized extension that clips 
the loss on unlabeled data to zero. 
This is considered the current 
state-of-the-art in PU literature 
(denoted nnPU in our experiments). 
More recently, \citet{ivanov2019dedpul} 
proposed DEDPUL, 
which finetunes the PvU classifiers
using several heuristics, Bayes rule, 
and Expectation Maximization (EM).
Since their method only applies a  
post-processing procedure, 
they rely on a good domain discriminator
classifier in the first place 
and several hyperparameters for their heuristics. 
%
%
%
Several classical methods attempt 
to learn weights that identify
reliable negative examples
\citep{liu2002partially, li2003learning, lee2003learning, liu2003building, zhang2005simple}. 
However, these earlier methods have not been 
successful with modern deep learning models. 

%% file: sections/03_setup.tex
By $\enorm{\cdot}$ and $\inner{\cdot}{\cdot}$,
we denote the Euclidean norm and inner product, respectively.
For a vector $v\in \Real^d$, 
we use $v_j$ to denote its $j^\text{th}$ entry, 
and for an event $E$, we let $\indict{E}$ 
denote the binary indicator of the event.
By $\abs{A}$, we denote the cardinality of set $A$.  
Let $\inpt \in \Real^d $ be the input space 
and $\out = \{-1, +1\}$ be the output space. 
Let $\Prob : \inpt\times\out \to [0,1]$
be the underlying joint distribution 
and let $p$ denote its corresponding density.

Let $\ProbP$ and $\ProbN$ be the 
class-conditional distributions for 
positive and negative class and 
$\pp(x) = p(x|y=+1)$ and $\pn(x) = p(x|y=-1)$ be
the corresponding class-conditional densities. 
$\ProbU$ denotes the distribution 
of the unlabeled data and $\pu$ denotes its density. 
Let $\alpha \in [0,1]$ be 
the fraction of positives 
among the unlabeled population,
i.e., $\ProbU = \alpha \ProbP + (1-\alpha)\ProbN$. 
When learning from positive and unlabeled data, 
we obtain i.i.d. samples from 
the positive (class-conditional) distribution, 
which we denote as
$X_p = \{x_1, x_2, \ldots, x_{n_p}\} \sim \ProbP^{n_p}$
and i.i.d samples from unlabeled distribution as 
$X_u = \{x_{n_p+1}, x_{n_p+2}, \ldots, x_{n_p + n_u}\} \sim \ProbU^{n_u}$. 

MPE is the problem of estimating $\alpha$.  
Absent assumptions on $\ProbP$, $\ProbN$ and $\ProbU$, 
the mixture proportion $\alpha$ is not 
identifiable~\citep{blanchard2010semi}. 
Indeed,  if $\ProbU  = \alpha \ProbP + (1-\alpha) \ProbN$, 
then any alternate decomposition of the form 
$\ProbU  = (\alpha - \gamma) \ProbP + (1-\alpha + \gamma) \ProbN^\prime$, 
for $\gamma \in [0, \alpha)$ and 
$\ProbN^\prime = (1 - \alpha + \gamma)^{-1}(\gamma \ProbP + (1-\alpha)\ProbN)$, 
is also valid. 
Since we do not observe samples 
from the distribution $\ProbN$, 
the parameter $\alpha$ is not identifiable.
\citet{blanchard2010semi} formulate an 
\emph{irreducibility} condition 
under which $\alpha$ is identifiable.
Intuitively, the condition restricts $\ProbN$ 
to ensure that it can not be a (non-trivial) mixture 
of $\ProbP$ and any other distribution. 
While this irreducibility condition 
makes $\alpha$ identifiable, 
in the worst-case, the parameter $\alpha$ 
can be difficult to estimate 
and any estimator must suffer 
an arbitrarily slow rate of convergence~\citep{blanchard2010semi}. 
In this paper, we propose mild conditions on the 
PvU classifier that make $\alpha$ identifiable 
and allows us to derive finite-sample 
convergence guarantees.  

With PU learning, the aim is to learn 
a classifier $f: \inpt \to [0,1]$ 
to approximate $p(y=+1|x)$. 
We assume that we are given a loss function 
$\lossell: [0 ,1] \times \out \to \Real $,
such that $\ell(z, y)$ is the loss incurred 
by predicting $z$ when the true label is $y$. 
For a classifier $f$ and a sampled set 
$X = \{x_1, x_2, \ldots, x_n\}$,  
we let $\wh L^+(f; X) = \sum_{i=1}^n \ell(f(x_i), +1) / n$ 
denote the loss when predicting the samples as positive 
and  $\wh L^-(f; X) = \sum_{i=1}^n \ell(f(x_i), -1) / n$ 
the loss when predicting the samples as negative. 
For a sample set $X$ each with true label $y$,
we define 0-1 error as 
$\wh \calE^{y}(f; X) = \sum_{i=1}^n \indict{ y(f(x_i) - t) \le 0 } / n$
for some predefined threshold $t$ . 
Unless stated otherwise, 
the threshold is assumed to be $0.5$. 

%% file: sections/04_MPE.tex

\setlength{\textfloatsep}{12pt}
\begin{algorithm}[t]
  \caption{Best Bin Estimation (BBE)}
  \label{alg:MPE_PU}
  \begin{algorithmic}[1]
  \INPUT: Validation positive ($X_p$) and unlabeled ($X_u$) samples.
  Blackbox model classifier $\hat{\f}: \calX \to [0,1]$. 
  Hyperparameter $0< \delta,\gamma <1$.
    \STATE $Z_p, Z_u = f(X_p), f(X_u)$. 
    \STATE $\wh q_u (z), \wh q_p(z) = \frac{\sum_{z_i \in  Z_p} \indict{z_i \ge z}}{n_p}, \frac{\sum_{z_i \in  Z_u} \indict{z_i \ge z}}{n_u}$ for all $z \in [0,1]$. 
    \STATE {\update Estimate $\wh c \defeq \argmin_{c \in [0,1]} \left( \frac{\wh q_u(c)}{\wh q_p(c)}  + \frac{1 + \gamma}{\wh q_p(c)}\left( \sqrt{\frac{\log(4/\delta)}{2 n_u}} + \sqrt{\frac{\log(4/\delta)}{2n_p}}\right) \right)\,$. }
    \OUTPUT: { \update $\wh \alpha \defeq \frac{\wh q_u(\wh c)}{\wh q_p(\wh c)}$ }
\end{algorithmic}
\end{algorithm}

In this section, we introduce BBE, 
a new method that leverages 
a blackbox classifier $f$ to perform MPE 
and establish convergence guarantees. 
All proofs are relegated to \appref{ap:proof_mpe}. 
To begin, we assume access 
to a fixed classifier $f$.
For intuition, you may think of $f$ 
as a PvU classifer trained on some
portion fo the positive and unlabeled examples. 
In \secref{sec:classification}, 
we discuss other ways 
to obtain a suitable classifier from PU data.  

We now introduce some additional notation. 
Assume $f$ transforms an input $x\in\inpt$ 
to $z\in [0,1]$, i.e., $z = f(x)$.   
For given probability density function $p$ 
and a classifier $f$, 
define a function $q(z) = \int_{ A_z} p(x) dx$, 
where $A_z = \{x \in \inpt: f(x) \ge z\}$ for all $z\in [0,1]$. 
Intuitively, $q(z)$ captures the 
cumulative density of points in a top bin, 
the proportion of input domain 
that is assigned a value larger than $z$ 
by the classifier $f$ in the transformed space. 
We now define an empirical estimator $\wh q(z)$
given a set $X = \{x_1, x_2, \ldots, x_n\}$
sampled iid from $p(x)$. Let $Z = f(X)$. 
Define $ \wh q(z) = \sum_{i=1}^n \indict{z_i \ge z}/{n}$.
For each pdf $\pp$, $\pn$ and $\pu$, 
we define $\qp$, $\qn$ and $\qu$ respectively.

Without any assumptions on the underlying distribution 
and the classifier $f$, we aim to estimate 
$\alpha^* = \min_{c \in [0,1]} \qu(c)/\qp(c)$ with BBE. 
Later, under one mild assumption 
that empirically holds across numerous PU datasets, 
we show that $\alpha^* = \alpha$, i.e., $\alpha^*$ matches 
the true mixture proportion $\alpha$. 

Our procedure proceeds as follows: 
First, given a held-out dataset of positive ($X_p$) 
and unlabeled examples ($X_u$),
we push all examples through the classifier $f$
to obtain one-dimensional outputs 
$Z_p = f(X_p)$ and $Z_u = f(X_u)$. 
Next, with $Z_p$ and $Z_u$, 
we estimate $\wh q_p$ and $\wh q_u$. 
{\update 
Finally, we return the ratio 
$\wh q_u (\wh c) / \wh q_p (\wh c)$ at $\wh c$ 
that minimizes the upper confidence bound
(calculated using~\lemref{lem:ucb})
at a pre-specified level $\delta$ 
and a fixed parameter $\gamma \in (0,1)$.  
Our method is summarized in \algoref{alg:MPE_PU}.
For theoretical guarantees, we multiply 
the confidence bound term with $1+\gamma$ for 
a small positive constant $\gamma$. 
Refer to \appref{app:gamma_disc} for details.}
We now show that the proposed estimator 
comes with the following guarantee: 

{\update 
\begin{theorem}\label{thm:main_MPE}
Define $c^* = \argmin_{c \in [0,1]} \qu(c)/\qp(c)$. For $\min(n_p, n_u) \ge \frac{2\log(4/\delta)}{q_p(c^*)}$ and
for every $\delta >0$, 
the mixture proportion estimator $\wh \alpha$ 
defined in Algorithm~\ref{alg:MPE_PU} 
satisfies with probability $1-\delta$:  
\begin{align*}    
  \abs{ \wh \alpha - \alpha^*}  \le  \frac{c}{q_p(c^*)}\left( \sqrt{\frac{\log(4/\delta)}{n_u}} + \sqrt{\frac{\log(4/\delta)}{n_p}}\right)  
   \,,
\end{align*}
for some constant $c\ge0$.
\end{theorem}
}

\begin{figure*}[t!]
  \centering 
  \subfigure[]{\includegraphics[width=0.33\linewidth]{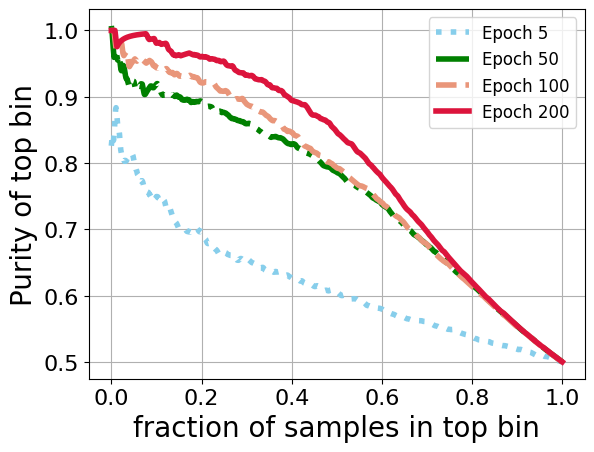}}\hfill
  \subfigure[]{\includegraphics[width=0.65\linewidth]{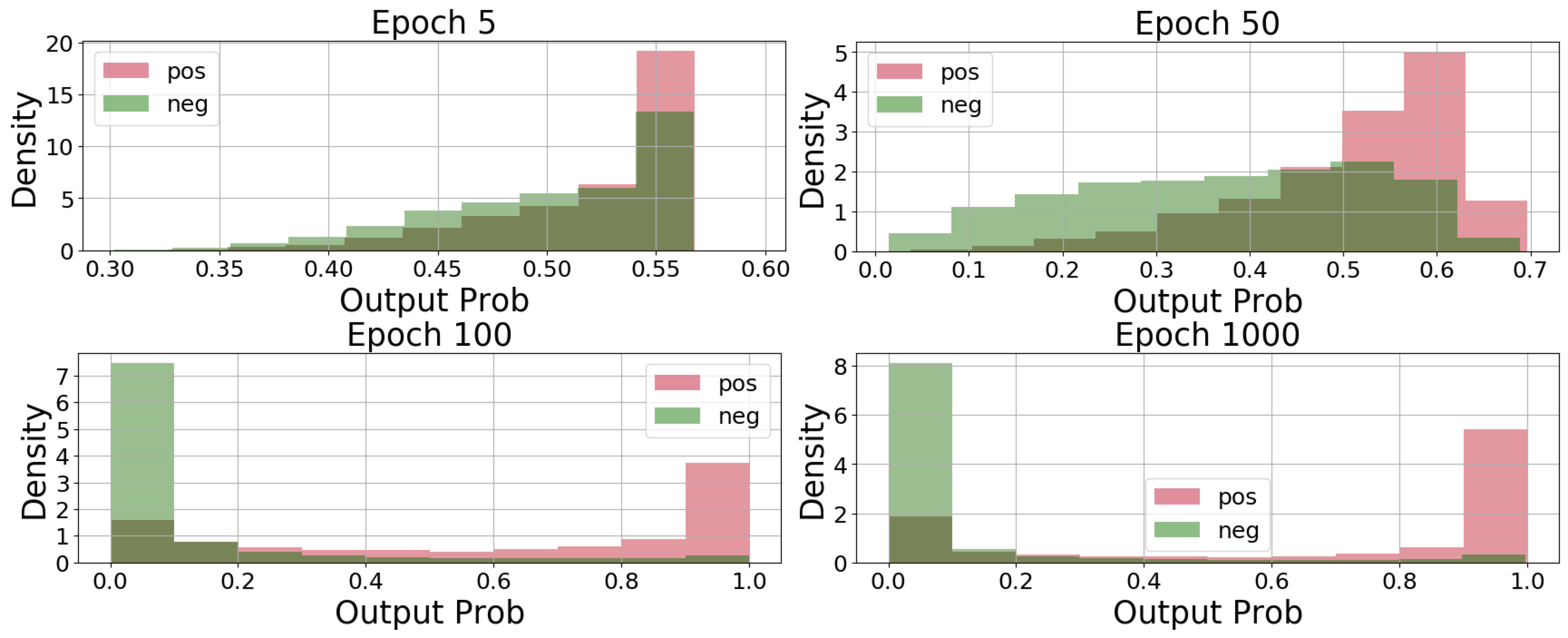}}
  \vspace{-5pt}
  \caption{(a) Purity and size (in terms of fraction of unlabeled samples) 
  in the top bin and
  (b) Distribution of predicted probabilities (of being positive)
  for unlabeled training data 
  as training proceeds  with (TED)$^n$. 
  Results with ResNet-18 on binary-CIFAR. 
  As in \figref{fig:intro},  we fix $W$ at $100$. 
  In \appref{ap:overfitting_unlabeled}, we show that as PvU training proceeds, 
  the purity of top bin degrades and the distribution 
  of predicted probabilities of positives and negatives 
  become less and less separable. 
  }
  \label{fig:loss_bin_property}
\end{figure*}

\thmref{thm:main_MPE} shows 
that with high probability,
our estimate is close to $\alpha^*$. 
The proof of the theorem is based on 
the following confidence bound inequality: 
\begin{lemma}\label{lem:ucb}
For every $\delta >0$, with probability at least $1-\delta$, 
we have for all $c \in [0,1]$
\begin{align*}
    \abs{ \frac{\wh q_u(c)}{ \wh q_p(c)} -  \frac{q_u(c)}{q_p(c)}} \le \frac{1}{\wh q_p(c)}\left( \sqrt{\frac{\log(4/\delta)}{2n_u}} + \frac{q_u(c)}{q_p(c)}\sqrt{\frac{\log(4/\delta)}{2n_p}}\right) \,.
\end{align*}
\end{lemma}  

Now, we discuss the convergence of our estimator 
to the true mixture proportion $\alpha$.
Since, $p_u(x) = \alpha p_p(x) + (1-\alpha) p_n(x)$, for all $x\in \calX$, 
we have  $q_u(z) = \alpha q_p(z) + (1-\alpha) q_n(z)$, for all $z \in [0,1]$. 

{\update 
\begin{corollary} \label{corollary:MPE_final}
Define $c^* = \argmin_{c \in [0,1]} \qn(c)/\qp(c)$. Assume $\min(n_p, n_u) \ge \frac{2\log(4/\delta)}{q_p(c^*)}$.
For every $\delta >0$, $\wh \alpha$ (in Algorithm~\ref{alg:MPE_PU})
satisfies with probability $1-\delta$:  
\begin{align*}    
  &\alpha  - \frac{c_1}{q_p(c^*)}\left( \sqrt{\frac{\log(4/\delta)}{n_u}} + \sqrt{\frac{\log(4/\delta)}{n_p}} \right)  \le \wh \alpha \,, \text{and} \\
  &\wh \alpha \le  \alpha + (1-\alpha)\frac{q_n(c^*)}{q_p(c^*)} +\frac{c_2}{q_p(c^*)} \left( \sqrt{\frac{\log(4/\delta)}{n_u}}  + \sqrt{\frac{\log(4/\delta)}{n_p}} \right) \,,
\end{align*}
for some constant $c_1, c_2\ge0$.
\end{corollary}
}

As a corollary to \thmref{thm:main_MPE}, 
we show that our estimator $\wh \alpha$ converges 
to the true $\alpha$ with convergence rate $\min(n_p,n_u)^{-1/2}$, 
as long as there exist a threshold $c_f \in (0,1)$ 
such that $q_p(c_f) \ge \epsilon_p$ and $q_n(c_f) = 0$ 
for some constant $\epsilon_p >0$. 
We refer to this condition as the
\emph{pure positive bin} property.

%
Note that in a more general case, 
our bound in \corollaryref{corollary:MPE_final} 
captures the tradeoff
due to the proportion of 
negative examples in the top bin (bias) 
versus the proportion of positives
in the top bin (variance). 

\textbf{Empirical Validation {} {}}
We now empirically validate 
the positive pure top bin property
(\figref{fig:loss_bin_property}). 
We observe that as PvU training proceeds, 
purity of the top bin improves 
for a fixed fraction of samples in the top bin. 
Moreover, this behavior 
becomes more pronounced 
when learning a PvU classifier 
with the CVIR objective 
proposed in the following section.

\textbf{Comparison with existing methods {} {}}
Due to the intractability of \citet{blanchard2010semi} 
estimator, \citet{scott2015rate} implements a heuristic  
based on identifying a point on the AUC curve such that 
the slope of the line segment between this point 
and (1,1) is minimized.
While this approach is similar 
in spirit to our BBE method, 
there are some striking differences. 
First, the heuristic estimator in \citet{scott2015rate}
provides no theoretical guarantees,
whereas we provide guarantees 
that BBE will converge to the best estimate 
achievable over all choices of the bin size
and provide consistent estimates 
whenever a pure top bin exists.
Second, while both estimates involve thresholds, 
the functional form of the estimates are different.
Corroborating theoretical results of BBE, 
we observe that the choices in BBE create 
substantial differences in 
the empirical performance as observed 
in \appref{app:comparison}. 
We work out details of comparison between 
\citet{scott2015rate} heuristic 
and BBE in \appref{app:comparison}.

On the other hand, recent 
works~\citep{jain2016nonparametric,ivanov2019dedpul}
that use PvU classifier for dimensionality reduction,  
discuss Bayes optimality 
of the PvU classifier 
(or its one-to-one mapping)
as a sufficient condition 
to preserve $\alpha$ in 
transformed space. 
By contrast, we show that the milder 
pure positive bin property 
is sufficient to guarantee consistency 
and achieve $\calO(1/\sqrt{n})$ rates. 
Furthermore, in a simple toy setup in \appref{ap:failure},
we show that even when the hypothesis class
is well specified for PvN learning,
it will not in general contain
the Bayes optimal PvU classifier
and thus PvU training 
will not recover the 
Bayes-optimal scoring function, 
even in population.
Contrarily, we show that   
any monotonic mapping 
of the Bayes-optimal PvU scoring function 
induces a positive pure top bin property. 
%
%
%
We leave further theoretical investigations
concerning conditions under which
a pure positive top bin arises to future work.

%% file: sections/05_classification.tex
\begin{algorithm}[t!]
  \caption{PU learning with Conditional Value Ignoring Risk (CVIR) objective }
  \label{alg:classification_PU}
  \begin{algorithmic}[1]
  \INPUT: Labeled positive training data ($X_p$) and unlabeled training samples ($X_u$). Mixture proportion estimate $\alpha$. 
    \STATE Initialize a training model $f_\theta$ and an stochastic optimization algorithm $\calA$. 
    \STATE $X_n \defeq X_u$.
    \WHILE{training error $\wh \calE^+(f_\theta; X_p) + \wh\calE^-(f_\theta; X_n)$ is not converged} 
        \STATE Rank samples $x_u \in X_u$ according to their loss values $\lossell( f_\theta(x_u),-1)$. 
        \STATE $X_n \defeq X_{u, 1-\alpha}$ where $X_{u, 1-\alpha}$ denote the lowest ranked $1- \alpha$ fraction of samples. 
        \STATE{Shuffle $(X_p, X_n)$ into $B$ mini-batches. With $(X_p^i, X_n^i)$ we denote $i$-th mini-batch}.
        \FOR{ $i=1$ to $B$}
            \STATE Set the gradient $\grad_\theta  \left[ \alpha \cdot \wh L^+(f_\theta; X_p^i) + (1-\alpha) \cdot \wh L^-(f_\theta; X_n^i) \right]$ and update $\theta$ with algo. $\calA$. 
        \ENDFOR
    \ENDWHILE
  \OUTPUT: Trained classifier $f_\theta$ 
\end{algorithmic}
\end{algorithm}

Given positive and unlabeled data, 
we hope not only to identify $\alpha$,
but also to obtain a classifier 
that distinguishes effectively
between positive and negative samples. 
In supervised learning with separable data 
(e.g., cleanly labeled image data), 
overparameterized models generalize well 
even after achieving near-zero training error. 
However, with PvU training
over-parameterized models can
memorize the unlabeled positives,
assigning them confidently 
to the negative class,
which can severely hurt
generalization on PN data~\citep{zhang2016understanding}.
Moreover, while unbiased losses exist
that estimate the PvN loss given PU data 
and the mixture proportion $\alpha$,
this unbiasedness only holds 
before the loss is optimized,
and becomes ineffective 
with powerful deep learning models 
capable of memorization.

A variety of heuristics,
including ad-hoc early stopping criteria,
have been explored~\cite{ivanov2019dedpul},
where training proceeds 
until the loss on unseen PU data
ceases to decrease. 
However, this approach 
leads to severe under-fitting 
(results in \appref{ap:early_stopping}).
On the other hand, by regularizing the loss function, 
nnPU~\citet{kiryo2017positive} mitigates
overfitting issues due to memorization. 

However, we observe that nnPU 
still leaves a substantial accuracy gap 
when compared to a model trained just on the 
positive and negative (from the unlabeled) data 
(ref. experiment in \appref{ap:PN}).  
This leads us to ask the following question: 
\emph{can we improve performance 
over nnPU of a model 
just trained with PU data  
and bridge this gap?} 
In an ideal scenario, if we could 
identify and remove all the positive
points from the unlabeled data 
during training
then we can hope to achieve improved
performance over nnPU. 
Indeed, in practice, we observe that 
in the initial stages of PvU training,
the model assigns much higher scores 
to positives than to negatives 
in the unlabeled data 
(\figref{fig:loss_bin_property}(b)). 

%
%
Inspired by this observation, we propose CVIR,  
a simple yet effective objective for PU learning. 
Below, we present our method assuming 
an access to the true MPE.
Later, we combine BBE with CVIR optimization,
yielding (TED)$^n$, 
an alternating optimization
that significantly improves 
both the BBE estimates and the PvU classifier.

Given a training set of positives $X_p$ and 
unlabeled $X_u$ and the mixture proportion $\alpha$, 
we begin by ranking the unlabeled data according 
the predicted probability (of being positive) 
by our classifier. 
Then, in every epoch of training,
we create a (temporary) set 
of provisionally negative samples $X_n$
by removing $\alpha$ fraction
of the unlabeled samples
currently scored as most positive.
Next, we update our classifier
by minimize the loss on the positives $X_p$
and provisional negatives $X_n$
by treating them as negatives.
We repeat this procedure until 
the training error on $X_p$ and $X_n$ converges.  
Likewise nnPU, note that this procedure 
does not need early stopping.
Summary in \algoref{alg:classification_PU}.  

In \appref{ap:anlysis_CVuO}, 
we justify our loss function in the scenario 
when the positives and 
negatives are separable.
For a more general scenario, 
we show that each step of our alternating 
procedure in CVIR cannot increase the 
population loss and hence, 
CVIR can only improve (or plateau)  
after every iteration. 

%% file: sections/051_unifying.tex
We are now ready to present our 
algorithm Transfer, Estimate and Discard (TED)$^n$ 
that combines BBE and CVIR objective.

\begin{algorithm}[t!]
  \caption{Transform-Estimate-Discard (TED)$^n$ }
  \label{alg:classification_MPE_PU}
  \begin{algorithmic}[1]
  \INPUT: Positive data ($X_p$) and unlabeled samples ($X_u$). Hyperparameter $W, \delta$. 
    \STATE Initialize a training model $f_\theta$ and an stochastic optimization algorithm $\calA$.
    \STATE Randomly split positive and unlabeled data into training $X^1_p, X^1_u$ and hold-out set ($X^2_p, X^2_u$).
    \STATE $X^1_n\defeq X^1_u$.

    \COMMENT {// Warm start with domain discrimination training}
    \FOR{ $i=1$ to $W$} 
        \STATE{Shuffle $(X_p^1, X_n^1)$ into $B$ mini-batches. With $({X^1_p}^i, {X^1_n}^i)$ we denote $i$-th mini-batch}. 
        \FOR{ $i=1$ to $B$} 
            \STATE Set the gradient $\grad_\theta \left[ \wh L^+(f_\theta; {X^1_p}^i) + \wh L^-(f_\theta; {X^1_n}^i) \right]$ and update $\theta$ with algorithm $\calA$.
        \ENDFOR
    \ENDFOR
    \WHILE{training error $\wh \calE^+(f_\theta; X^1_p) + \wh\calE^-(f_\theta; X^1_n)$ is not converged}
        \STATE Estimate $\wh \alpha$ using \algoref{alg:MPE_PU} with $(X^2_p,X^2_u)$ and $f_\theta$ as input.
        \STATE Rank samples $x_u \in X^1_u$ according to their loss values $l(f_\theta(x_u),-1)$.
        \STATE $X^1_n \defeq X^1_{u, 1-\wh \alpha}$ where $X^1_{u, 1-\wh \alpha}$ denote the lowest ranked $1- \wh \alpha$ fraction of samples.
        \STATE Train model $f_\theta$ for one epoch on $(X^1_p,X^1_n)$ as in Lines 4-7.
    \ENDWHILE
  \OUTPUT: Trained classifier $f_\theta$ 
\end{algorithmic}
\end{algorithm}

First, we observe the interaction  
between BBE and CVIR objective. 
If we have an accurate mixture 
proportion estimate, then 
it leads to improved classifier, 
in particular, we reject 
accurate number of
prospective positive
samples 
from unlabeled.  
Consequently, updating the 
classifier to minimize loss on 
positive versus retained unlabeled 
improves purity of top bin. 
%
This leads to an 
obvious alternating procedure
where at each epoch, we first use BBE 
to estimate $\wh \alpha$ 
and then update the classifier 
with CVIR objective with 
$\wh \alpha$ as input.  
We repeat this until training error 
has not converged. 
Our method is summarized 
in \algoref{alg:classification_MPE_PU}. 

Note that we need 
to warm start with PvU (positive versus negative) 
training, since in the 
initial stages mixture proportion
estimate is often close to 1 
rejecting all the 
unlabeled examples.
However, in next section,
we show that our procedure is not sensitive 
to the choice of number of warm start epochs
and in a few cases with large datasets, 
we can even get away without warm start 
(i.e., $W=0$) without hurting 
the performance. 
Moreover, recall that our aim is to 
distinguish positive versus negative
examples among the unlabeled set  
where the proportion of positives
is determined by the true mixture proportion $\alpha$.
However, unlike CVIR, we do not re-weight the losses 
in (TED)$^n$. While true MPE $\alpha$ is unknown, one 
natural choice is to use the estimate 
$\wh \alpha$ at each iteration. 
However, in our initial experiments, we observed that 
re-weighted objective 
with estimate $\wh \alpha$ led to  
comparatively poor classification performance 
due to presence of bias in estimate $\wh \alpha$ 
in the initial iterations. 
We note that for deep neural networks 
(for which model mis-specification 
is seldom a prominent concern) 
and when the underlying classes are separable 
(as with most image datasets), 
it is known that importance weighting has little to 
no effect on the final classifier~\citep{byrd2019effect}. 
Therefore, we may not need importance-reweighting 
with (TED)$^n$ on separable datasets.
Consequently, following earlier 
works~\citep{kiryo2017positive,du2015convex} 
we do not re-weight the loss with our (TED)$^n$ procedure. 
In future work, a simple empirical strategy can be explored 
where we first obtain an estimate of $\wh \alpha$ by running the 
full (TED)$^n$ procedure till convergence and then discarding
the (TED)$^n$ classifier, we use estimate 
$\wh \alpha$ to train a fresh classifier with CVIR procedure.  


Finally, we discuss an important distinction with Dedpul 
which is also an alternating procedure.   
While in our algorithm, after updating mixture 
proportion estimate we retrain the 
classifier, Dedpul fixes the 
classifier, obtains output probabilities  
and then iteratively updates the 
mixture proportion estimate (prior)
and output probabilities (posterior). 
Dedpul doesn't re-train the 
classifier. 


%% file: sections/06_exp.tex
\begin{figure*}[t!]
  \centering 
  \subfigure{\includegraphics[width=0.45\linewidth]{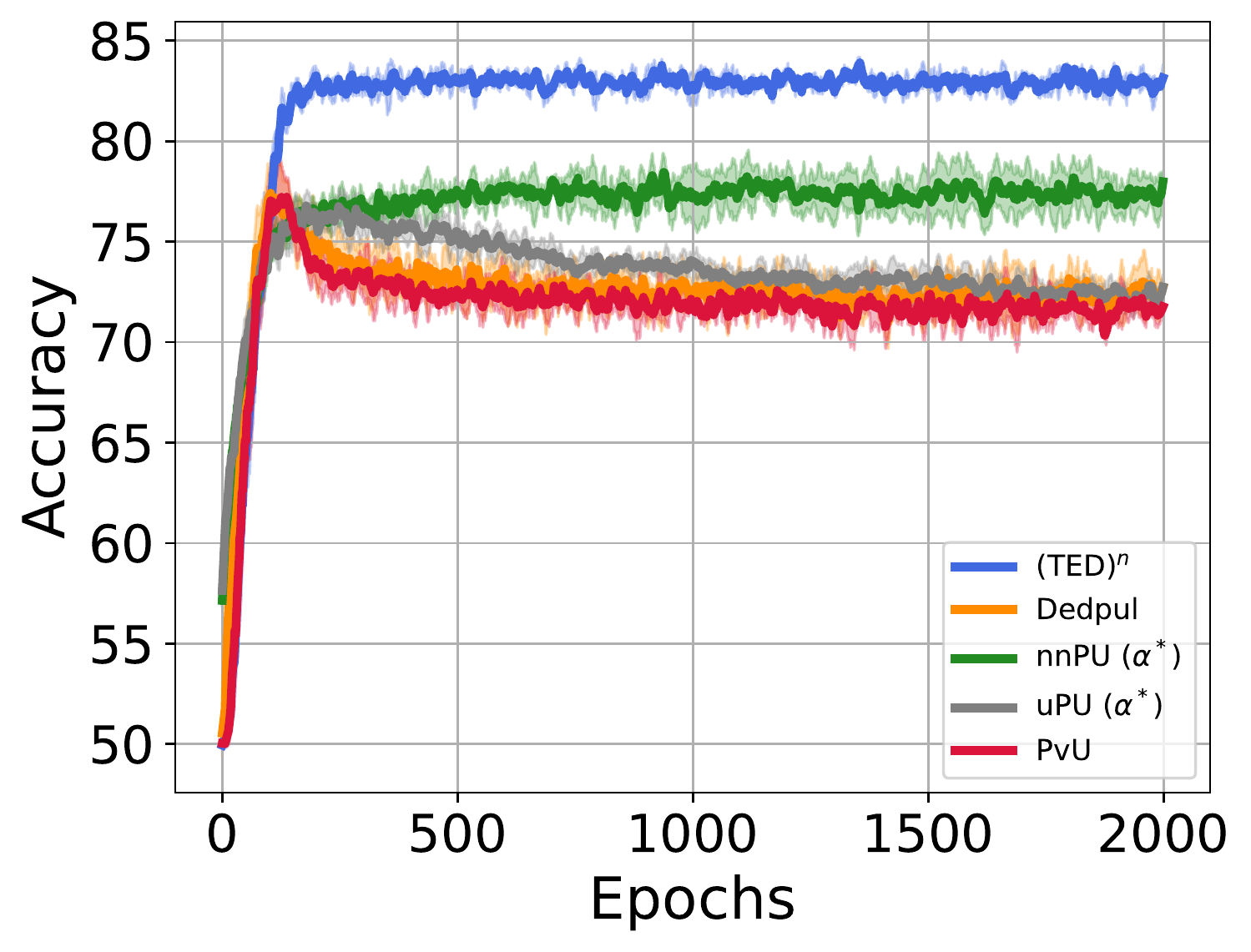}}\hfill
  \subfigure{\includegraphics[width=0.46\linewidth]{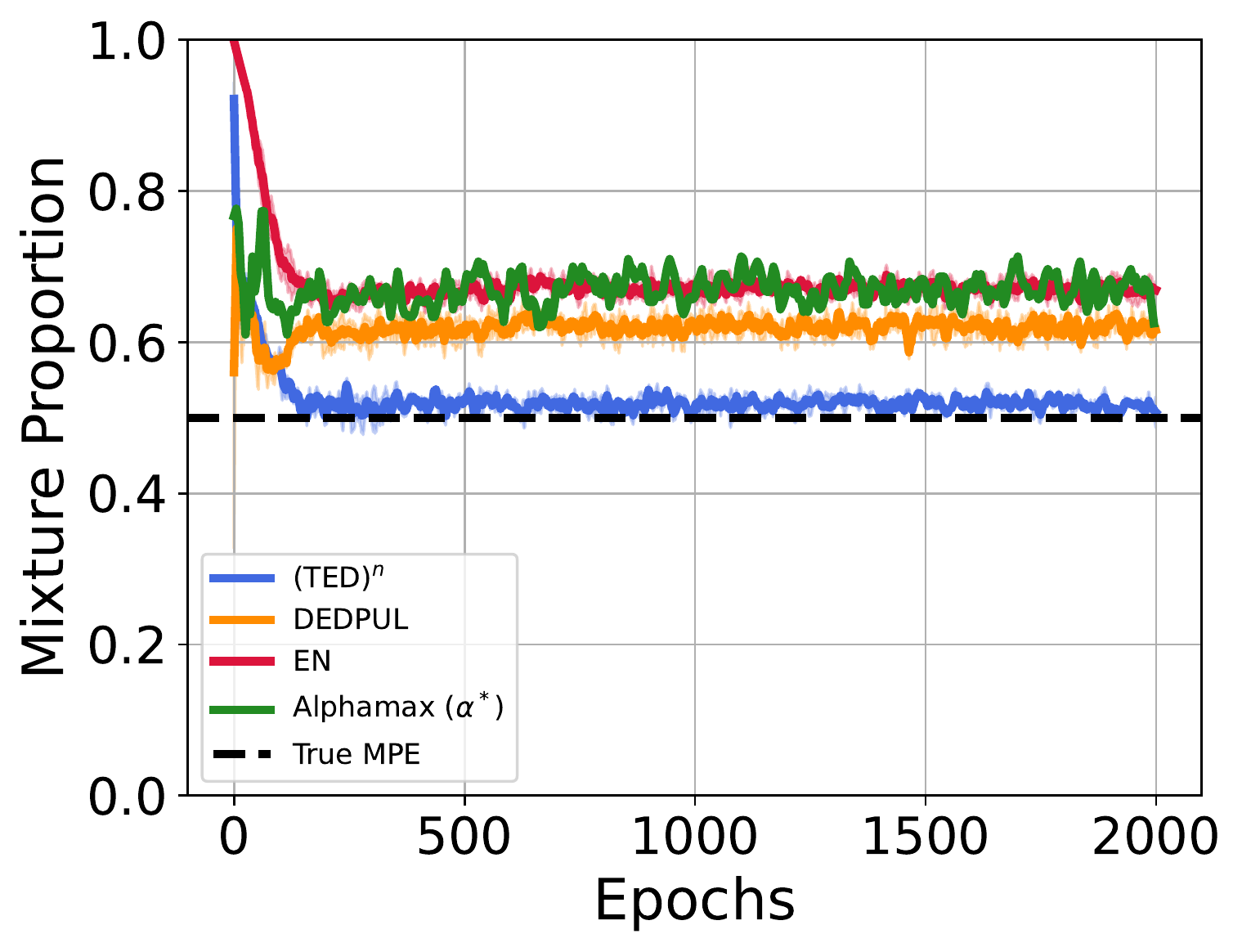}}
  \caption{ Epoch wise results with ResNet-18 trained on binary-CIFAR when $\alpha$ is $0.5$.
  Parallel results on other datasets and architectures in \appref{ap:figure3_app}.  
  For both classification and MPE, (TED)$^n$ substantially improves over existing methods.
  Additionally, (TED)$^n$ maintains the superior performance till convergence removing the 
  need for early stopping. 
  Results aggregated over 3 seeds.}
  \label{fig:cifar_results}
\end{figure*}

Having presented our PU learning 
and MPE algorithms, 
we now compare their performance 
with other methods empirically. 
We mainly focus on vision 
and text datasets in our experiments. 
We include results on UCI datasets in 
\appref{ap:UCI}.

\textbf{Datasets and Evaluation {} {}} 
We simulate PU tasks on CIFAR-10~\citep{krizhevsky2009learning},
MNIST~\citep{lecun1998mnist}, 
and IMDb sentiment analysis~\citep{maas2011learning} datasets. 
We consider binarized versions of CIFAR-10 and MNIST.  
On CIFAR-10 dataset, we consider 
two classification problems: 
(i)  binarized CIFAR, i.e., first 5 classes vs rest; 
(ii) Dog vs Cat in CIFAR. 
Similarly, on MNIST, we consider: 
(i)  binarized MNIST, i.e., digits 0-4 vs 5-9; 
(ii) MNIST17, i.e., digit 1 vs 7. 
IMDb dataset is binary. 
%
%
%
For MPE, we use a held out PU validation set.  
To evaluate PU classifiers, 
we calculate accuracy on  
held out positive versus negative dataset.
For baselines that suffer from 
issues due to overfitting on unlabeled data, 
we report results with 
an \emph{oracle early stopping} criterion. 
In particular, we report the 
accuracy averaged 
over $10$ iterations of  
the best performing model
as evaluated on positive versus negative data.
Note that we use this oracle stopping criterion
only for previously proposed methods and not for methods proposed
in this work. This allows us to compare 
(TED)$^n$ with the best performance achievable
by previous methods that 
suffer from over-fitting issues.   
With nnPU and (TED)$^n$, 
we report average accuracy 
over $10$ iterations 
of the final model.

\textbf{Architectures {} {}}
For CIFAR datasets,
we consider (fully connected) 
multilayer perceptrons (MLPs) 
with ReLU activations, 
all convolution nets~\citep{springenberg2014striving},
and ResNet18~\citep{he2016deep}.
For MNIST, we consider
multilayer perceptrons (MLPs) 
with ReLU activations 
For the IMDb dataset, 
we fine-tune an off-the-shelf uncased BERT model
\citep{devlin2018bert, wolf-etal-2020-transformers}. 
We did not tune hyperparameters 
or the optimization algorithm---instead 
we use the same benchmarked hyperparameters 
and optimization algorithm for each dataset. 
For our method,
we use cross-entropy loss. 
For uPU and nnPU, we use 
Adam~\citep{kingma2014adam}
with sigmoid loss. 
We provide additional details about the datasets 
and architectures in \appref{ap:details}.

\begin{table}[t]
  \centering
  \small
  \tabcolsep=0.12cm
  \renewcommand{\arraystretch}{1.2}
  \begin{tabular}{@{}*{9}{c}@{}}
  \toprule
  Dataset & Model  & \thead{(TED)$^n$} &  \thead{BBE$^*$} & \thead{DEDPUL$^*$} & \thead{AlphaMax$^*$} & \thead{EN$^*$} & \thead{KM2} & \thead{TiCE} \\
  \midrule
  \multirow{3}{*}{ \parbox{1.5cm}{\centering Binarized CIFAR} }  & ResNet & $\mathbf{0.026}$  & $0.091$ & $0.091$ & $0.125$ & $0.192$  & & \\
  & All Conv & $0.042$  & $\mathbf{0.037}$ & $0.052$ & $0.09$ & $0.221$  & $0.168$ & $0.251$  \\
  & MLP & $0.225$  & $0.177$  & $\mathbf{0.138}$ & $0.3$ & $0.372$ & &  \\
  \midrule
  \multirow{2}{*}{ \parbox{1.5cm}{\centering CIFAR Dog vs Cat} }  & ResNet & $\mathbf{0.078}$  & $0.176$ & $0.170$ & $0.17$ & $0.226$ & $0.331$ & $0.286$ \\
  & All Conv & $\mathbf{0.066}$  & $0.128$ & $0.115$ & $0.19$& $0.250$ &  &  \\
  \midrule 
  \multirow{1}{*}{ \parbox{2.5cm}{\centering Binarized MNIST} } & MLP & $\mathbf{0.024}$  & $0.032$ & $0.031$ & $0.090$ & $0.080$ & $0.029$ & $0.056$  \\
  \midrule
  \multirow{1}{*}{ \parbox{1.5cm}{\centering MNIST17} }  & MLP & $\mathbf{0.003}$  & $0.023$ & $0.021$ & $0.075$ & $0.028$ & $0.022$ & $0.043$ \\
  \midrule
  \multirow{1}{*}{ \parbox{1.5cm}{\centering IMDb} }  & BERT & $\mathbf{0.008}$ & $0.011$ & $0.016$ & $0.07$ & $0.12$ & - & - \\
  \bottomrule 
  \end{tabular}  
  \vspace{5pt}
  \caption{Absolute estimation error when $\alpha$ is 0.5. "*" denote oracle early stopping as defined in \secref{sec:exp}.
  (TED)$^n$ significantly reduces estimation error when compared with existing methods.  Results reported by aggregating absolute error over 10 epochs and 3 seeds. For aggregate numbers with standard deviation see \appref{app:results_std}. 
  }\label{table:MPE}
\end{table}

\begin{table}[t]
  \centering
  \small
  \tabcolsep=0.12cm
  \renewcommand{\arraystretch}{1.2}
  \begin{tabular}{@{}*{8}{c}@{}}
  \toprule
  Dataset & Model  & \thead{(TED)$^n$ \\(unknown $\alpha$)}  & \thead{CVIR\\(known $\alpha$)} & \thead{PvU$^*$ \\(known $\alpha$)} & \thead{DEDPUL$^*$ \\(unknown $\alpha$)}  & \thead{nnPU \\(known $\alpha$)} & \thead{uPU$^*$ \\(known $\alpha$)} \\
  \\
  \midrule
  \multirow{3}{*}{ \parbox{1.5cm}{\centering Binarized CIFAR} }  & ResNet & $\mathbf{82.7}$ & $82.3$& $76.9$ & $77.1$ & $77.2$ & $76.7$ \\
  & All Conv & $77.9$  & $\mathbf{78.1}$ & $75.8$ & $77.1$ & $73.4$ & $72.5$   \\
  & MLP & $64.2$  & $\mathbf{66.9}$ & $61.6$ & $62.6$ & $63.1$ & $64.0$\\
  \midrule
  \multirow{2}{*}{ \parbox{1.5cm}{\centering CIFAR Dog vs Cat} }  & ResNet & $\mathbf{75.2}$  & $73.3$ & $67.3$ & $67.0$ & $71.8$ & $68.8$  \\
  & All Conv & $\mathbf{73.0}$  & $71.7$ & $70.5$ & $69.2$ & $67.9$ & $67.5$  \\
  \midrule 
  \multirow{1}{*}{ \parbox{1.5cm}{\centering Binarized MNIST} } & MLP & $95.6$  & $\mathbf{96.3}$ & $94.2$ & $94.8$ & $96.1$ & $95.2$  \\
  \addlinespace[0.2cm]
  \midrule
  \multirow{1}{*}{ \parbox{1.5cm}{\centering MNIST17} }  & MLP & $\mathbf{98.7}$  & $\mathbf{98.7}$ & $96.9$ & $97.7$ & $98.4$ & $98.4$ \\
  \midrule
  \multirow{1}{*}{ \parbox{1.5cm}{\centering IMDb} }  & BERT & $\mathbf{87.6}$ & $87.4$ & $86.1$ & $87.3$& $86.2$  & $85.9$ \\
  \bottomrule 
  \end{tabular}  
  \vspace{5pt}
  \caption{Accuracy for PvN classification with PU learning. 
  "*" denote oracle early stopping as defined in \secref{sec:exp}.
  Results reported by aggregating over 10 epochs 
  and 3 seeds. Both CVIR (with known MPE) and (TED)$^n$ (with unknown MPE) significantly improve over previous baselines with oracle early stopping and known MPE. For aggregate numbers with standard deviation see \appref{app:results_std}. 
  }\label{table:classification}
\end{table}
  
\textbf{Mixture Proportion Estimation {} {}}
First, we discuss results for MPE (\tabref{table:MPE}). 
We compare our method with KM2, TiCE, 
DEDPUL, AlphaMax and EN.
Following earlier works~\citep{ivanov2019dedpul,ramaswamy2016mixture},  
we reduce datasets to $50$ dimensions with PCA
for KM2 and TiCE. 
We use existing implementation 
for other methods\footnote{DEDPUL: https://github.com/dimonenka/DEDPUL, 
KM: https://web.eecs.umich.edu/\textasciitilde cscott/code.html\#kmpe, 
TiCE: https://dtai.cs.kuleuven.be/software/tice, and
AlphaMax: https://github.com/Dzeiberg/AlphaMax}. 
For BBE, DEDPUL and Alphamax, we use the 
same PvU classifier as input.
%
%
On CIFAR datasets, convolutional 
classifier based  
estimators significantly outperform 
KM2 and TiCE. In contrast, the performance of 
KM2 is comparable to DEDPUL on MNIST datasets.
On all datasets, (TED)$^n$ achieves 
lowest estimation error.
%
%
With the same blackbox classifier
obtained with oracle early stopping, BBE 
performs similar or better than best alternate(s). 
Since overparamterized models start memorizing 
unlabeled samples negatives,
the quality of MPE degrades substantially 
as PvU training proceeds 
for all methods but (TED)$^n$ 
as in \figref{fig:cifar_results} 
(epoch-wise results for on other tasks in \appref{ap:figure3_app}).

\textbf{Classification with known MPE {} {}}
Now, we discuss results for classification with known $\alpha$. 
We compare our method with uPU,
nnPU\footnote{uPU and nnPU: https://github.com/kiryor/nnPUlearning}, DEDPUL
and PvU training. 
Although, we solve both MPE and classification,  
some comparison methods do not. 
Ergo, we compare our classification algorithm 
with known MPE (\algoref{alg:classification_PU}). 

To begin, first we note that nnPU 
and PvU training with CVIR
doesn't need early stopping. 
For all other methods,  
we report the best performance dictated by
the aformentioned oracle stopping criterion. 
On all datasets, PvU training with 
CVIR leads to improved classification 
performance when compared 
with alternate approaches (\tabref{table:classification}). 
Moreover, as training proceeds (\figref{fig:cifar_results}),
the performance of DEDPUL, PvU training and uPU substantially degrade. 
We repeated experiments with 
the early stopping criterion mentioned in 
DEDPUL (\appref{ap:early_stopping}), however, 
their early stopping criterion is too pessimistic 
resulting in poor results due to under-fitting. 

\textbf{Classification with unknown MPE {} {}}
Finally, we evaluate (TED)$^n$,
our alternating procedure 
for MPE and PU learning. 
Across many tasks, 
we observe substantial improvements 
over existing methods.  
Note that these improvements often are over an oracle early
stopping baselines highlighting significance of 
our procedure. 

In \appref{ap:ablation}, we show that our procedure 
is not sensitive to warm start epochs W, 
and in many tasks with $W=0$,
we observe minor-to-no differences
in the performance of (TED)$^n$. 
While for the experiments in this section,
we used fixed $W=100$, 
in the Appendix we show 
behavior with varying W. 
We also include ablations 
with different mixture proportions $\alpha$.

%% file: sections/07_conclusion.tex
In this paper, we proposed 
two practical algorithms,
BBE (for MPE) and CVIR optimization (for PU learning). 
Our methods outperform others empirically
and BBE's mixture proportion estimates
leverage black box classifiers 
to produce (nearly) consistent estimates 
with finite sample convergence guarantees
whenever we possess a classifier
with a (nearly) pure top bin.
Moreover, (TED)$^n$ combines our procedures 
in an iterative fashion,
achieving further gains.
%
%
An important next direction 
is to extend our work 
to the multiclass problem~\citep{sanderson2014class},
bridging work on label shift and PU learning.
Here, we imagine that 
a deployed $k$-way classifier 
may encounter not only label shift
among previously seen classes (\cite{lipton2018detecting,garg2020unified})
but also, potentially, 
instances from one previously unseen class.
%
We also plan to investigate distributional properties 
under which we can hope
to reliably or approximately satisfy
the pure positive bin property 
with an off-the-shelf classifier 
trained on PvU data. 
While we improve significantly 
over previous PU methods,  
there is still a gap between 
(TED)$^n$'s performance and PvN training.  
We hope that our work can open 
a pathway towards further narrowing this gap.

%% file: sections/appendix.tex
\section{Appendix}

\section{Proofs from \secref{sec:MPE}}\label{ap:proof_mpe}

\begin{proof}[Proof of \lemref{lem:ucb}]
    The proof primarily involves using 
    DKW inequality~\citep{dvoretzky1956asymptotic} 
    on $\wh q_u(c)$ and $\wh q_p(c)$ to show convergence 
    to their respective means $q_u(c)$ and $q_p(c)$. 
    First, we have 

    \begin{align*} 
        \abs{ \frac{\wh q_u(c)}{ \wh q_p(c)} -  \frac{q_u(c)}{q_p(c)}} &= \frac{1}{\wh q_u(c) \cdot q_u(c) }\abs{ \wh q_u(c) \cdot q_p(c) - q_p(c)\cdot q_u(c) +  q_p(c)\cdot q_u(c) - \wh q_p(c) \cdot q_u(c) } \\
        &\le \frac{1}{\wh q_p(c)}\abs{\wh q_u(c) - q_u(c)} + \frac{q_u(c)}{\wh q_p(c) \cdot q_u(c)} \abs{ \wh q_p(c) - q_p(c)} \,. \numberthis \label{eq:ratio_bound}
    \end{align*}
    
    Using DKW inequality, we have with probability $1-\delta$: $\abs{ \wh q_p(c) - q_p(c)} \le \sqrt{\frac{\log(2/\delta)}{2n_p}}$ for all $c \in [0,1]$. Similarly, we have with probability $1-\delta$: $\abs{\wh q_u(c) - q_u(c)} \le \sqrt{\frac{\log(2/\delta)}{2n_u}}$ for all $c\in [0,1]$. Plugging this in \eqref{eq:ratio_bound}, we have 
    \begin{align*}
        \abs{\frac{\wh q_u(c)}{\wh q_p(c)} -  \frac{q_u(c)}{q_p(c)}} \le \frac{1}{\wh q_p(c)}\left( \sqrt{\frac{\log(4/\delta)}{2n_u}} + \frac{q_u(c)}{q_p(c)}\sqrt{\frac{\log(4/\delta)}{2n_p}}\right) \,.
    \end{align*}
\end{proof}

\input{sections/new_proof.tex}

\section{Comparison of BBE with \citet{scott2015rate}} \label{app:comparison}

Heuristic estimator due to \citet{scott2015rate} is 
motivated by the estimator in \citet{blanchard2010semi}. 
The estimator in \citet{blanchard2010semi} relies on 
VC bounds, which are known to be loose in typical 
deep learning situations. Therefore, \citet{scott2015rate} 
proposed an heuristic implementation based on the 
minimum slope of any point in the ROC space to the 
point $(1,1)$. To obtain ROC estimates, authors 
use direct binomial tail inversion (instead of VC bounds as in \citet{blanchard2010semi})
to obtain tight upper bounds for 
true positives and lower bounds for true negatives.  
Finally, using these conservatives estimates 
the estimator in \citet{scott2015rate} is obtained 
as the minimum slope of any of the operating points to the point $(1,1)$.

While the estimate of one minus true positives at a threshold $t$ is 
similar in spirit to our number of unlabeled examples in the top bin 
and the estimate of one minus true negatives at a threshold $t$ is similar 
in spirit to our number of positive examples in the unlabeled data,
the functional form of these estimates are very different. 
\citet{scott2015rate} estimator is the ratio of quantities obtained by binomial tail inversion (i.e. upper bound in the numerator and lower bound in the denominator). By contrast, the final BBE estimate is simply the ratio of empirical CDFs at the optimal threshold. Mathematically, we have 
\begin{align}
  \wh \alpha_{\text{Scott}} &= \frac{ \wh q_u(c_\text{Scott}) + \text{binv}(n_u, \wh q_u(c_\text{Scott}), \delta/n_u) }{ \wh q_p(c_\text{Scott}) - \text{binv}(n_p, \wh q_p(c_\text{Scott}), \delta/n_p)} \qquad \text{and} \label{eq:scott} \\ 
  \wh \alpha_{\text{BBE}} &= \frac{ \wh q_u(c_\text{BBE})}{\wh q_p(c_\text{BBE})} \,, \label{eq:BBE}
\end{align}
where $c_\text{Scott} = \argmin_{c \in [0,1]} \frac{\wh q_u(c) + \text{binv}(n_u, \wh q_u(c), \delta/n_u)}{\wh q_p(c) - \text{binv}(n_p, \wh q_p(c), \delta/n_p)}$ and binv$(n_p, q_p(c), \delta/n_p)$ is the tightest possible deviation bound for a binomial random variable~\citep{scott2015rate} and and $c_\text{BBE}$ is given by Algorithm~\ref{alg:MPE_PU}.  
Moreover, \citet{scott2015rate} provide no theoretical guarantees for their heuristic estimator $ \wh \alpha_{\text{Scott}}$. On the hand, we provide guarantees that our estimator $\wh \alpha_{\text{BBE}}$ will converge to the best estimate achievable over all choices of the bin size and provide consistent estimates whenever a pure top bin exists.
Supporting theoretical results of BBE, 
we observe that these choices in BBE create 
substantial differences in 
the empirical performance as observed in \tabref{table:MPE_scott}. We repeat experiment for MPE from \secref{sec:exp} where we compare other methods with the \citet{scott2015rate} estimator as defined in \eqref{eq:scott}.

As a side note, a naive implementation of $\wh \alpha_{\text{Scott}}$ instead of \eqref{eq:scott} where we directly minimize the empirical ratio yields poor estimates due to noise introduced with finite samples. In our experiments, we observed that $\wh \alpha_{\text{Scott}}$ improves a lot over this naive estimator.

\begin{table}[t]
  \centering
  \small
  \tabcolsep=0.12cm
  \renewcommand{\arraystretch}{1.2}
  \begin{tabular}{@{}*{6}{c}@{}}
  \toprule
  Dataset & Model  & \thead{(TED)$^n$} &  \thead{BBE$^*$} & \thead{DEDPUL$^*$} & \thead{Scott$^*$}\\
  \midrule
  Binarized CIFAR  & ResNet & $\mathbf{0.018}$ & $0.072$  & $0.075$ & $0.091$ \\
  \midrule
  \parbox{1.5cm}{\centering CIFAR Dog vs Cat}  & ResNet & $\mathbf{0.074}$ & $0.120$ & $0.113$ & $0.158$ \\
  \midrule 
  \multirow{1}{*}{ \parbox{2.5cm}{\centering Binarized MNIST} } & MLP & $\mathbf{0.021}$ & $0.028$ & $0.027$ & $0.063$  \\
  \midrule
  \multirow{1}{*}{ \parbox{1.5cm}{\centering MNIST17} }  & MLP & $\mathbf{0.003}$ & $0.008$ & $0.006$ & $0.037$ \\
  \bottomrule 
  \end{tabular}  
  \vspace{5pt}
  \caption{Absolute estimation error when $\alpha$ is 0.5. "*" denote oracle early stopping as defined in \secref{sec:exp}. As mentioned in \citet{scott2015rate} implementation in \href{https://web.eecs.umich.edu/\textasciitilde cscott/code/mpe\_v2.zip}{https://web.eecs.umich.edu/\textasciitilde cscott/code/mpe\_v2.zip}, we use the binomial inversion at $\delta$ instead of $\delta/n$ (rescaling using the union bound). Since we are using Binomial inversion at n discrete points simultaneously, we should use the union-bound penalty. However, using union bound penalty substantially increases the bias in their estimator.
  }\label{table:MPE_scott}
\end{table}

\section{Toy setup }\label{ap:failure}

\citet{jain2016nonparametric}
and \citet{ivanov2019dedpul}
discuss Bayes optimality 
of the PvU classifier 
(or its one-to-one mapping)
as a sufficient condition 
to preserve $\alpha$ in 
transformed space. 
However, in a simple toy setup 
(in \appref{ap:failure}),
we show that even when the hypothesis class
is well specified for PvN learning,
it will not in general contain
the Bayes optimal scoring function 
for PvU data
and thus PvU training 
will not recover the 
Bayes-optimal scoring function, 
even in population.

Consider a scenario with $\inpt = \Real^2$.  Assume points from the positive class are sampled uniformly from the interior of the triangle defined by coordinates $\{(-1, 0.1), (0,4), (1, 0.1)\}$ and negative points are sampled uniformly from the interior of triangle defined by coordinates $\{ (-1, -0.1), (4, -4), (1, -0.1)\}$. Ref. to \figref{fig:counter_example_Bayes} for a pictorial representation. Let mixture proportion be $0.5$ for the unlabeled data. 
Given access to distribution of positive data and unlabeled data, we seek to train a linear classifier to minimize logistic or Brier loss for PvU training. 

Since we need a monotonic transformation of the Bayes optimal scoring function, we want to recover a predictor parallel to x-axis, the Bayes optimal classifier for PvN training. However, minimizing the logistic loss (or Brier loss) using numerical methods, we obtain a predictor that is inclined at a non-zero acute angle to the x-axis. Thus, the PvU classifier obtained fails to satisfy the sufficient condition from \citet{jain2016nonparametric} and \citet{ivanov2019dedpul}. On the other hand, note that the linear classifier obtained by PvU training satisfies the pure positive bin property. 

\begin{figure}[t!]
    \centering
    \subfigure[]{\includegraphics[width=0.5\linewidth]{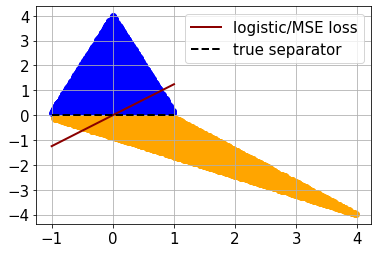}}
    \caption{ Blue points show samples from the positive distribution and orange points show samples from the negative distribution. Unalabeled data is obtained by mixing positive and negative distribution with equal proportion. BCE (or Brier) loss minimization on P vs U data leads to a classifiers that is not consistent with the ranking of the Bayes optimal score function.}
    \label{fig:counter_example_Bayes}
\end{figure}

Now we show that under the subdomain assumption~\citep{scott2015rate,ramaswamy2016mixture}, any monotonic transformation of Bayes optimal scoring function induces positive pure bin property. First, we define the subdomain assumption. 

\begin{assumption}[Subdomain assumption] \label{asmp:subdomain} 
    A family of subsets $\calS \subseteq 2^\calX$, and distributions $p_{p}$, $p_{n}$ are said to satisfy the anchor set condition with margin $\gamma >0$, if there exists a compact set $A \in \calS$ such that $A \subseteq \supp(p_{p}) / \supp(p_{n})$ and $p_{p}(A) \ge \gamma$.  
\end{assumption}

Note that any monotonic mapping of the Bayes optimal scoring function can be represented by $ \tau^\prime =  g \circ \tau$, where g is a monotonic function and 
\begin{align}
    \tau(x) = \begin{cases}
                    p_p(x)/ p_u(x) \quad & \text{if } p_p(x) >0\\
                    0 \quad & \text{o.w}\,.
                \end{cases} 
\end{align}

For any point $x\in A$ and $x^\prime \in \inpt / A $, we have $\tau(x) > \tau(x^\prime)$ which implies $\tau^\prime(x) > \tau^\prime(x^\prime)$. Thus, 
any monotonic mapping of Bayes optimal scoring function yields the positive pure bin property with $\epsilon_p \ge \gamma$.

\section{Analysis of CVIR}\label{ap:anlysis_CVuO}


First we analyse our loss 
function in the scenario when 
the support of positives and 
negatives is separable.
We assume that the 
true alpha $\alpha$ is known and we have 
access to populations of 
positive and unlabeled data.   
We also assume that
their exists a separator $f^*: \mathcal{X} \mapsto \{0, 1\}$
that can perfectly separate the 
positive and negative distribution, i.e., $\int dx p_p(x) \indict{f^*(x) \ne 1} + \int dx p_n(x) \indict{f^*(x) \ne 0} = 0$.
Our learning objective can be written as jointly optimizing a classifier $f$ and a weighting function $w$ on the unlabeled distribution:  
\begin{align*}
& \min_{f\in \mathcal{F}, w} \int dxp_p(x)l(f(x), 1) + \frac{1}{1- \alpha} \int dxp_u(x)w(x)l(f(x), 0) \,, \\
& \text{ s.t. } w: \mathcal{X}\mapsto [0, 1] \,, \int dxp_u(x)w(x) = 1 - \alpha \,. \addeq\label{eq:cvuo-obj} 
\end{align*}
The following proposition shows that minimizing the objective \eqref{eq:cvuo-obj} on separable positive and negative distributions gives a perfect classifier.
\begin{proposition}
For $\alpha \in (0, 1)$, if there exists a classifier $f^* \in \mathcal{F}$ that can perfectly separate the positive and negative distributions, optimizing objective \eqref{eq:cvuo-obj} with 0-1 loss leads to a classifier $f$ that achieves $0$ classification error on the unlabeled distribution.
\end{proposition}

\begin{proof}
First we observe that having $w(x) = 1 - f^*(x)$ leads to the objective value being minimized to 0 as well as a perfect classifier $f$. This is because
\begin{align*}
\frac{1}{1- \alpha} \int dxp_u(x)(1 - f^*(x))l(f(x), 0)   = \int dxp_n(x)l(f(x), 0) 
\end{align*}
thus the objective becomes classifying positive v.s. negative, which leads to a perfect classifier  if $\mathcal{F}$ contains one. Now we show that for any $f$ such that the classification error is non-zero then the objective \eqref{eq:cvuo-obj} must be greater than zero no matter what $w$ is.
Suppose $f$ satisfies
\begin{align*}
\int dx p_p(x) l(f(x),1) + \int dx p_n(x) l(f(x), 0) > 0 \,.
\end{align*}
We know that either $\int dx p_p(x) l(f(x),1) > 0$ or $\int dx p_n(x) l(f(x), 0) > 0$ will hold. If $\int dx p_p(x) l(f(x),1) > 0$ we know that \eqref{eq:cvuo-obj} must be positive. If $\int dx p_p(x) l(f(x),1) = 0$ and $\int dx p_n(x) l(f(x), 0) > 0$ we have $l(f(x), 0) = 1$ almost everywhere in $p_p(x)$ thus
\begin{align*}
& \frac{1}{1- \alpha} \int dxp_u(x)w(x)l(f(x), 0) \\
& = \frac{\alpha}{1 - \alpha} \int dxp_p(x)w(x)l(f(x), 0) + \int dxp_n(x)w(x)l(f(x), 0) \\
&  = \frac{\alpha}{1 - \alpha} \int dx p_p(x)w(x) + \int dxp_n(x)w(x)l(f(x), 0) \,.
\end{align*}
If $\int dx p_p(x)w(x)>0$ we know that \eqref{eq:cvuo-obj} must be positive. If $\int dx p_p(x)w(x) = 0$,
since we know that 
\begin{align*}
\int dxp_u(x)w(x) = \alpha \int dxp_p(x)w(x) + (1 - \alpha)\int dxp_n(x)w(x) = 1 - \alpha
\end{align*}
we have $\int dxp_n(x)w(x) = 1$ which means $w(x)=1$ almost everywhere in $p_n(x)$. This leads to the fact that
$\int dx p_n(x) l(f(x), 0) > 0$ indicates $\int dxp_n(x)w(x)l(f(x), 0) > 0$, which concludes the proof.

\end{proof}

The intuition is that, any classifier that 
discards an $\wt \alpha >0$ 
proportion
of negative distribution 
from unlabeled
will have loss strictly 
greater than zero
with our CVIR objective. 
Since only a perfect linear separator 
(with weights $\to \infty$)
can achieves loss $\to 0$, 
CVIR objective will (correctly)
discard the $\alpha$ proportion of positive 
from unlabeled data 
achieving 
a classifier that
perfectly separates the data. 

We leave theoretic 
investigation on non-separable 
distributions for future work. However, as an 
initial step towards a general theory, 
we show that 
in the population case 
one step of our alternating procedure 
cannot increase the loss.

Consider the following objective function 
\begin{align*}
  L(f_t, w_t) &= E_{x \sim P_p}[ l( f_t(x), 0) ] + E_{x \sim P_u}[ w_t(x) l( f_t(x), 1) ] \numberthis \label{eq:obj_onestep} \\ 
  \text{such that} \qquad & E_{x \sim P_u}[ w(x)] = 1-\alpha \text{ and } w(x) \in \{0,1\}
\end{align*}

Given $f_t$ and $w_t$, CVIR can be summarized as the following two step iterative procedure: 
(i) Fix $f_t$, optimize the loss to obtain $w_{t+1}$; and 
(ii) Fix $w_{t+1}$ and optimize the loss to obtain $f_{t+1}$. 
By construction of CVIR, we select $w_{t+1}$ 
such that we discard points with highest loss, 
and hence $L(f_t, w_{t+1}) \le L(f_t, w_{t})$. 
Fixing $w_{t+1}$, we minimize the $L(f_t, w_{t+1})$ to obtain $f_{t+1}$ and 
hence $L(f_{t+1}, w_{t+1}) \le L(f_t, w_{t+1})$. 
Combining these two steps, we get $L(f_{t+1}, w_{t+1}) \le L(f_t, w_{t})$.

\section{Experimental Details}\label{ap:details}

Below we present dataset details. We present experiments with MNIST Overlap in \appref{ap:mnist_overlap}. 
\begin{center}
    \begin{table}[H] 
        \centering
        \tabcolsep=0.12cm
        \renewcommand{\arraystretch}{1.2}
        \begin{tabular}{@{}*{7}{c}@{}}
        \toprule
        
        Dataset & Simulated PU Dataset & P vs N & \multicolumn{2}{c}{\#Positives} & \multicolumn{2}{c}{\#Unlabeled} \\ 
        & & & Train & Val & Train & Val \\ [0.5ex] 
        \midrule
        \multirow{2}{*}{CIFAR10} & Binarized CIFAR & [0-4] vs [5-9] & 12500 & 12500 & 2500 & 2500 \\
                                &  CIFAR Dog vs Cat & 3 vs 5 & 2500 & 2500 & 500 & 500 \\
        \midrule 
        \multirow{3}{*}{MNIST} & Binarized MNIST & [0-4] vs [5-9] & 15000 & 15000 & 2500 & 2500 \\ 
                                &  MNIST 17 & 1 vs 7 & 3000 & 3000 & 500 & 500 \\
                                &  MNIST Overlap & [0-7] vs [3-9] & 150000 & 15000 & 2500 & 2500 \\
        \midrule 
        IMDb &  IMDb & pos vs neg & 6250 & 6250 & 5000 & 5000 \\
        \bottomrule 
        \end{tabular}
    \end{table}    
\end{center}

For CIFAR dataset, we also use the standard data augementation of random crop and horizontal flip. PyTorch code is as follows: 

\texttt{(transforms.RandomCrop(32, padding=4),\\
\tab transforms.RandomHorizontalFlip())}

\subsection{Architecture and Implementation Details} 

All experiments were run on NVIDIA GeForce RTX 2080 Ti GPUs. We used PyTorch~\citep{NEURIPS2019a9015} and Keras with Tensorflow~\citep{abadi2016tensorflow} backend for experiments.  

For CIFAR10, we experiment with convolutional nets and MLP. For MNIST, we train MLP. In particular, we use ResNet18 \citep{he2016deep} and all convolution net~\citep{springenberg2014striving} . Implementation adapted from:  \url{https://github.com/kuangliu/pytorch-cifar.git}. We consider a 4-layered MLP. The PyTorch code for 4-layer MLP is as follows: 

\texttt{ nn.Sequential(nn.Flatten(), \\
\tab        nn.Linear(input\_dim, 5000, bias=True),\\
\tab        nn.ReLU(),\\
\tab        nn.Linear(5000, 5000, bias=True),\\
\tab        nn.ReLU(),\\
\tab        nn.Linear(5000, 50, bias=True),\\
\tab        nn.ReLU(),\\
\tab        nn.Linear(50, 2, bias=True)\\
\tab        )}

For all architectures above, we use Xaviers initialization~\citep{glorot2010understanding}. For all methods except nnPU and uPU, we do cross entropy loss minimization with SGD optimizer with momentum $0.9$. For convolution architectures we use a learning rate of $0.1$ and MLP architectures we use a learning rate of $0.05$. For nnPU and uPU, we minimize sigmoid loss with ADAM optimizer with learning rate $0.0001$ as advised in its original paper. For all methods,  we fix the weight decay param at $0.0005$.

For IMDb dataset, we fine-tune an off-the-shelf uncased BERT model~\citep{devlin2018bert}. Code adapted from Hugging Face Transformers~\citep{wolf-etal-2020-transformers}: \url{https://huggingface.co/transformers/v3.1.0/custom_datasets.html}. For all methods except nnPU and uPU, we do cross entropy loss minimization with Adam optimizer with learning rate $0.00005$ (default params). With the same hyperparameters and Sigmoid loss, we could not train BERT with nnPU and uPU due to vanishing gradients. Instead we use learning rate $0.00001$.

\subsection{Division between training set and hold-out set} 

Since the training set is used to 
learn the classifier 
(parameters of a deep neural network) 
and the hold-out set is just used 
to learn the mixture proportion estimate (scalar), 
we use a larger dataset for training. 
Throughout the experiments, we use an 
80-20 split of the original set. 

At a high level, we have an error bound 
on the mixture proportion estimate and 
we can use that to decide the split in general. 
As long as we use enough samples to make the 
$\calO(1/\sqrt{n})$ small in
our bound in \thmref{thm:main_MPE}, 
we can use the rest of the samples 
to learn the classifier.

\section{Additional Experiments}

\subsection{nnPU vs PN classification} \label{ap:PN}

In this section, we compare the performance of nnPU and PvN training on the same positive and negative (from the unlabeled) data at $\alpha = 0.5$.   
We highlight the huge classification performance gap between nnPU and PvN training and show that training with CVuO objective partially recovers the performance gap. 
Note, to train PvN classifier, we use the same hyperparameters as that with PvU training. 

\begin{table}[h]
    \centering
    \small
    \tabcolsep=0.12cm
    \renewcommand{\arraystretch}{1.2}
    \begin{tabular}{@{}*{6}{c}@{}}
    \toprule
    Dataset & Model  & \thead{nnPU \\(known $\alpha$)} & \thead{PvN}  & \thead{CVuO\\(known $\alpha$)} & \thead{(TED)$^n$ \\(unknown $\alpha$)} \\
    \midrule
    \multirow{3}{*}{ \parbox{1.5cm}{\centering Binarized CIFAR} }  & ResNet & $76.8$ & $86.9$ & $82.6$ & $82.7$\\
    & All Conv & $72.1$ & $76.7$ & $77.1$ & $76.8$ \\
    & MLP & $63.9$ & $65.1$ & $65.9$ & $63.2$ \\
    \midrule
    \multirow{2}{*}{ \parbox{1.5cm}{\centering CIFAR Dog vs Cat} }  & ResNet & $72.6$ & $80.4$ & $74.0$ & $76.1$ \\
    & All Conv & $68.4$ & $77.9$ & $71.0$ &  $72.2$ \\
    \midrule 
    \multirow{1}{*}{ \parbox{1.5cm}{\centering Binarized MNIST} } & MLP & $95.9$ & $96.7$ & $96.4$ & $95.9$ \\
    \addlinespace[0.2cm]
    \midrule
    \multirow{1}{*}{ \parbox{1.5cm}{\centering MNIST17} }  & MLP & $98.2$ & $99.0$ & $98.6$ & $98.6$\\
    \midrule
    \multirow{1}{*}{ \parbox{1.5cm}{\centering IMDb} }  & BERT & $86.2$ & $89.1$ & $87.4$ & $88.1$ \\
    \bottomrule 
    \end{tabular}  
    \vspace{10pt}
    \caption{Accuracy for PvN classification with nnPU, PvN, CVuO objective and (TED)$^n$ training. 
    Results reported by aggregating aggregating over 10 epochs.}\label{table:pn_classification}
  \end{table}

\subsection{Under-Fitting due to pessimistic early stopping} \label{ap:early_stopping}

\citet{ivanov2019dedpul} explored the following 
heuristics for ad-hoc early stopping criteria: 
training proceeds 
until the loss on unseen PU data
ceases to decrease. 
In particular, the authors 
suggested early stopping criterion based
on the loss on unseen PU data 
doesn't decrease in epochs separated 
by a pre-defined window of length $l$.
The early stopping is done when this happens 
consecutively for $l$ epochs.
However, this approach 
leads to severe under-fitting. 
When we fix $l=5$, we observe a significant 
performance drop in CIFAR classification and MPE. 

With PvU training, 
the performance of ResNet model on Binarized CIFAR
(in \tabref{table:classification}) 
drops from $78.3$ (orcale stopping) to 
$60.4$ (with early stopping). Similar on CIFAR CAT vs Dog, 
the performance of the same architecture drops 
from $71.6$ (orcale stopping) to 
$58.4$ (with early stopping). 
Note that the decrease in accuracy is less 
or not significant
for MNIST.
With PvU training, 
the performance of MLP model on Binarized MNIST
(in \tabref{table:classification}) 
drops from $94.5$ (orcale stopping) to 
$94.1$ (with early stopping). 
This is because we obtain good performance 
on MNIST early in training.

\subsection{Results parallel to \figref{fig:cifar_results}} \label{ap:figure3_app}

Epoch wise results for all models for Binarized CIFAR, CIFAR Dog vs Cat, Binarized MNIST, MNIST 17 and IMDb. 

\begin{figure*}[h!]
  \centering 
  \subfigure{\includegraphics[width=0.4\linewidth]{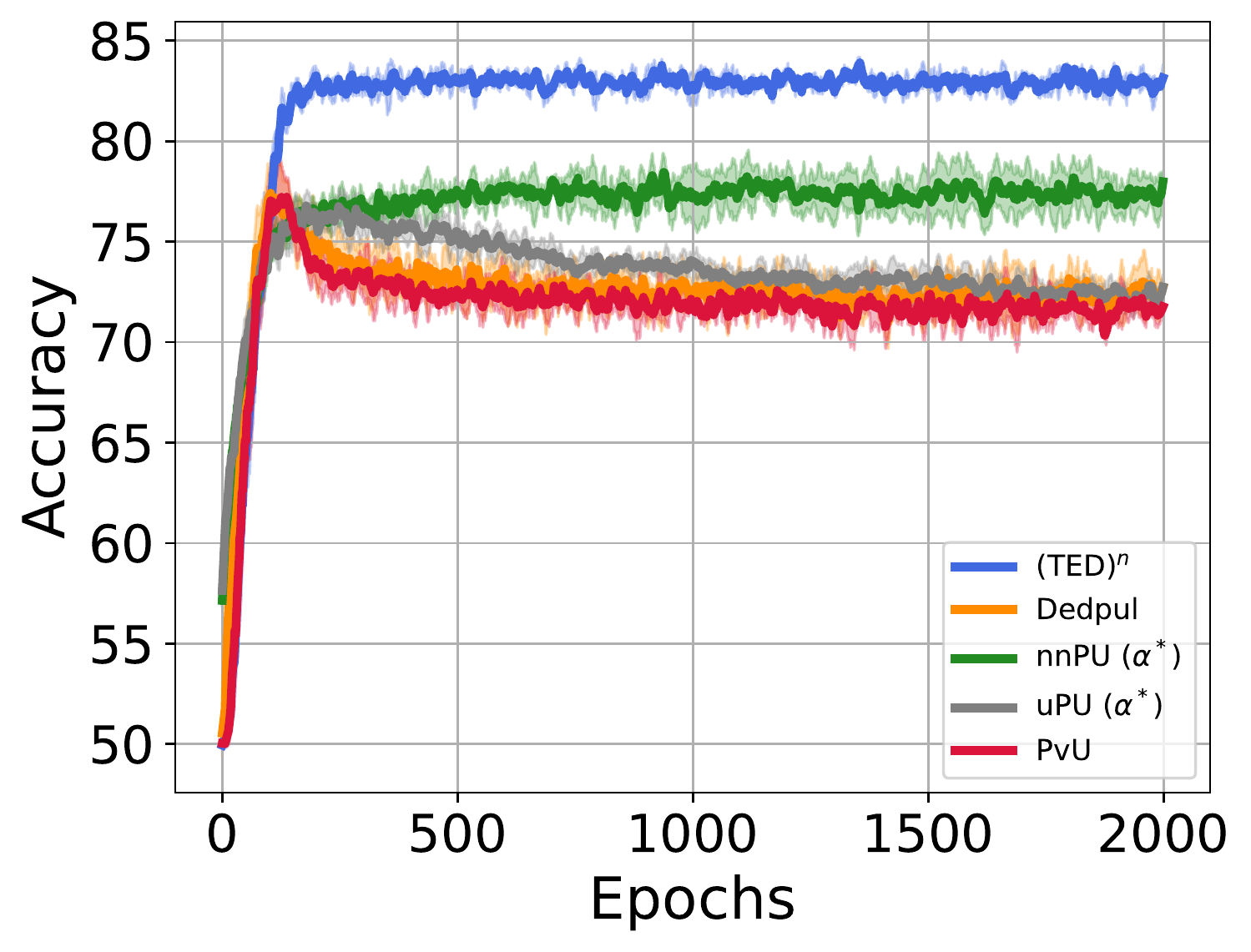}}\hfil
  \subfigure{\includegraphics[width=0.4\linewidth]{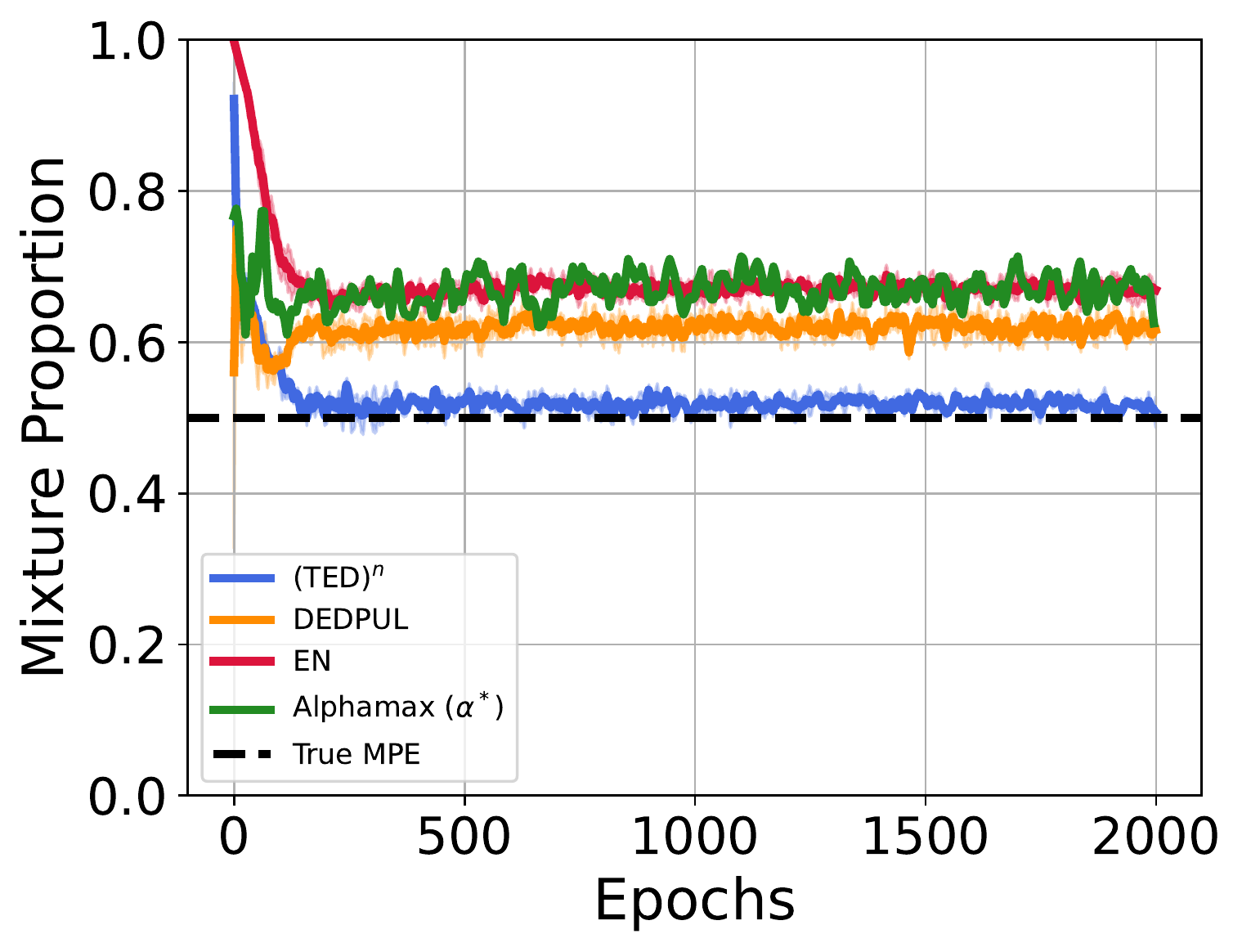}}
  \caption{ Epoch wise results with ResNet-18 network trained on CIFAR-binarized.}
  \label{fig:cifar_binarized_resnet_results}
\end{figure*}

\begin{figure*}[h!]
  \centering 
  \subfigure{\includegraphics[width=0.4\linewidth]{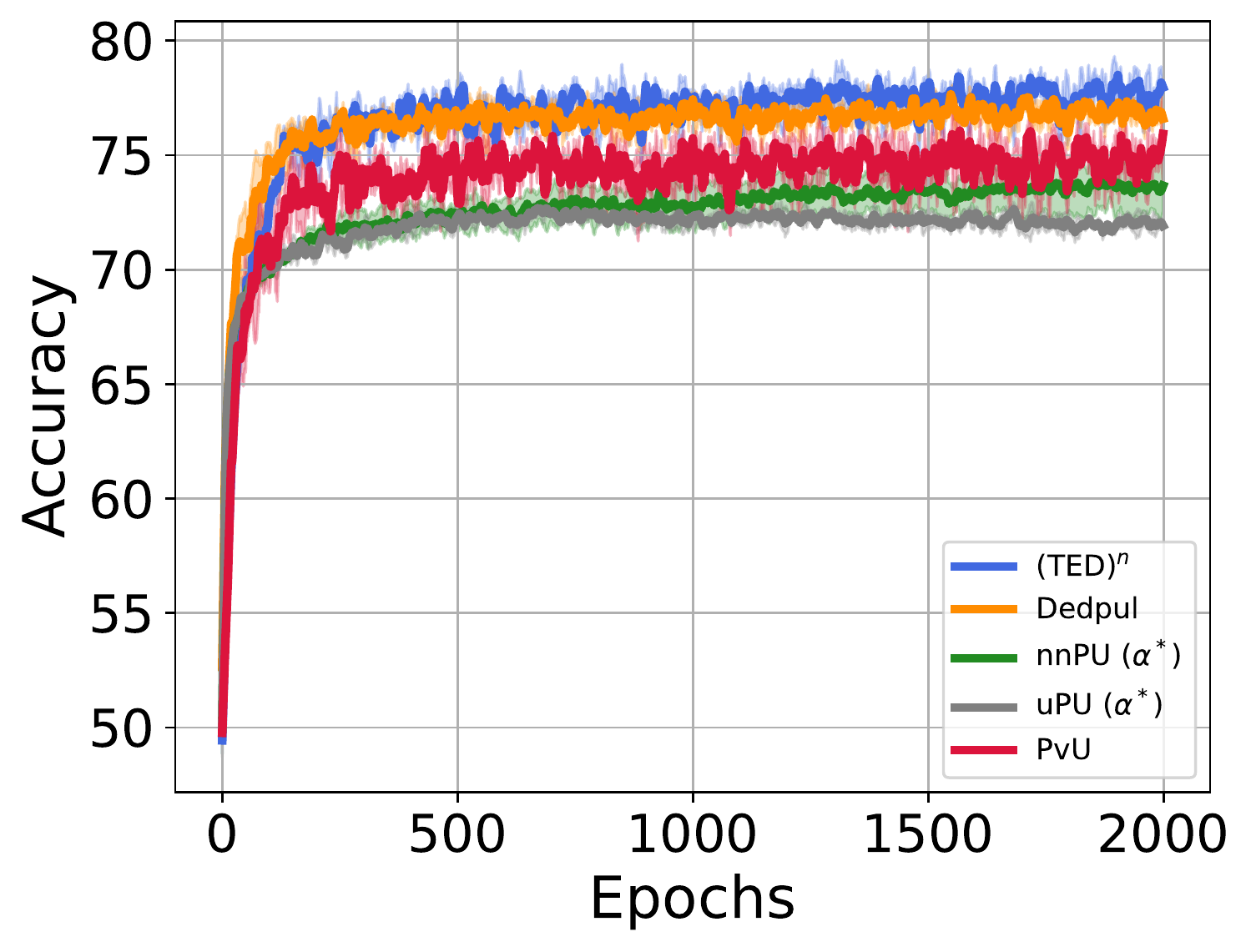}}\hfil
  \subfigure{\includegraphics[width=0.4\linewidth]{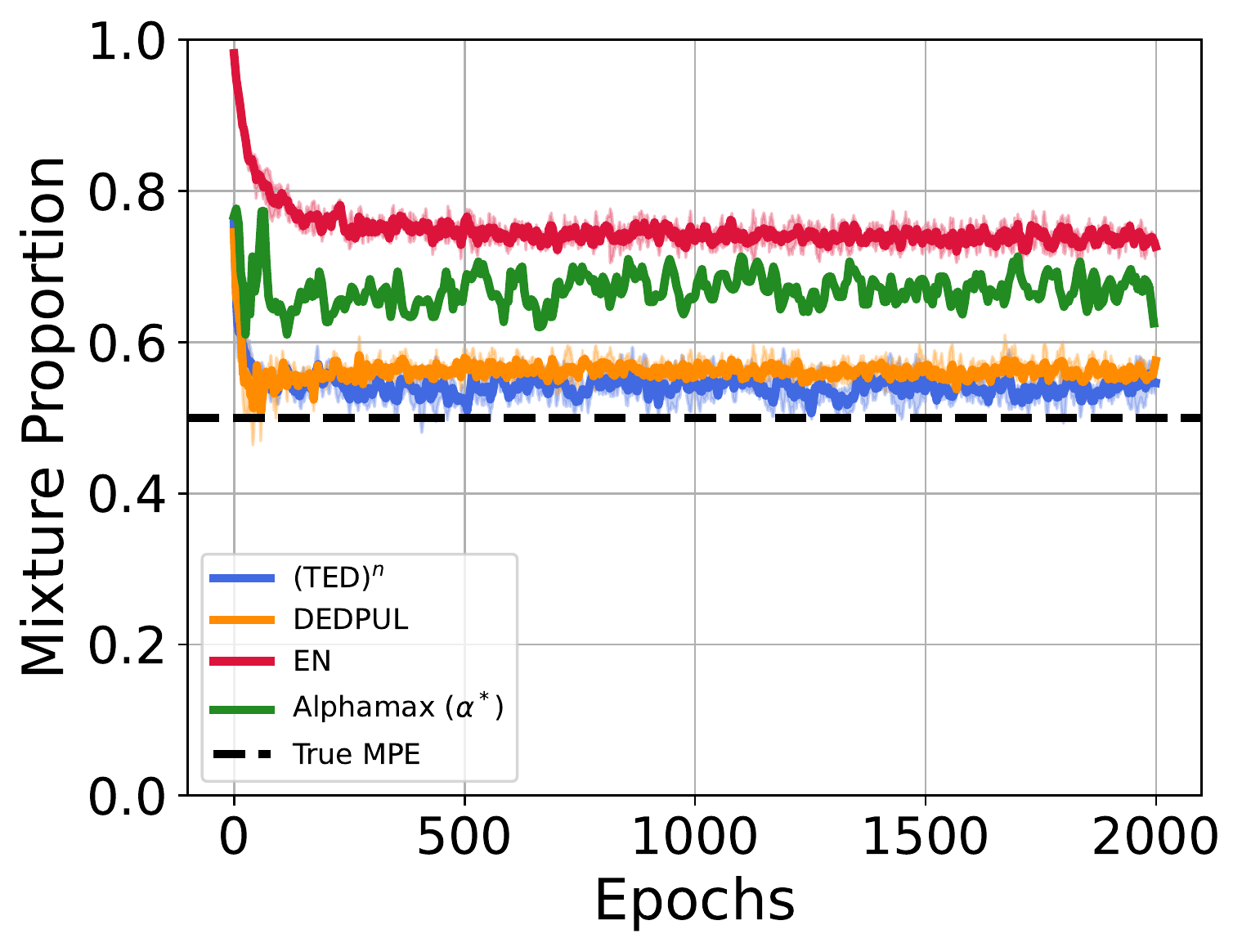}}
  \caption{ Epoch wise results with All convolutional network trained on CIFAR-binarized.}
  \label{fig:cifar_binarized_allconv_results}
\end{figure*}

\begin{figure*}[h!]
  \centering 
  \subfigure{\includegraphics[width=0.4\linewidth]{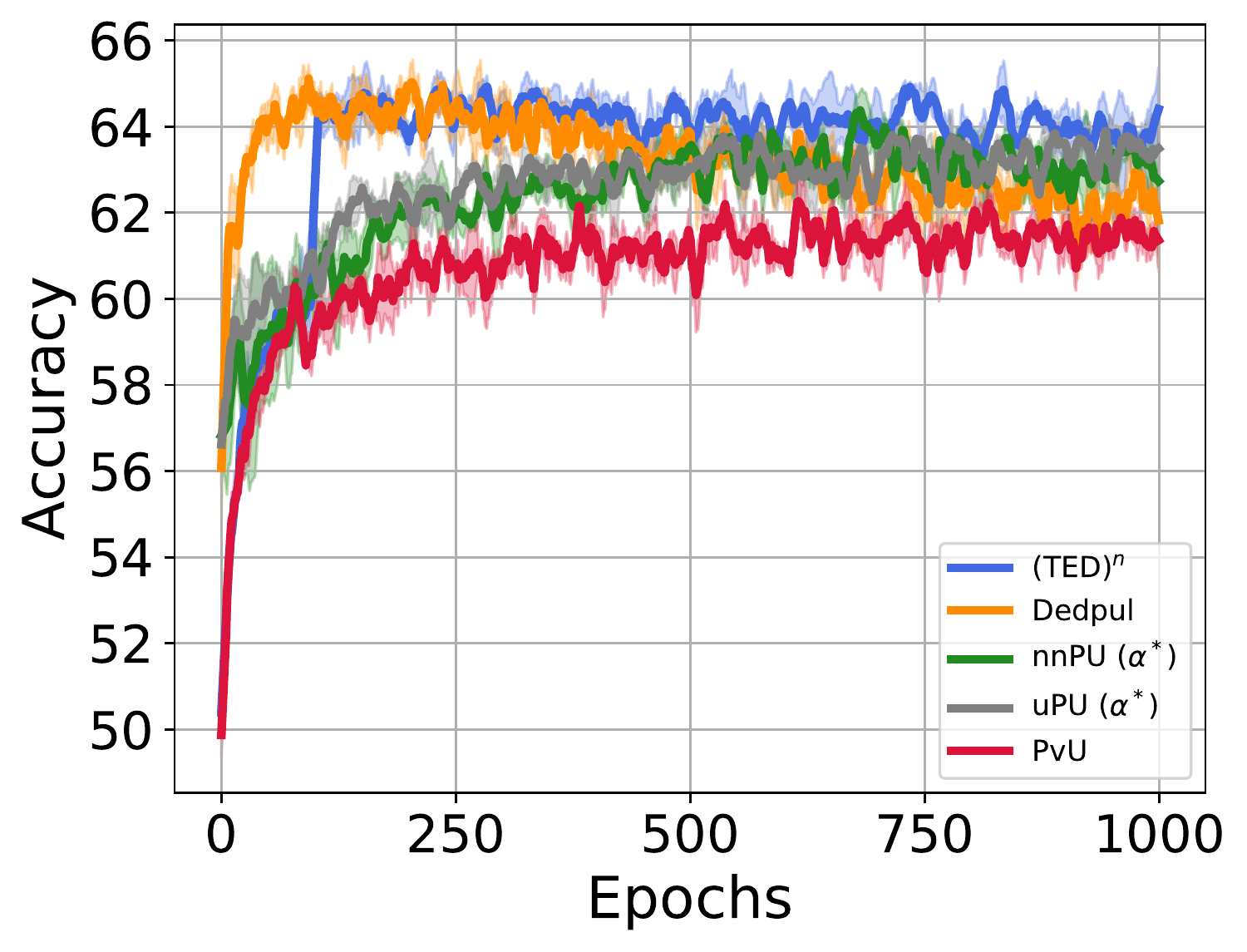}}\hfil
  \subfigure{\includegraphics[width=0.4\linewidth]{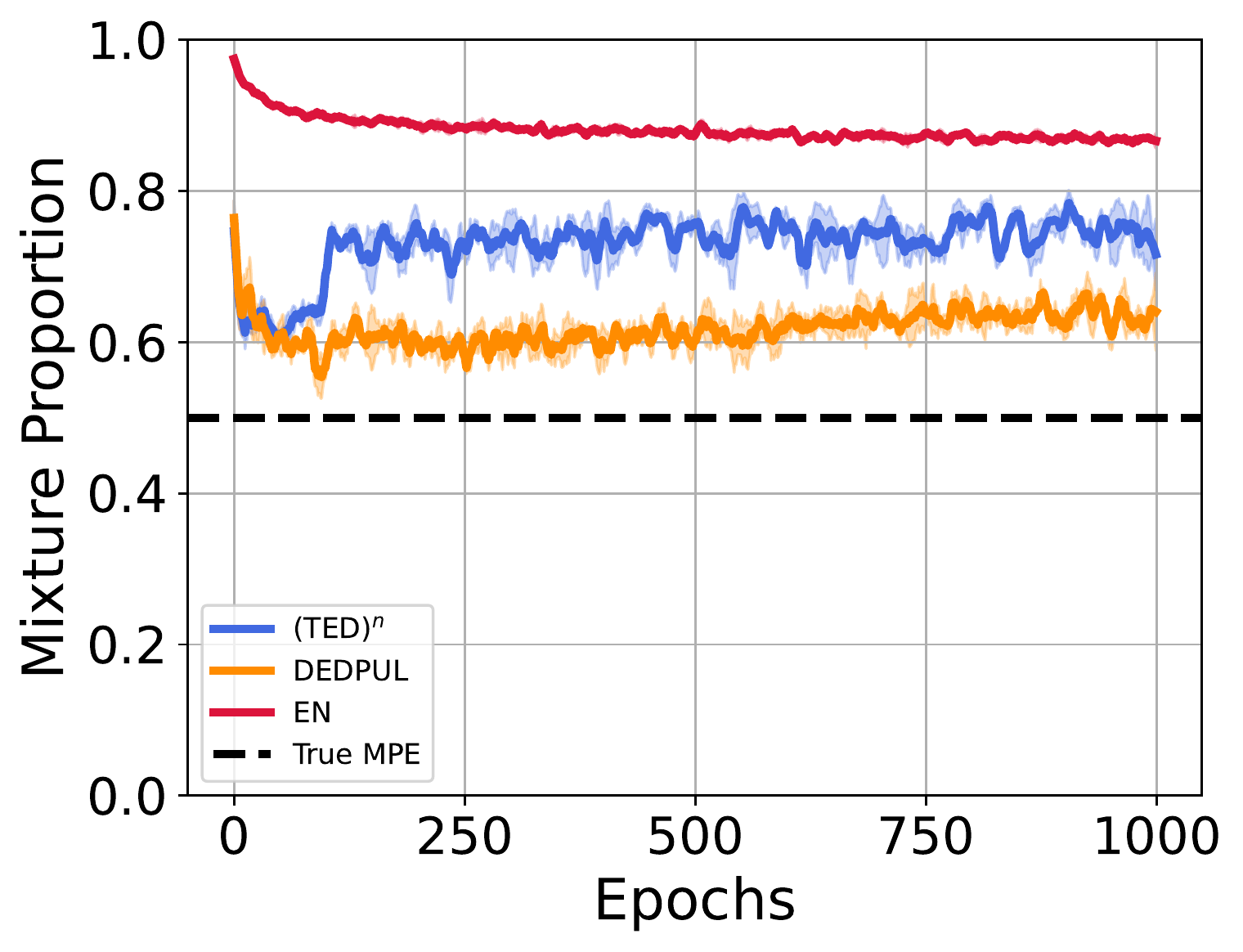}}
  \caption{ Epoch wise results with FCN trained on CIFAR-binarized.}
  \label{fig:cifar_binarized_FCN_results}
\end{figure*}

\begin{figure*}[h!]
    \centering 
    \subfigure{\includegraphics[width=0.4\linewidth]{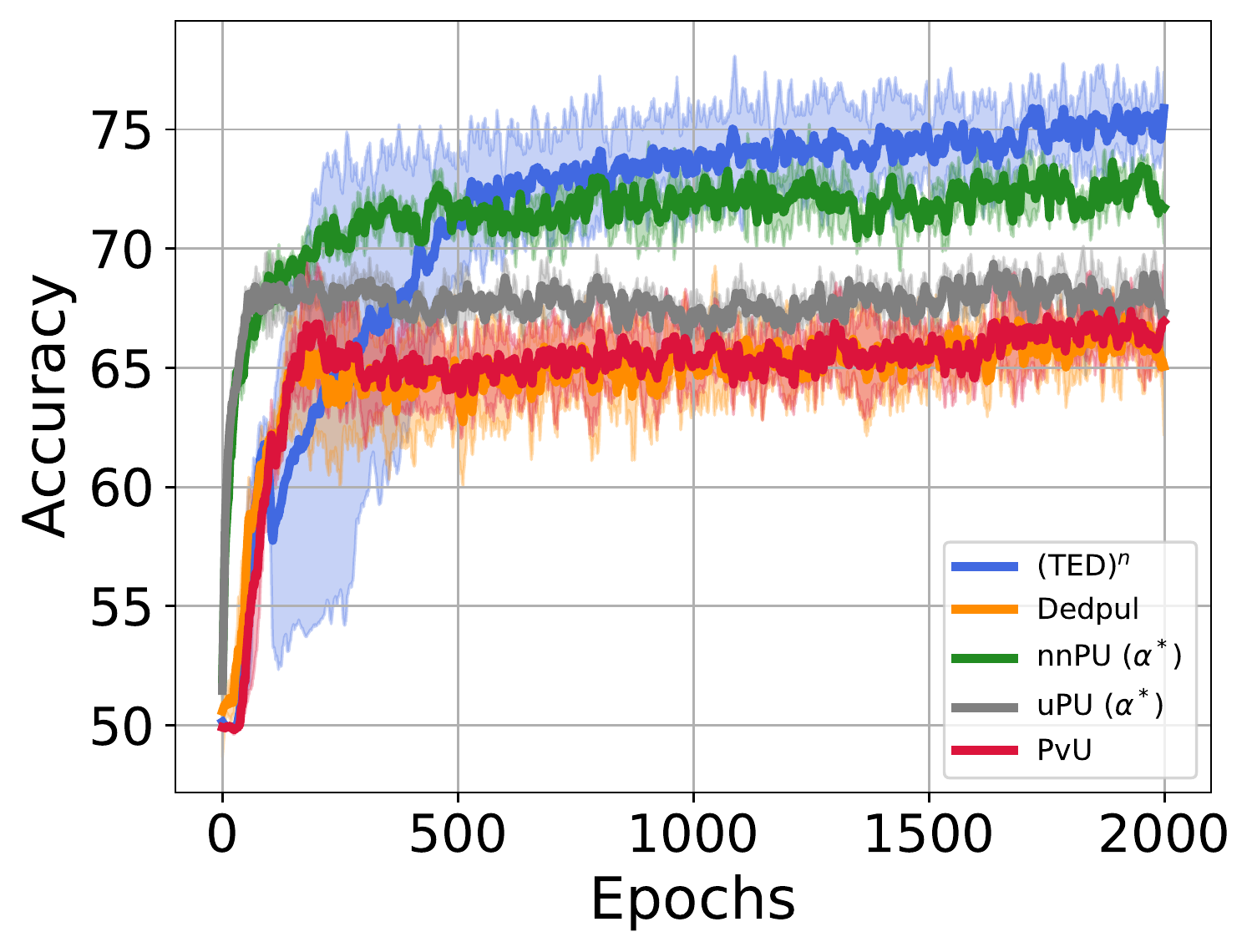}}\hfil
    \subfigure{\includegraphics[width=0.4\linewidth]{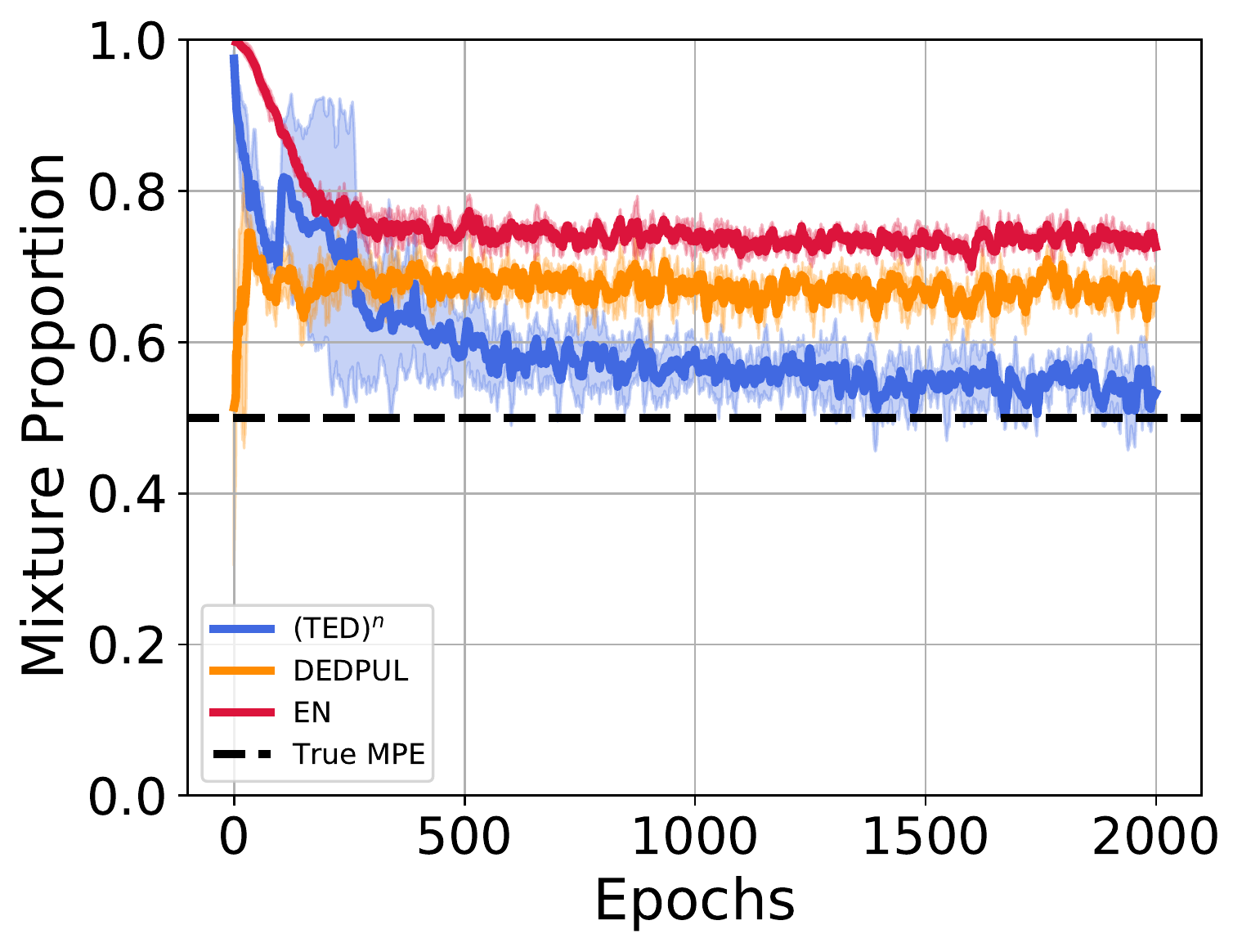}}
    \caption{ Epoch wise results with ResNet-18 trained on CIFAR Dog vs Cat.}
    \label{fig:cifar_dogcat_resnet_results}
  \end{figure*}

  \begin{figure*}[h!]
    \centering 
    \subfigure{\includegraphics[width=0.4\linewidth]{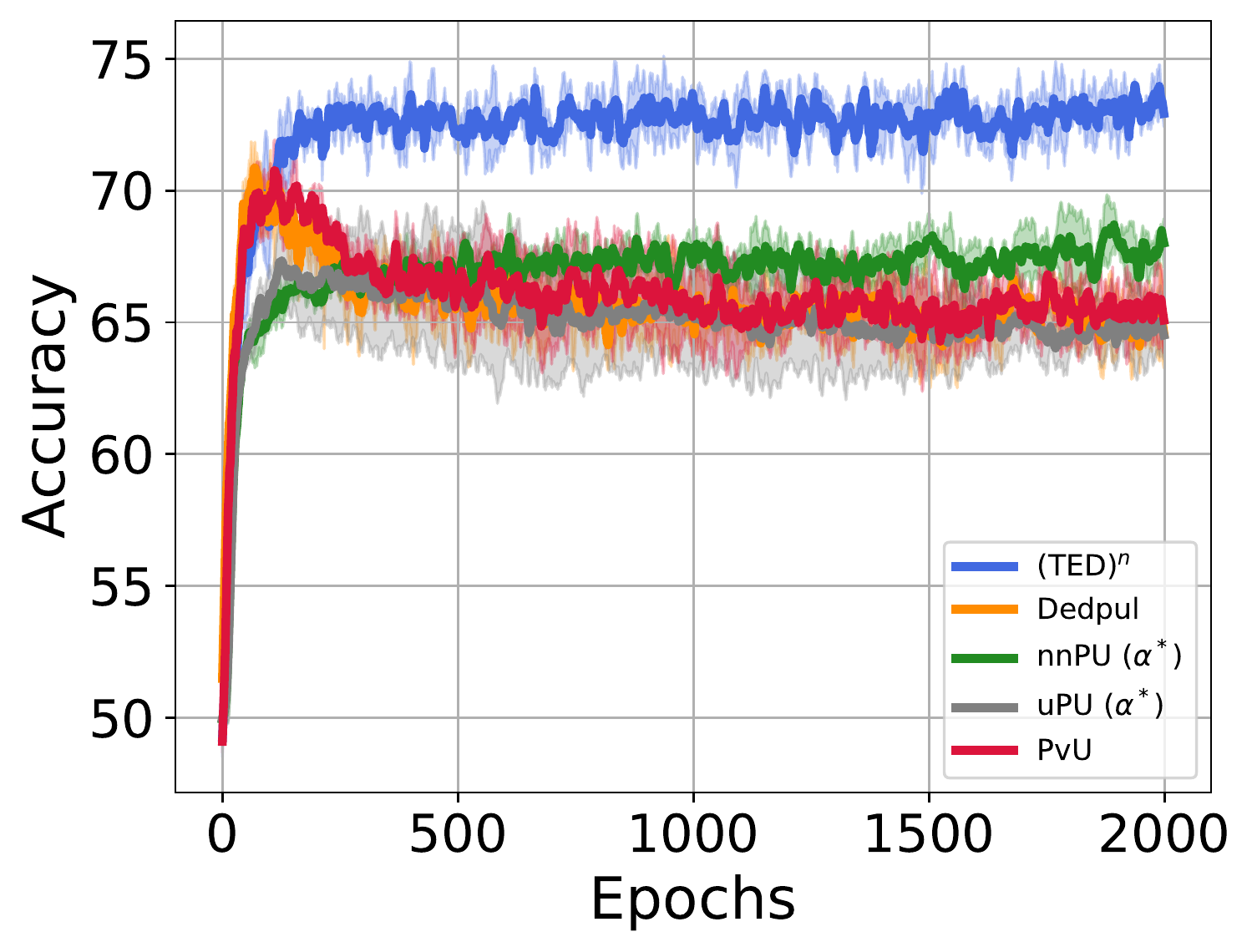}}\hfil
    \subfigure{\includegraphics[width=0.4\linewidth]{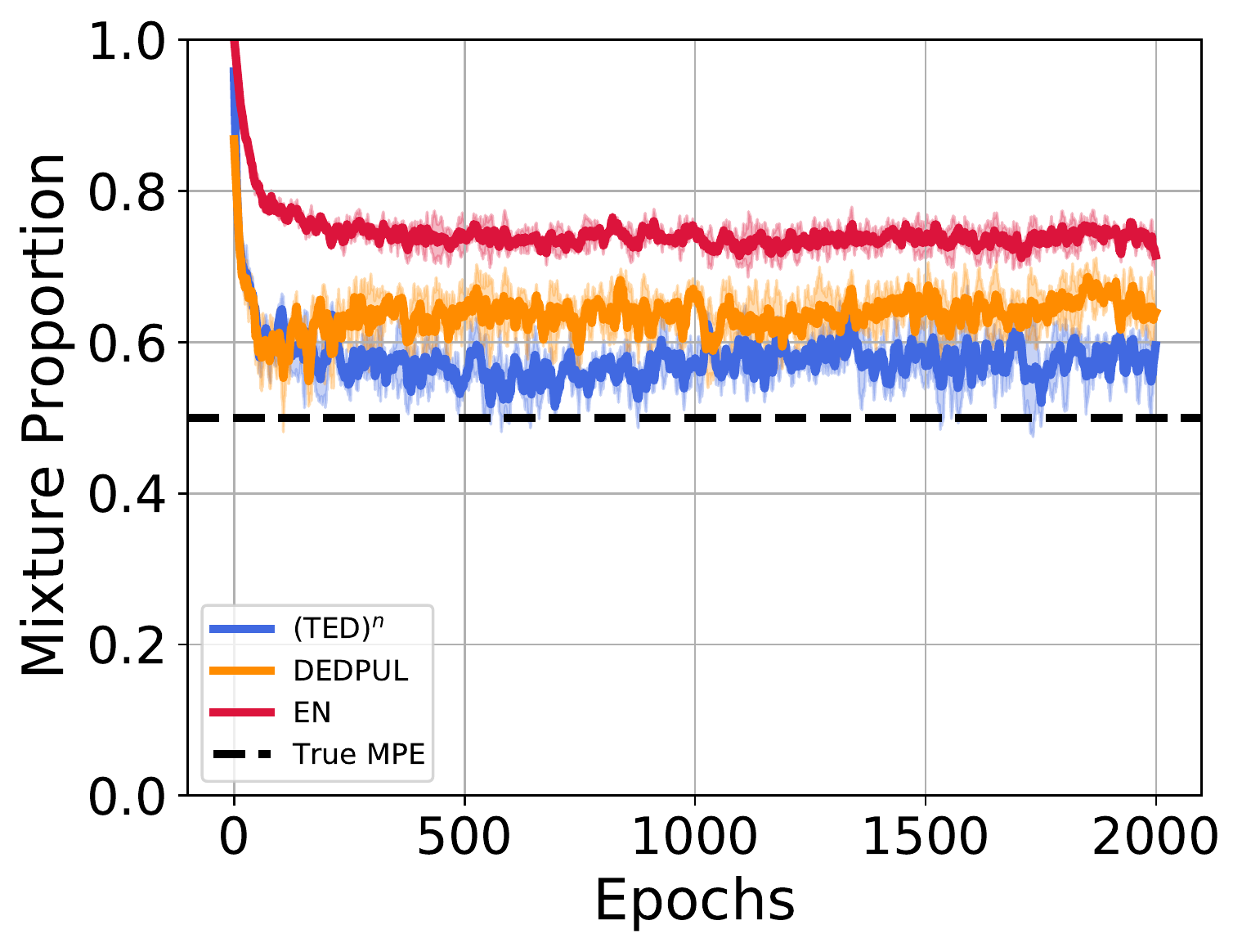}}
    \caption{ Epoch wise results with All convolutional network trained on CIFAR Dog vs Cat.}
    \label{fig:cifar_dogcat_allconv_results}
  \end{figure*}

  \begin{figure*}[h!]
    \centering 
    \subfigure{\includegraphics[width=0.4\linewidth]{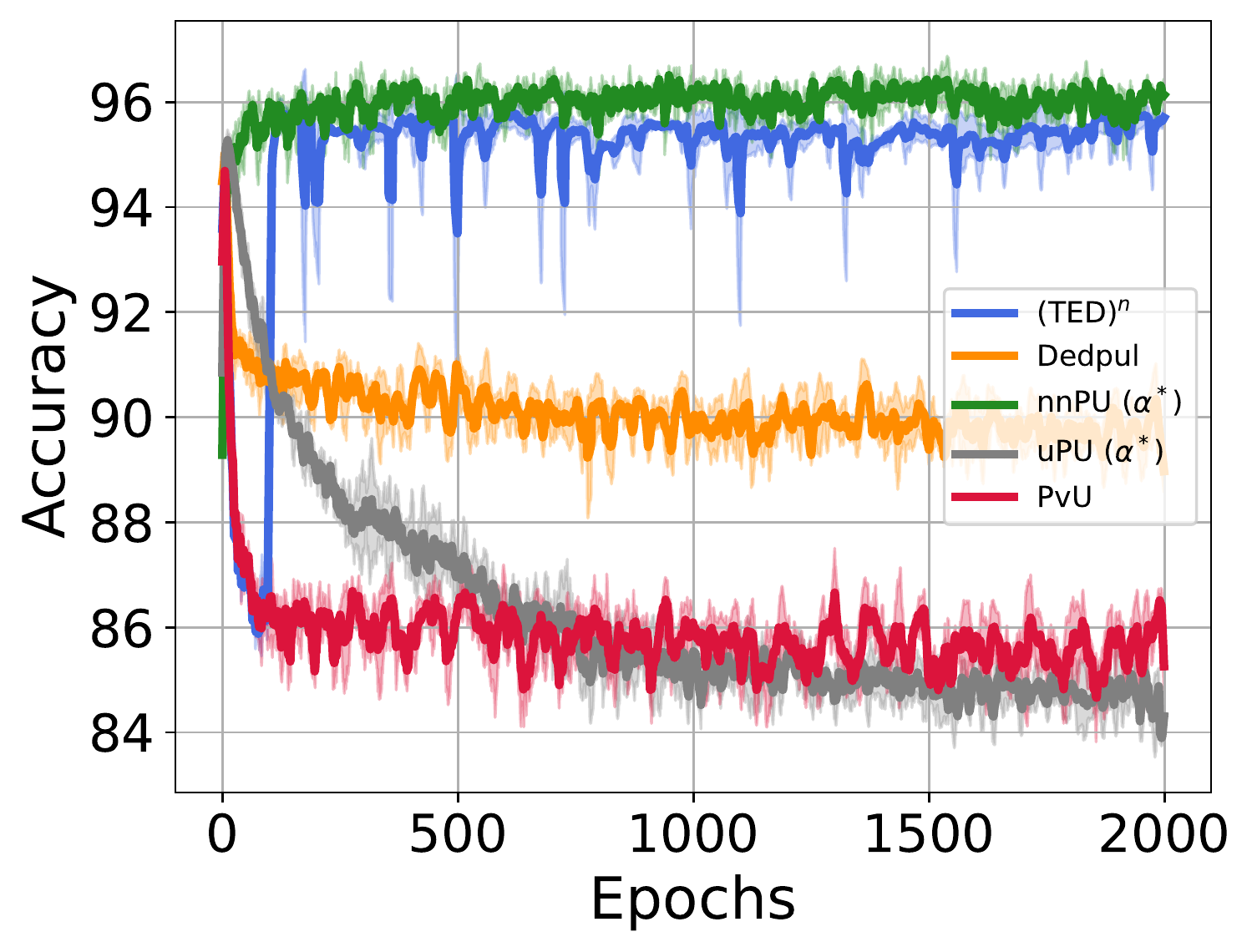}}\hfil
    \subfigure{\includegraphics[width=0.4\linewidth]{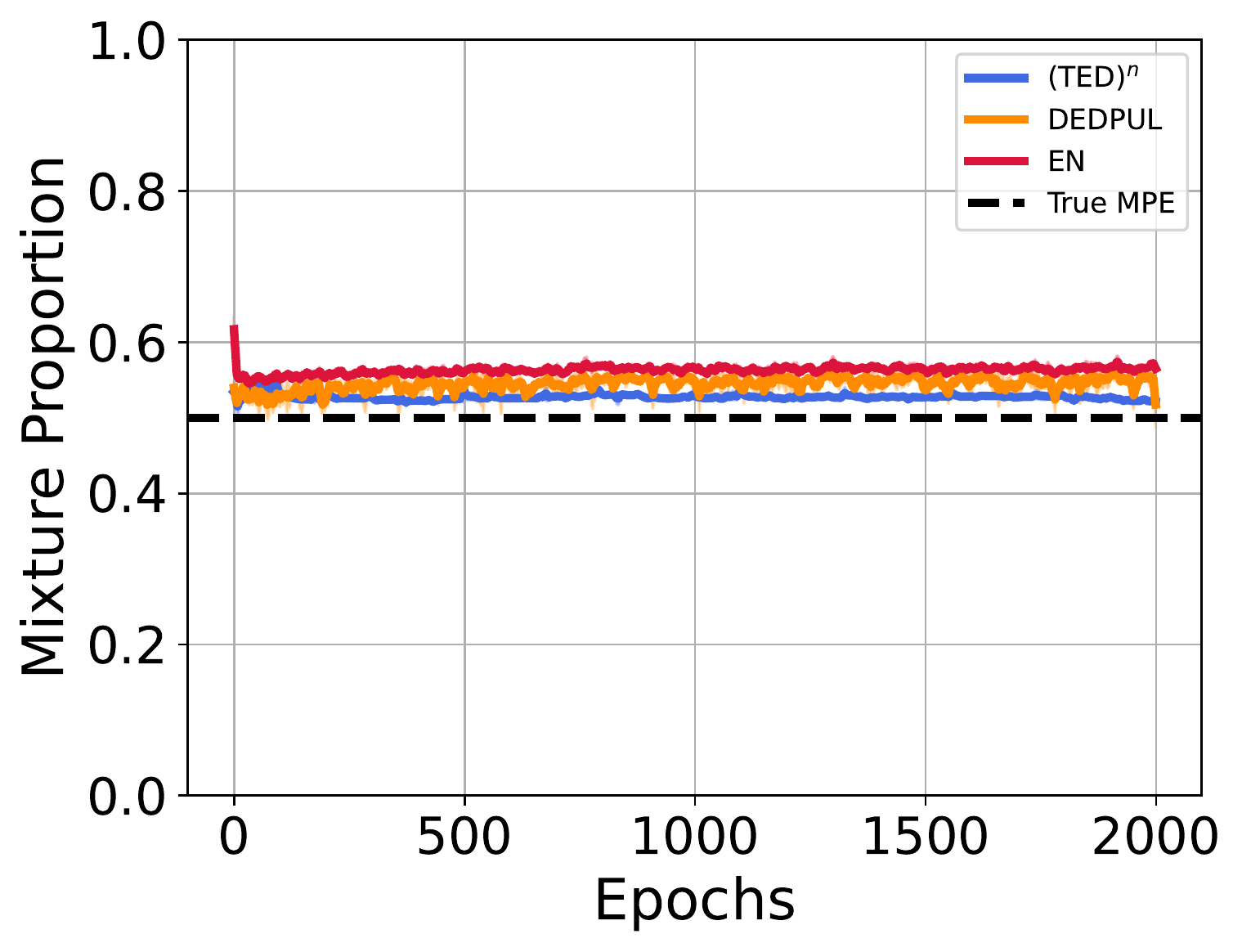}}
    \caption{ Epoch wise results with MLP trained on Binarized MNIST.}
    \label{fig:mnist_binary_results}
  \end{figure*}

  \begin{figure*}[h!]
    \centering 
    \subfigure{\includegraphics[width=0.4\linewidth]{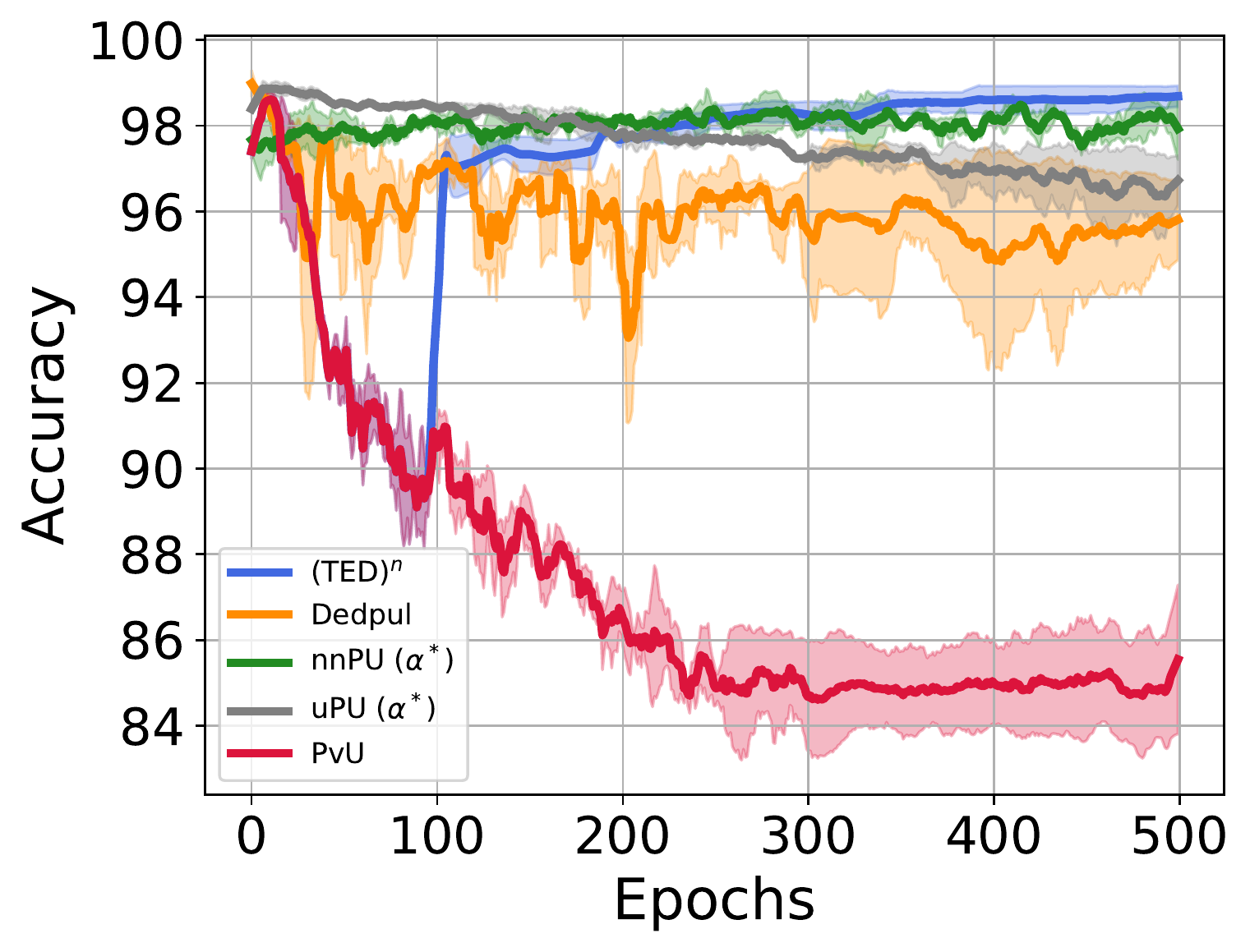}}\hfil
    \subfigure{\includegraphics[width=0.4\linewidth]{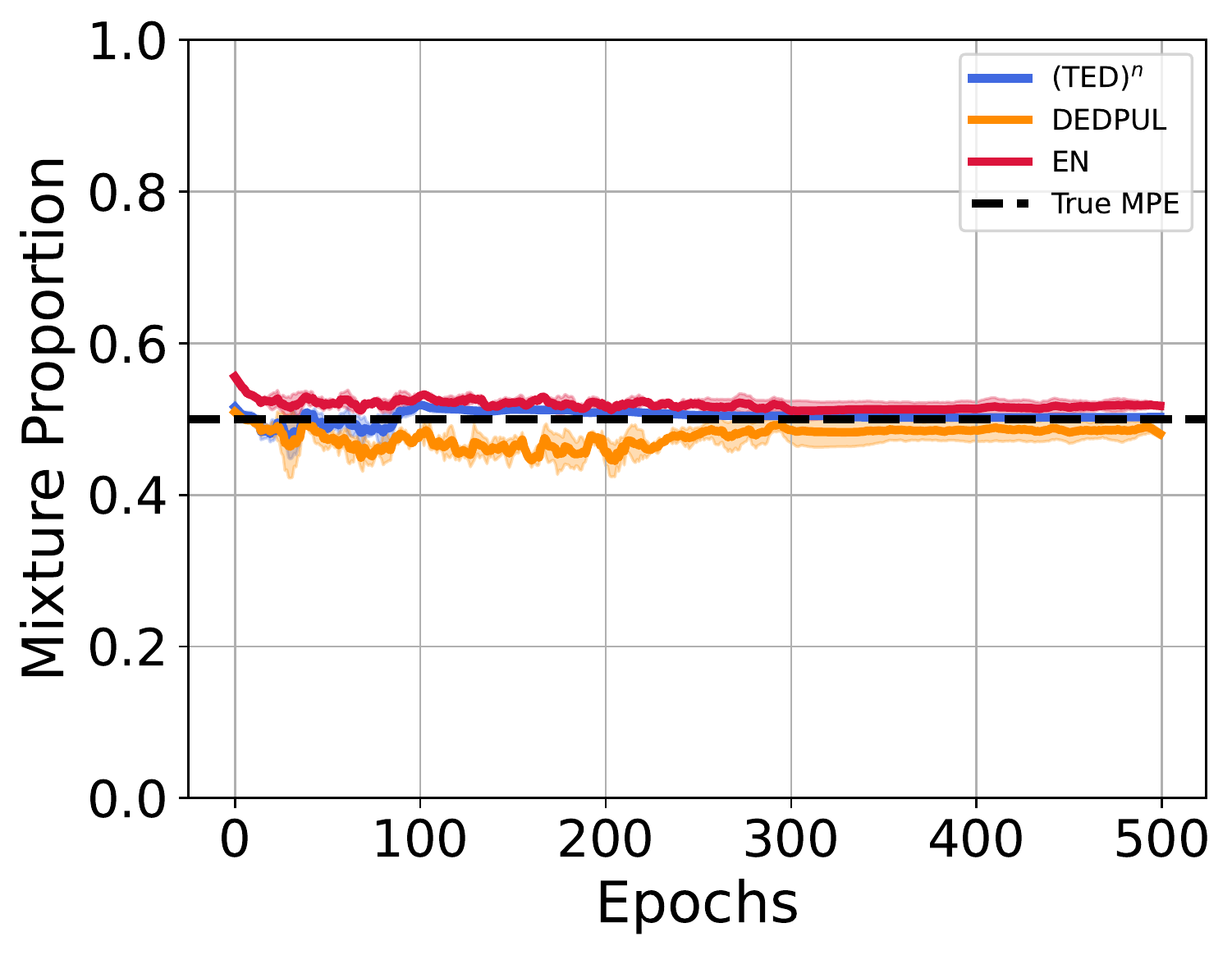}}
    \caption{ Epoch wise results with MLP trained on MNIST 17.}
    \label{fig:mnist_17_results}
  \end{figure*}

  \begin{figure*}[h!]
    \centering 
    \subfigure{\includegraphics[width=0.4\linewidth]{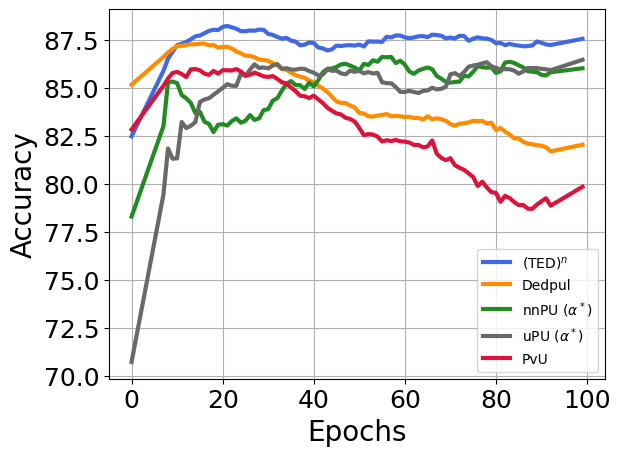}}\hfill
    \subfigure{\includegraphics[width=0.4\linewidth]{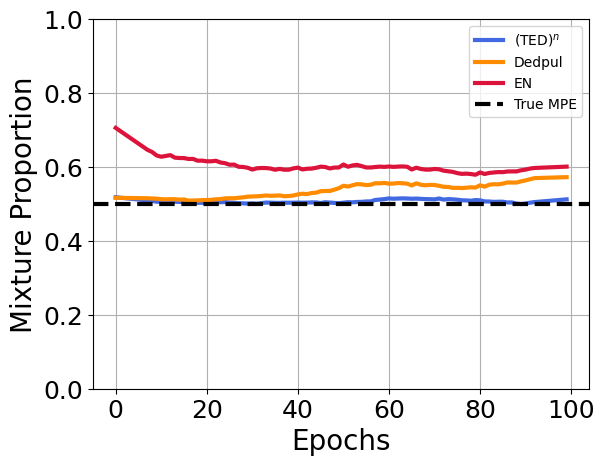}}
    \caption{ Epoch wise results with BERT trained on IMDb.}
    \label{fig:imdb_results}
  \end{figure*}

\newpage
\subsection{Overfitting on unlabeled data as PvU training proceeds} \label{ap:overfitting_unlabeled}

  \begin{figure*}[h!]
    \centering 
    \subfigure{\includegraphics[width=0.3\linewidth]{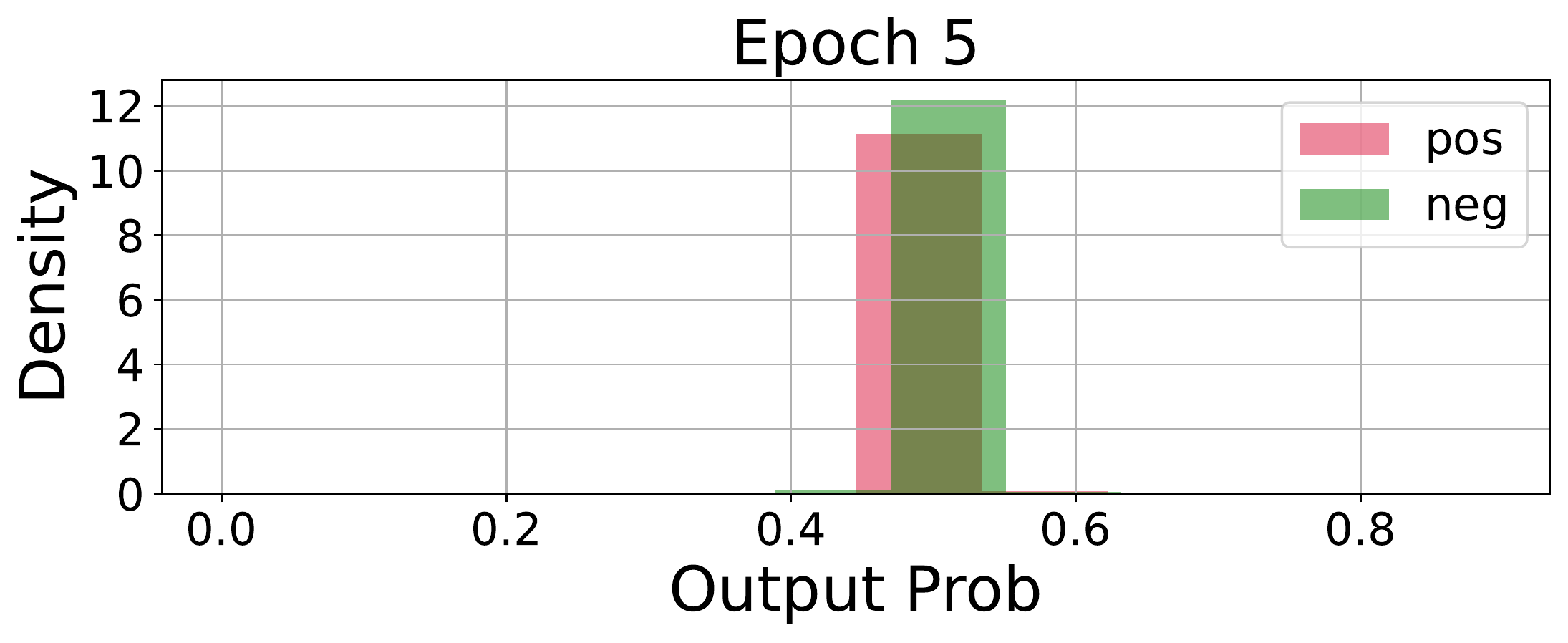}}\hfill
    \subfigure{\includegraphics[width=0.3\linewidth]{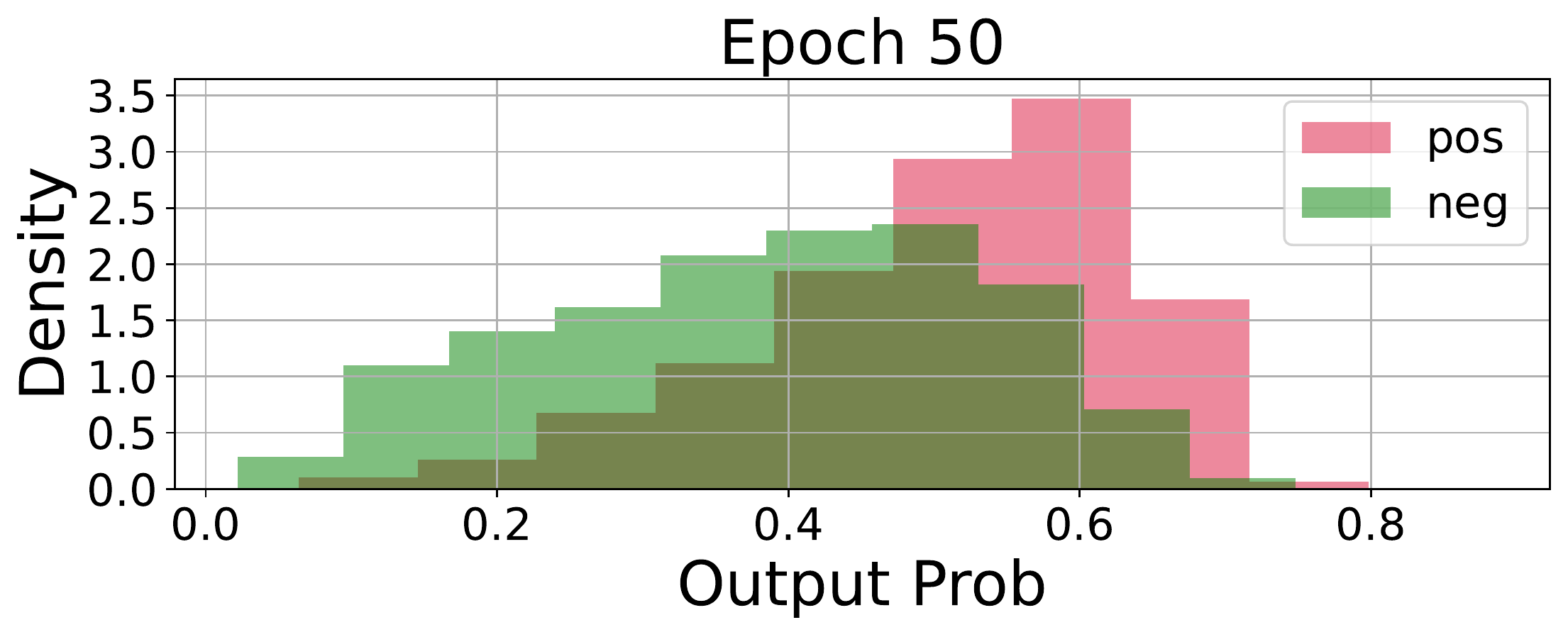}}\hfill
    \subfigure{\includegraphics[width=0.3\linewidth]{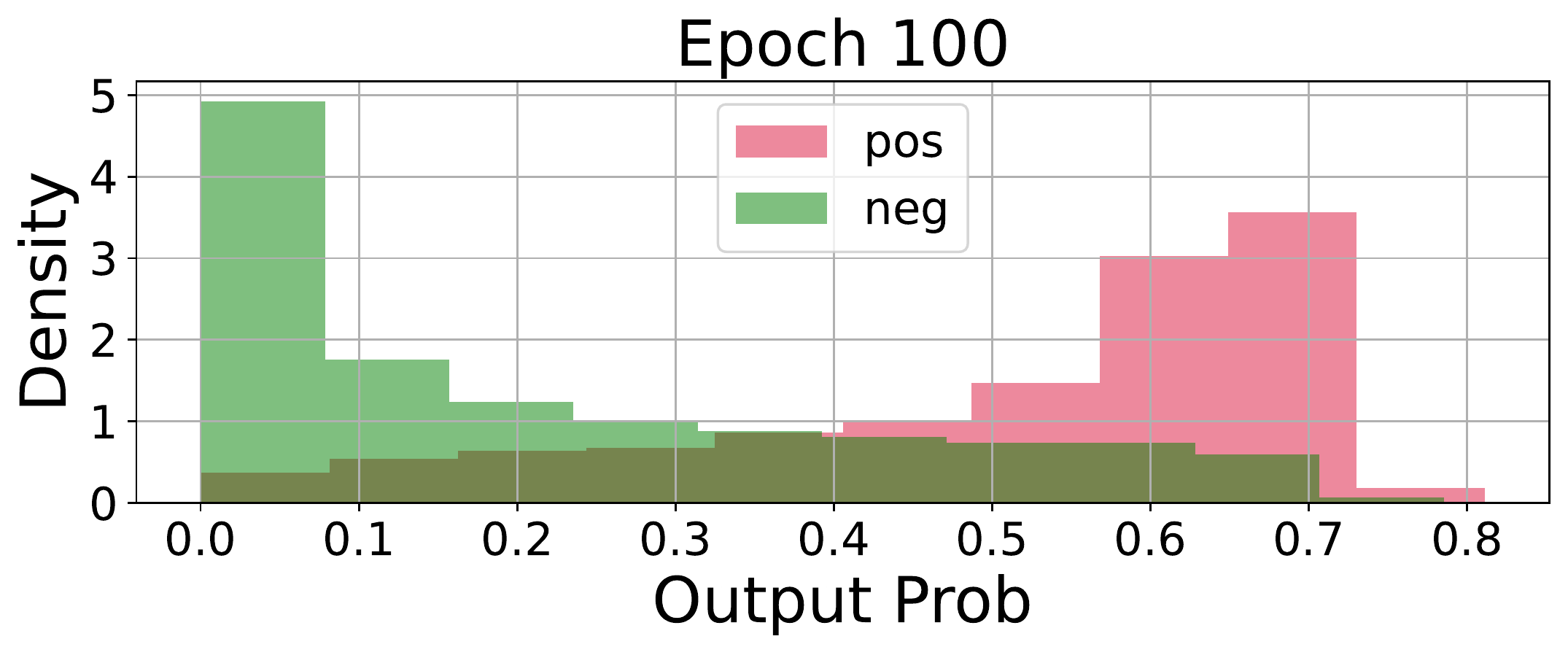}}
    \subfigure{\includegraphics[width=0.3\linewidth]{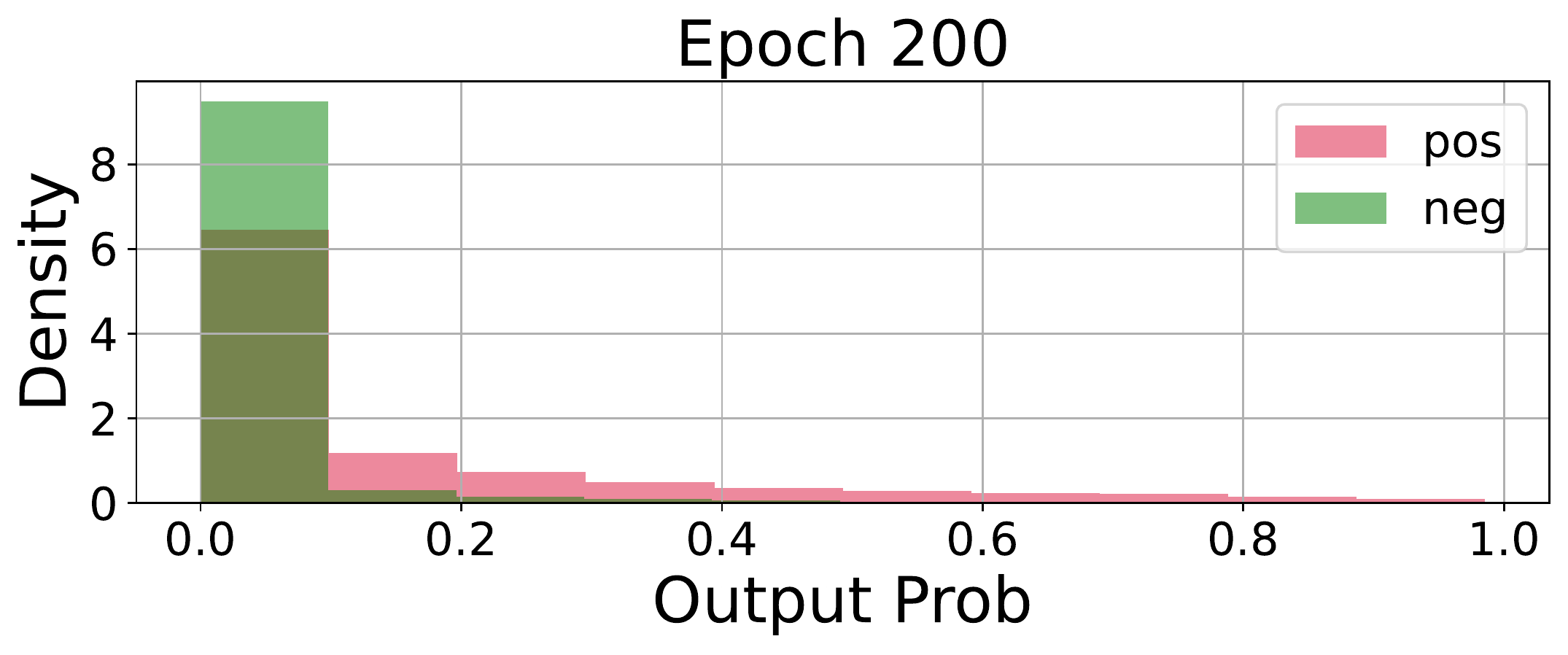}}\hfill
    \subfigure{\includegraphics[width=0.3\linewidth]{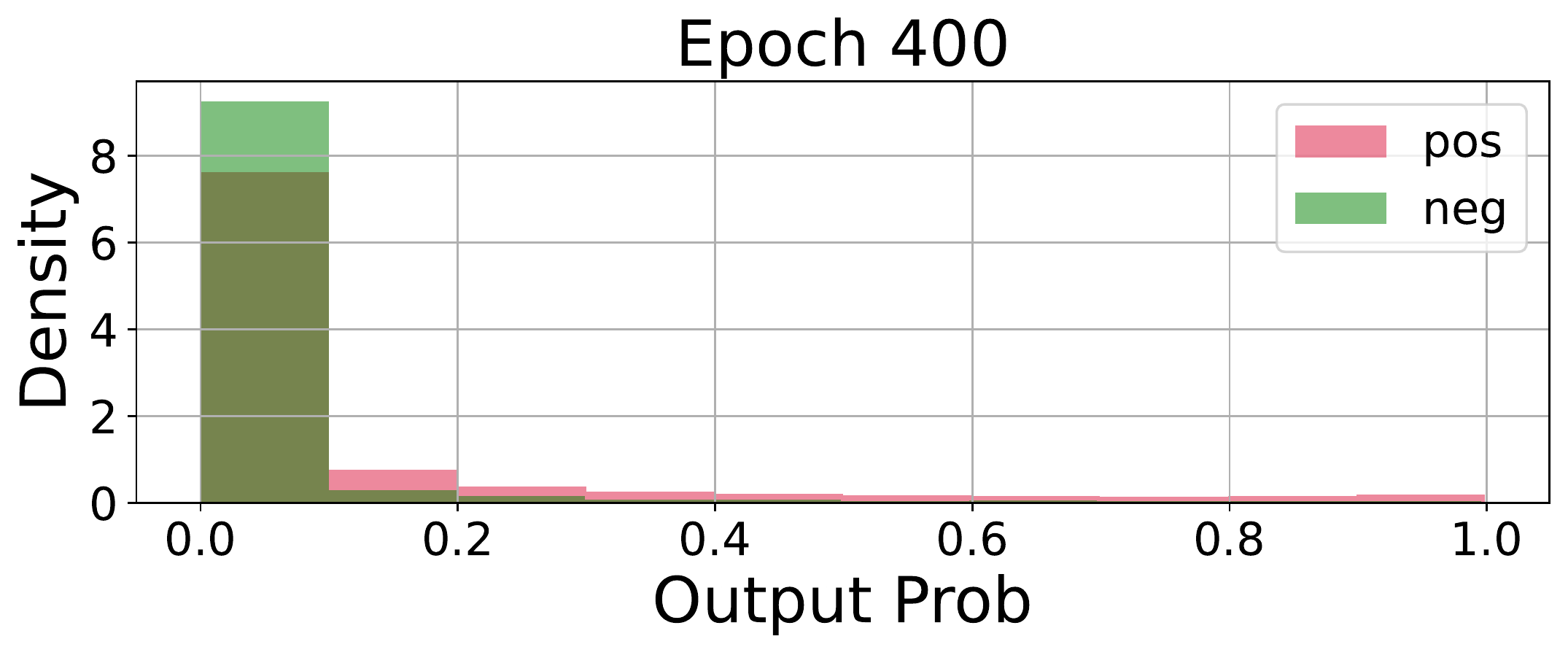}}\hfill
    \subfigure{\includegraphics[width=0.3\linewidth]{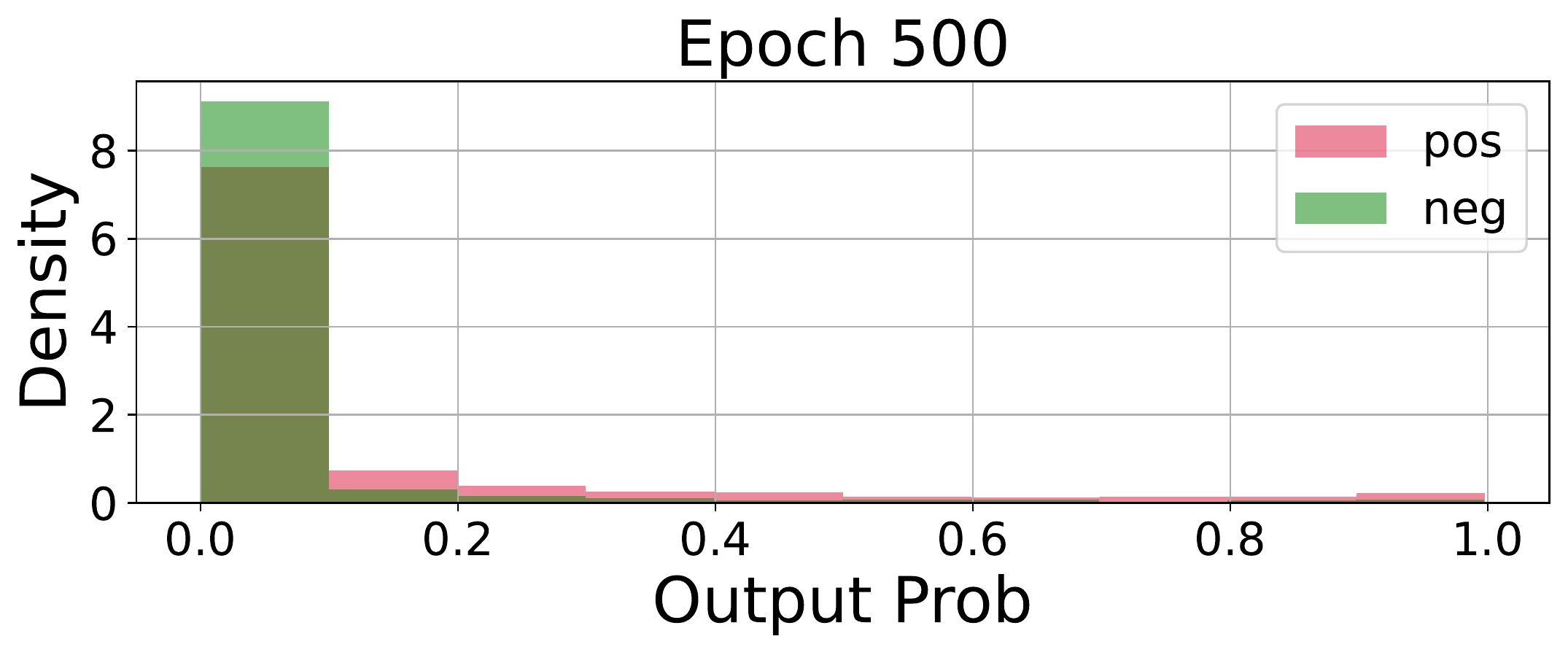}}
  
    \caption{Score assigned by the classifier to positive and negative points in the unlabeled training set as PvU training proceeds. As training proceeds, classifier memorizes both positive and negative in unlabeled as negatives. }
    \label{fig:PvU_training_overfit}
  \end{figure*}
  
In \figref{fig:PvU_training_overfit}, we show the distribution of unlabeled training points. We show that as positive versus unlabeled training proceeds  with a ResNet-18 model 
on binarized CIFAR dataset, 
classifier memorizes all the unlabeled data as negative 
assigning them very small scores (i.e., the probability of them being negative).


\subsection{Ablations to (TED)$^n$} \label{ap:ablation}

\textbf{Varying the number of warm start epochs {} {}}
We now vary the number of warm start epochs with (TED)$^n$. We observe that increasing the number of warm start epochs doesn't hurt (TED)$^n$ even when the classifier at the end of the warm start training memorized PU training data due PvU training. 
While in many cases (TED)$^n$ training without warm start is able to recover the same performance, it fails to learn anything for CIFAR Dog vs Cat with all convolutional neural network. This highlights the need for warm start training with (TED)$^n$.  

\begin{figure*}[h]
    \centering 
    \subfigure{\includegraphics[width=0.23\linewidth]{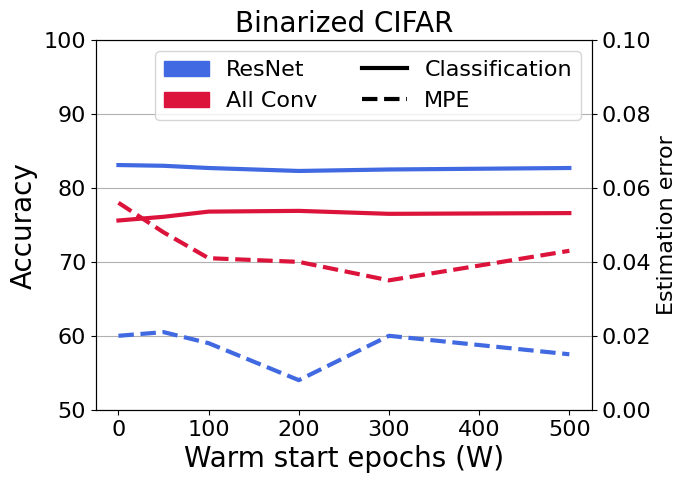}}\hfill
    \subfigure{\includegraphics[width=0.23\linewidth]{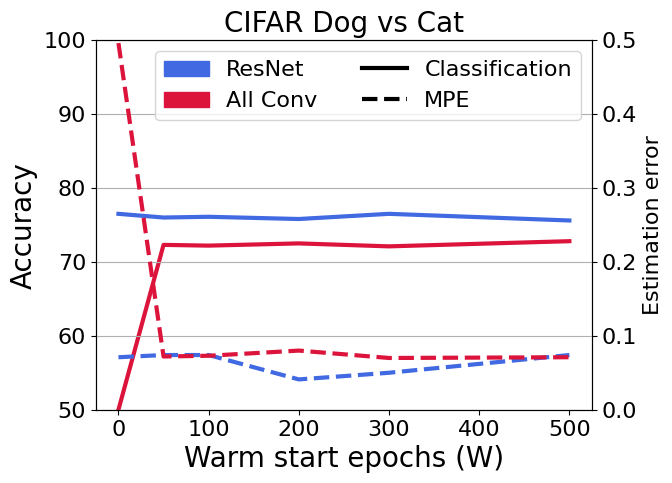}}\hfill
    \subfigure{\includegraphics[width=0.23\linewidth]{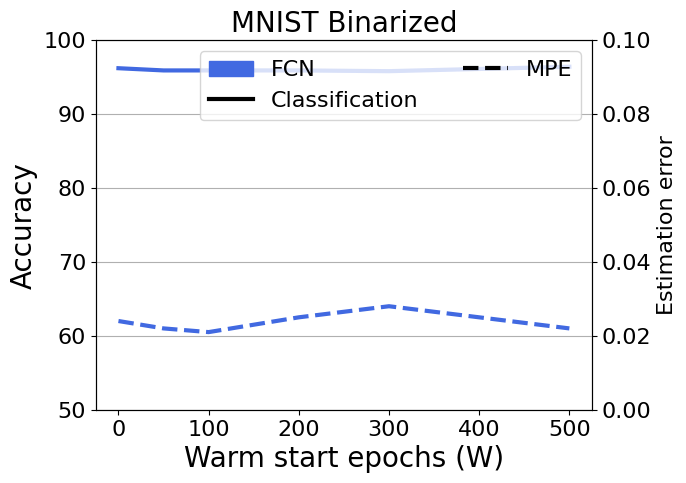}}\hfill
    \subfigure{\includegraphics[width=0.23\linewidth]{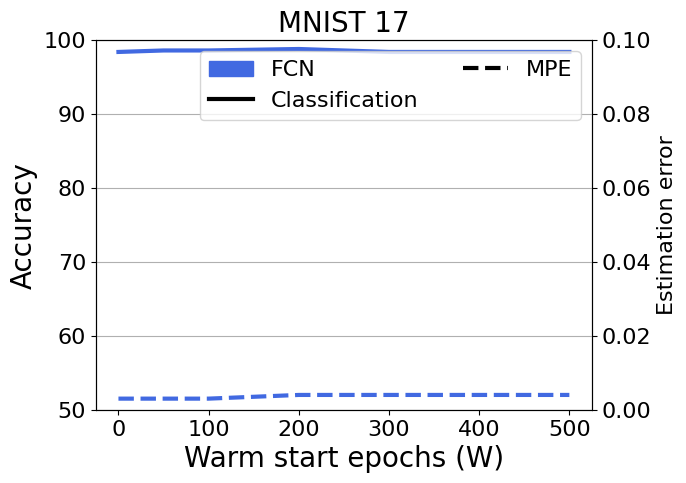}}
    \caption{Classification and MPE results with varying warm start epochs $W$ with (TED)$^n$}
    \label{fig:warm_start_epochs}
  \end{figure*}

 \textbf{Varying the true mixture proportion $\alpha$ {} {}} 
Next, we vary $\alpha$, the true mixture proportion and present results for MPE and  classification in \figref{fig:mpe_classification_alpha}. 
Overall, across all $\alpha$, our method (TED)$^n$ is able to achieve superior performance as compared to alternate algorithms. 
We omit high $\alpha$ for CIFAR and IMDb datasets as all the methods result in trivial accuracy and mixture proportion estimate. 

 \begin{figure*}[h]
    \centering 
    \subfigure{\includegraphics[width=0.19\linewidth]{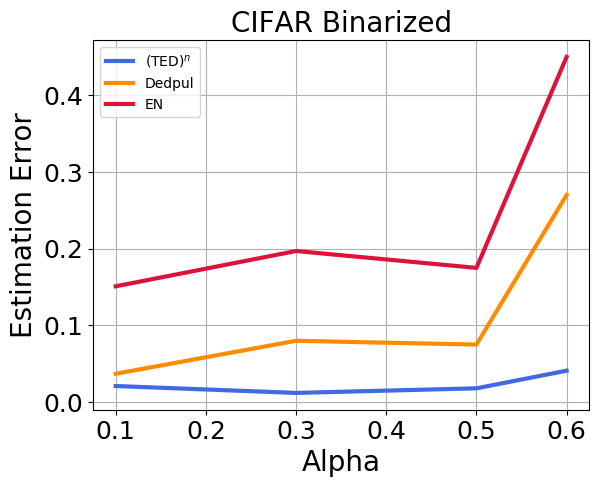}}\hfill
    \subfigure{\includegraphics[width=0.19\linewidth]{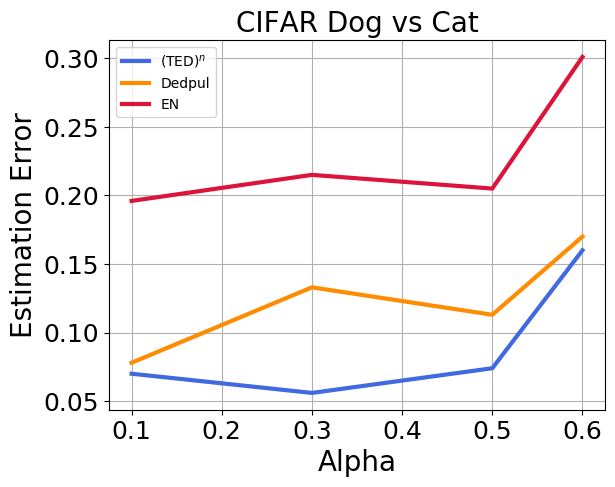}}\hfill
    \subfigure{\includegraphics[width=0.19\linewidth]{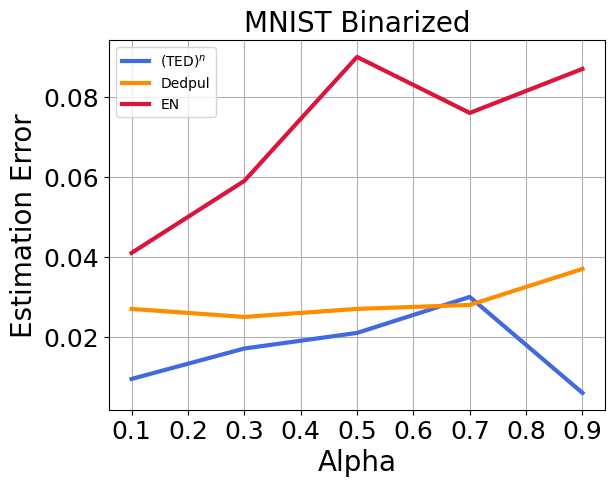}}\hfill
    \subfigure{\includegraphics[width=0.19\linewidth]{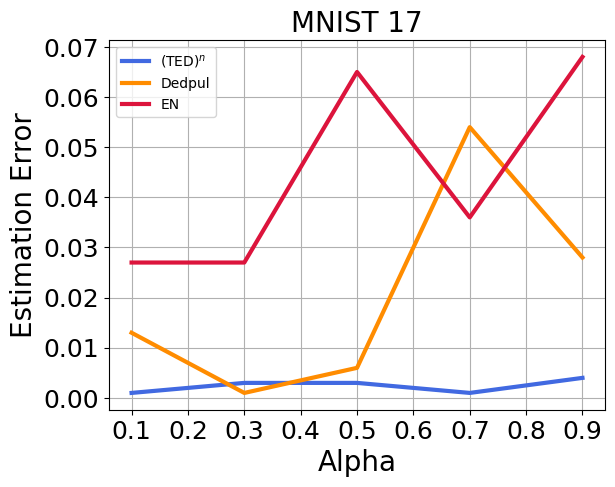}}
    \subfigure{\includegraphics[width=0.19\linewidth]{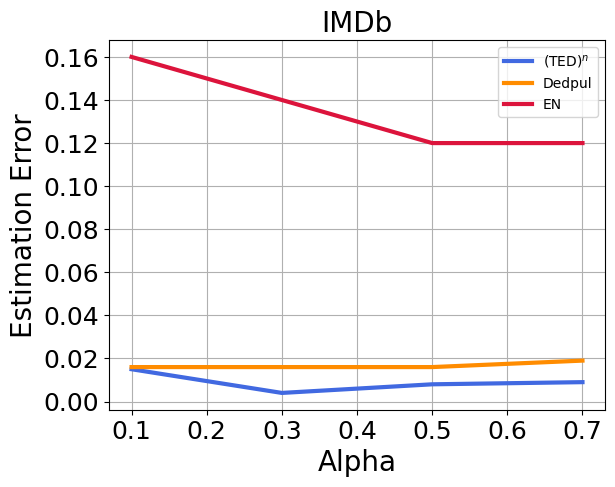}}
    \centering 
    \subfigure{\includegraphics[width=0.19\linewidth]{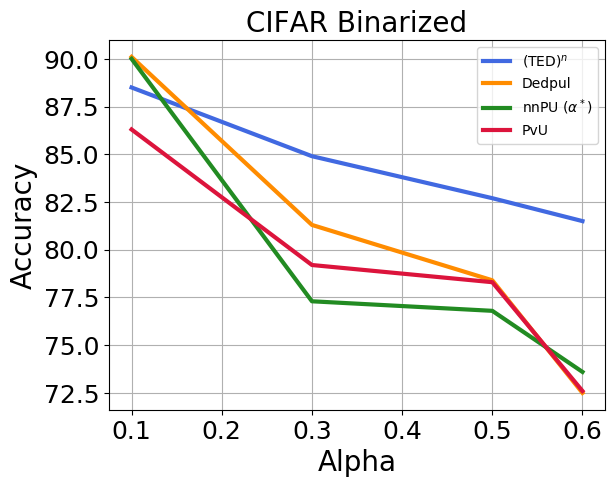}}\hfill
    \subfigure{\includegraphics[width=0.19\linewidth]{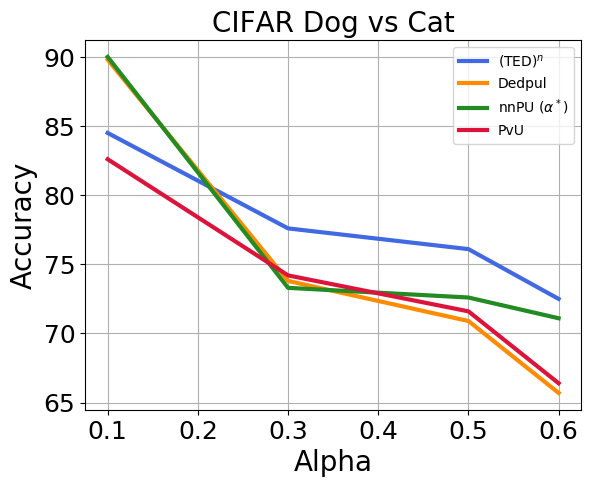}}\hfill
    \subfigure{\includegraphics[width=0.19\linewidth]{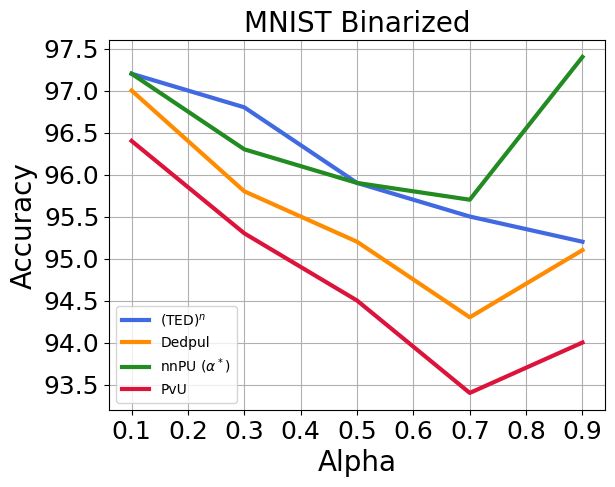}}\hfill
    \subfigure{\includegraphics[width=0.19\linewidth]{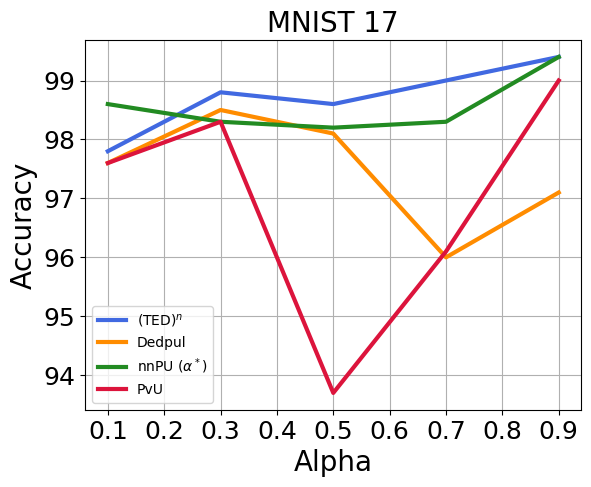}}
    \subfigure{\includegraphics[width=0.19\linewidth]{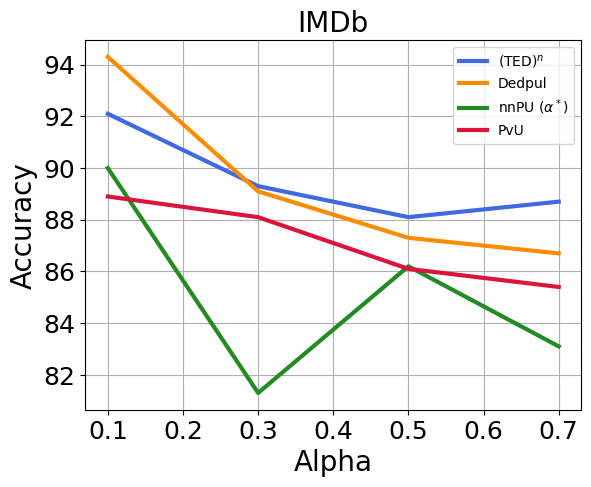}}
    \caption{MPE and Classification results with varying mixture proportion. For each method we show results with the best performing architecture.}
    \label{fig:mpe_classification_alpha}
  \end{figure*}

\subsection{Classification and MPE results with error bars} \label{app:results_std}

\begin{table}[h]
  \begin{adjustbox}{width=\columnwidth,center}
  \centering
  \small
  \tabcolsep=0.12cm
  \renewcommand{\arraystretch}{1.2}
  \begin{tabular}{@{}*{8}{c}@{}}
  \toprule
  Dataset & Model  & \thead{(TED)$^n$} &  \thead{BBE$^*$} & \thead{DEDPUL$^*$} & \thead{EN} & \thead{KM2} & \thead{TiCE}\\
  \midrule
  \multirow{3}{*}{ \parbox{1.5cm}{\centering Binarized CIFAR} }  & ResNet & $\mathbf{0.026 \pm 0.005}$  & $0.091 \pm 0.027 $ & $0.091 \pm 0.023  $ & $0.192 \pm 0.007 $ & &  \\
  & All Conv & $0.042 \pm 0.003 $  & $\mathbf{0.037 \pm 0.018}$ & $0.052 \pm 0.017  $ & $0.221 \pm 0.017 $  & $0.168 \pm 0.207$ &  $0.194 \pm 0.039$  \\
  & MLP & $0.225 \pm 0.013 $  & $0.177 \pm 0.011 $ & $\mathbf{0.138 \pm 0.009}$ & $0.372 \pm 0.002 $  & &  \\
  \midrule
  \multirow{2}{*}{ \parbox{1.5cm}{\centering CIFAR Dog vs Cat} }  & ResNet & $\mathbf{0.078 \pm 0.010}$  & $0.176 \pm 0.015 $ & $0.170 \pm 0.010  $ & $0.226 \pm 0.003 $ & $0.331 \pm 0.238$ & $0.286 \pm 0.013$ \\
  & All Conv & $\mathbf{0.066 \pm 0.015}$  & $0.128 \pm 0.020 $ & $0.115 \pm 0.014  $ & $0.250 \pm 0.019 $ & & \\
  \midrule 
  \multirow{1}{*}{ \parbox{2.5cm}{\centering Binarized MNIST} } & MLP & $\mathbf{0.024 \pm 0.001}$  & $0.032 \pm 0.001 $ & $0.031 \pm 0.003  $ & $0.080 \pm 0.009 $ & $0.029 \pm 0.008$ & $0.056 \pm 0.05$  \\
  \midrule
  \multirow{1}{*}{ \parbox{1.5cm}{\centering MNIST17} }  & MLP & $\mathbf{0.003 \pm 0.000}$  & $0.023 \pm 0.017 $ & $0.021 \pm 0.011  $ & $0.028 \pm 0.017 $ & $0.022 \pm 0.003$ & $0.043 \pm 0.023$ \\
  \midrule
  \multirow{1}{*}{ \parbox{1.5cm}{\centering IMDb} }  & BERT & $\mathbf{0.008 \pm 0.001}$ & $0.011 \pm 0.002$ & $0.016 \pm 0.005$ & $0.07 \pm 0.01$ & - & -\\
  \bottomrule 
  \end{tabular}  
  \end{adjustbox}
  \vspace{5pt}
  \caption{Absolute estimation error when $\alpha$ is 0.5. "*" denote oracle early stopping as defined in \secref{sec:exp}.
  Results reported by aggregating absolute error over 10 epochs and 3 seeds.
  }\label{table:MPE_error_bar}
\end{table}

\begin{table}[h]
  \begin{adjustbox}{width=\columnwidth,center}
  \centering
  \small
  \tabcolsep=0.12cm
  \renewcommand{\arraystretch}{1.2}
  \begin{tabular}{@{}*{8}{c}@{}}
  \toprule
  Dataset & Model  & \thead{(TED)$^n$ \\(unknown $\alpha$)}  & \thead{CVIR\\(known $\alpha$)} & \thead{PvU$^*$ \\(known $\alpha$)} & \thead{DEDPUL$^*$ \\(unknown $\alpha$)}  & \thead{nnPU \\(known $\alpha$)} & \thead{uPU$^*$ \\(known $\alpha$)} \\
  \midrule
  \multirow{3}{*}{ \parbox{1.5cm}{\centering Binarized CIFAR} }  & ResNet & $\mathbf{82.7 \pm 0.13}$  & $82.3 \pm 0.18$ & $76.9 \pm 1.12$ & $77.1 \pm 1.52$ & $77.2 \pm 1.03$ & $76.7 \pm 0.74$    \\
  & All Conv & $77.9 \pm 0.29$  & $\mathbf{78.1 \pm 0.47}$ & $75.8 \pm 0.75$ & $77.1 \pm 0.64$ & $73.4 \pm 1.31$ & $72.5 \pm 0.21$   \\
  & MLP & $64.2 \pm 0.37$  & $\mathbf{66.9 \pm 0.28}$ & $61.6 \pm 0.38$ & $62.6 \pm 0.30$ & $63.1 \pm 0.79$ & $64.0 \pm 0.24$ \\
  \midrule
  \multirow{2}{*}{ \parbox{1.5cm}{\centering CIFAR Dog vs Cat} }  & ResNet & $\mathbf{75.2 \pm 1.74}$  & $73.3 \pm 0.94$ & $67.3 \pm 1.52$ & $67.0 \pm 1.46$ & $71.8 \pm 0.33$ & $68.8 \pm 0.53$  \\
  & All Conv & $\mathbf{73.0 \pm 0.81}$  & $71.7 \pm 0.47$ & $70.5 \pm 0.60$ & $69.2 \pm 0.86$ & $67.9 \pm 0.52$ & $67.5 \pm 2.28$ \\
  \midrule 
  \multirow{1}{*}{ \parbox{1.5cm}{\centering Binarized MNIST} } & MLP & $95.6 \pm 0.42$  & $\mathbf{96.3 \pm 0.07}$ & $94.2 \pm 0.58$ & $94.8 \pm 0.10$ & $96.1 \pm 0.14$ & $95.2 \pm 0.19$  \\
  \addlinespace[0.2cm]
  \midrule
  \multirow{1}{*}{ \parbox{1.5cm}{\centering MNIST17} }  & MLP & $\mathbf{98.7 \pm 0.25}$  & $\mathbf{98.7 \pm 0.09}$ & $96.9 \pm 1.51$ & $97.7 \pm 0.62$  & $98.4 \pm 0.20$ & $98.4 \pm 0.09$ \\
  \midrule
  \multirow{1}{*}{ \parbox{1.5cm}{\centering IMDb} }  & BERT & $\mathbf{87.6 \pm 0.20}$ & $87.4 \pm 0.25$ & $86.1 \pm 0.53$ & $87.3\pm 0.18$ & $86.2 \pm 0.25$  & $85.9 \pm 0.12$ \\
  \bottomrule 
  \end{tabular}  
  \end{adjustbox}
  \vspace{5pt}
  \caption{Accuracy for PvN classification with PU learning. 
  "*" denote oracle early stopping as defined in \secref{sec:exp}.
  Results reported by aggregating over 10 epochs 
  and 3 seeds.
  }\label{table:classification_error_bar}
\end{table}

\newpage
\subsection{Experiments on UCI dataset} \label{ap:UCI}

In this section, we will present results on 5 UCI datasets. 

\begin{center}
  
    \begin{table}[H] 
        \centering
        \tabcolsep=0.12cm
        \renewcommand{\arraystretch}{1.2}
        \begin{tabular}{@{}*{5}{c}@{}}
        \toprule
        
        Dataset &  \multicolumn{2}{c}{\#Positives} & \multicolumn{2}{c}{\#Unlabeled} \\ 
        & Train & Val & Train & Val \\ [0.5ex] 
        \midrule
        concrete &  162 & 162 & 81 & 81 \\
        mushroom & 1304 & 1304 & 652 & 652 \\
        landsat &  946 & 946 & 472 & 472 \\
        pageblock  &185 &185 & 92 & 92 \\
        spambase &  604 & 604 & 302 & 302 \\
        \bottomrule 
        \end{tabular}
    \end{table}    
\end{center}

We train a MLP with 2 hidden layers each with $512$ units. The PyTorch code for 4-layer MLP is as follows: 

\texttt{ nn.Sequential(nn.Flatten(), \\
\tab        nn.Linear(input\_dim, 512, bias=True),\\
\tab        nn.ReLU(),\\
\tab        nn.Linear(512, 512, bias=True),\\
\tab        nn.ReLU(),\\
\tab        nn.Linear(512, 2, bias=True),\\
\tab        )}

Similar to vision datasets and architectures, we do cross entropy loss minimization with SGD optimizer with momentum $0.9$ and learning rate $0.1$. For nnPU and uPU, we minimize sigmoid loss with ADAM optimizer with learning rate $0.0001$ as advised in its original paper. For all methods,  we fix the weight decay param at $0.0005$. 

\begin{table}[h]
    \centering
    \small
    \tabcolsep=0.12cm
    \renewcommand{\arraystretch}{1.2}
    \begin{tabular}{@{}*{9}{c}@{}}
    \toprule
    Dataset  & \thead{(TED)$^n$} &  \thead{BBE$^*$} & \thead{DEDPUL$^*$} & \thead{EN$^*$} & \thead{KM2} & \thead{TiCE} \\
    \midrule
    concrete & $\mathbf{0.071}$ & $0.152$ & $0.176$ & $0.239$ & $0.099$ & $0.268$ \\
    mushroom & $\mathbf{0.001}$ & $0.015$ & $0.014$ & $0.013$ & $0.038$ & $0.069$ \\
    landsat & ${0.022}$ & $0.021$ & $\mathbf{0.012}$ & $0.080$ & $0.037$ & $0.027$ \\
    pageblock & $\mathbf{0.007}$ & $0.066$ & $0.041$ & $0.135$ & $0.008$ & $0.298$ \\
    spambase & $\mathbf{0.006}$ & $0.047$ & $0.077$ & $0.127$ & $0.062$ & $0.276$ \\
    \bottomrule 
    \end{tabular}  
    \vspace{5pt}
    \caption{Absolute estimation error when $\alpha$ is 0.5.
     "*" denote oracle early stopping as defined in \secref{sec:exp}.
    Results reported by aggregating absolute error over 10 epochs.
    }\label{table:uci_MPE}
  \end{table}

  \begin{table}[h]
    \centering
    \small
    \tabcolsep=0.12cm
    \renewcommand{\arraystretch}{1.2}
    \begin{tabular}{@{}*{7}{c}@{}}
    \toprule
    Dataset   & \thead{(TED)$^n$ \\(unknown $\alpha$)}  & \thead{CVuO\\(known $\alpha$)} & \thead{PvU$^*$ \\(known $\alpha$)} & \thead{DEDPUL$^*$ \\(unknown $\alpha$)}  & \thead{nnPU \\(known $\alpha$)} & \thead{uPU$^*$ \\(known $\alpha$)} \\
    \midrule
    concrete & $\mathbf{86.3}$ & $80.1$ & $83.1$ & $83.7$& $83.2$  & $84.4$ \\
    mushroom & $96.4$ & $96.3$ & $\mathbf{98.7}$ & $\mathbf{98.7}$ & $97.5$ & $93.9$ \\
    landsat & $\mathbf{93.8}$ & $93.1$ & $93.4$ & $92.4$ & $92.9$ & $92.3$ \\
    pageblock & $\mathbf{95.7}$ & $\mathbf{95.7}$ & $95.1$ & $94.5$ & $93.9$ & $93.9$ \\
    spambase & $\mathbf{89.4}$ & $88.1$ & $89.2$ & $86.8$ & $88.5$ & $87.7$ \\
    \bottomrule 
    \end{tabular}  
    \vspace{5pt}
    \caption{Accuracy for PvN classification with PU learning. 
    "*" denote oracle early stopping as defined in \secref{sec:exp}.
    Results reported by aggregating aggregating over 10 epochs.  
    }\label{table:uci_classification}
  \end{table}
  
  On $4$ out of $5$ UCI datasets, our proposed methods are better than the best performing alternatives (\tabref{table:uci_MPE} and \tabref{table:uci_classification}).   

\subsection{Experiments on MNIST Overlap} \label{ap:mnist_overlap}
Similar to binarized MNIST, we create a new dataset called MNIST Overlap, where the positive class contains digits from $0$ to $7$ and the negative class contains digits from $3$ to $9$. This creates a dataset with overlap between positive and negative support. Note that while the supports overlap, we sample images from the overlap classes with replacement, and hence, in absence of duplicates in the dataset, exact same images don't appear both in positive and negative subsets.

We train MLP with the same hyperparameters as before. Our findings in \tabref{table:mnist_MPE} and \tabref{table:mnist_classification} highlight superior performance of the proposed approaches in the cases of support overlap. 

\begin{table}[h]
    \centering
    \small
    \tabcolsep=0.12cm
    \renewcommand{\arraystretch}{1.2}
    \begin{tabular}{@{}*{9}{c}@{}}
    \toprule
    Dataset  & \thead{(TED)$^n$} &  \thead{BBE$^*$} & \thead{DEDPUL$^*$} & \thead{EN$^*$} & \thead{KM2} & \thead{TiCE} \\
    \midrule
    MNIST Overlap & $\mathbf{0.035}$ & $0.100$ & $0.104$ & $0.196$ & $0.099$ & $0.074$ \\
    \bottomrule 
    \end{tabular}  
    \vspace{5pt}
    \caption{Absolute estimation error when $\alpha$ is 0.5.
     "*" denote oracle early stopping as defined in \secref{sec:exp}.
    Results reported by aggregating absolute error over 10 epochs.
    }\label{table:mnist_MPE}
  \end{table}

\begin{table}[h]
    \centering
    \small
    \tabcolsep=0.12cm
    \renewcommand{\arraystretch}{1.2}
    \begin{tabular}{@{}*{7}{c}@{}}
    \toprule
    Dataset   & \thead{(TED)$^n$ \\(unknown $\alpha$)}  & \thead{CVuO\\(known $\alpha$)} & \thead{PvU$^*$ \\(known $\alpha$)} & \thead{DEDPUL$^*$ \\(unknown $\alpha$)}  & \thead{nnPU \\(known $\alpha$)} & \thead{uPU$^*$ \\(known $\alpha$)} \\
    \midrule
    MNIST Overlap & $\mathbf{79.0}$ & $78.4$ & $77.4$ & $77.5$& $78.6$  & $78.8$ \\
    \bottomrule 
    \end{tabular}  
    \vspace{5pt}
    \caption{Accuracy for PvN classification with PU learning. 
    "*" denote oracle early stopping as defined in \secref{sec:exp}.
    Results reported by aggregating aggregating over 10 epochs.  
    }\label{table:mnist_classification}
  \end{table}

%% file: sections/new_proof.tex
{\update 
\begin{proof}[Proof of \thmref{thm:main_MPE}] 
    The main idea of the proof is to use the confidence bound derived in \lemref{lem:ucb} at $\wh c$ and use the fact that $\wh c$ minimizes the upper confidence bound. The proof is split into two parts. First, we derive a lower bound on $\wh q_p(\wh c)$ and next, we use the obtained lower bound to derive confidence bound on $\wh \alpha$. All the statements in the proof simultaneously hold with probability $1-\delta$. Recall,
    \begin{align}
        \wh c &\defeq \argmin_{c \in [0,1]}  \frac{\wh q_u(c)}{\wh q_p(c)}  + \frac{1}{\wh q_p(c)}\left( \sqrt{\frac{\log(4/\delta)}{2 n_u}} + (1+\gamma)\sqrt{\frac{\log(4/\delta)}{2n_p}}\right) \qquad \text{and} \\
        \wh \alpha &\defeq \frac{\wh q_u(\wh c)}{\wh q_p(\wh c)}\,.
    \end{align}
    Moreover, 
    \begin{align}
        c^* \defeq \argmin_{c \in [0,1]} \frac{q_u(c)}{q_p(c)} \qquad \text{and}\qquad \alpha^* \defeq \frac{q_u(c^*)}{q_p(c^*)}\,.
    \end{align}
    \textbf{Part 1:}  We establish lower bound on $\wh q_p(\wh c)$. Consider $c^\prime \in [0,1]$ such that $\wh q_p(c^\prime) = \frac{\gamma}{2 + \gamma} \wh q_p(c^*)$. We will now show that \algoref{alg:MPE_PU} will select $\wh c < c^\prime$. For any $c \in [0,1]$, we have with with probability $1-\delta$,  
    \begin{align}
        \wh q_p(c) - \sqrt{\frac{\log(4/\delta)}{2n_p}} \le q_p(c) \qquad \text{and} \qquad q_u(c) - \sqrt{\frac{\log(4/\delta)}{2n_u}} \le \wh q_u(c) \,.
    \end{align}
    Since $ \frac{q_u(c^*)}{q_p(c^*)} \le \frac{q_u(c)}{q_p(c)}$, we have 
    \begin{align}
        \wh q_u(c) \ge q_p(c) \frac{q_u(c^*)}{q_p(c^*)}  - \sqrt{\frac{\log(4/\delta)}{2n_u}} \ge \left(  \wh q_p(c) - \sqrt{\frac{\log(4/\delta)}{2n_p}}  \right) \frac{q_u(c^*)}{q_p(c^*)}  - \sqrt{\frac{\log(4/\delta)}{2n_u}} \,.
    \end{align}
    Therefore, at $c$ we have 
    \begin{align}
        \frac{\wh q_u(c)}{\wh q_p(c)}  &\ge \alpha^* -  \frac{1}{\wh q_p(c)}\left( \sqrt{\frac{\log(4/\delta)}{2n_u}} + \frac{q_u(c^*)}{q_p(c^*)}\sqrt{\frac{\log(4/\delta)}{2n_p}}\right) \,.
    \end{align}

    Using \lemref{lem:ucb} at $c^*$, we have 
    \begin{align}
        \frac{\wh q_u(c)}{\wh q_p(c)}  &\ge \frac{\wh q_u(c^*)}{\wh q_p(c^*)} - \left(\frac{1}{\wh q_p(c^*)} +  \frac{1}{\wh q_p(c)}\right)\left( \sqrt{\frac{\log(4/\delta)}{2n_u}} + \frac{q_u(c^*)}{q_p(c^*)}\sqrt{\frac{\log(4/\delta)}{2n_p}}\right) \\
        &\ge \frac{\wh q_u(c^*)}{\wh q_p(c^*)} - \left(\frac{1}{\wh q_p(c^*)} +  \frac{1}{\wh q_p(c)}\right)\left( \sqrt{\frac{\log(4/\delta)}{2n_u}} + \sqrt{\frac{\log(4/\delta)}{2n_p}}\right) \,,
    \end{align}
    where the last inequality follows from the fact that $\alpha^* = \frac{q_u(c^*)}{q_p(c^*)} \le 1$. Furthermore, the upper confidence bound at $c$ is lower bound as follows: 
    \begin{align}
        \frac{\wh q_u(c)}{\wh q_p(c)} + &\frac{1+\gamma}{\wh q_p(c)} \left( \sqrt{\frac{\log(4/\delta)}{2n_u}} + \sqrt{\frac{\log(4/\delta)}{2n_p}}\right) \\ &\ge \frac{\wh q_u(c^*)}{\wh q_p(c^*)} + \left(\frac{1 + \gamma}{\wh q_p(c)} - \frac{1}{\wh q_p(c^*)} - \frac{1}{\wh q_p(c)}\right)\left( \sqrt{\frac{\log(4/\delta)}{2n_u}} + \sqrt{\frac{\log(4/\delta)}{2n_p}}\right) \\  
        &= \frac{\wh q_u(c^*)}{\wh q_p(c^*)} +  \left(\frac{\gamma}{\wh q_p(c)} - \frac{1}{\wh q_p(c^*)} \right)\left( \sqrt{\frac{\log(4/\delta)}{2n_u}} + \sqrt{\frac{\log(4/\delta)}{2n_p}}\right) \label{eq:lower_ucb}
    \end{align}
    Using \eqref{eq:lower_ucb} at $c = c^\prime$, we have the following lower bound on ucb at $c^\prime$: 
    \begin{align}
        \frac{\wh q_u(c^\prime)}{\wh q_p(c^\prime)} + &\frac{1+\gamma}{\wh q_p(c^\prime)} \left( \sqrt{\frac{\log(4/\delta)}{2n_u}} + \sqrt{\frac{\log(4/\delta)}{2n_p}}\right) \\ 
        &\ge \frac{\wh q_u(c^*)}{\wh q_p(c^*)} + \frac{1 + \gamma}{\wh q_p(c^*)}\left( \sqrt{\frac{\log(4/\delta)}{2n_u}} + \sqrt{\frac{\log(4/\delta)}{2n_p}}\right) \,,
    \end{align}
    
    Moreover from \eqref{eq:lower_ucb}, we also have that the lower bound on ucb at $c \ge c^\prime$ is strictly greater than the lower bound on ucb at $c^\prime$. Using definition of $\wh c$, we have 
    \begin{align}
         \frac{\wh q_u(c^*)}{\wh q_p(c^*)} &+ \frac{1 + \gamma}{\wh q_p(c^*)}\left( \sqrt{\frac{\log(4/\delta)}{2n_u}} + \sqrt{\frac{\log(4/\delta)}{2n_p}}\right) \\ 
        &\ge \frac{\wh q_u(\wh c)}{\wh q_p(\wh c)} + \frac{1 + \gamma}{\wh q_p(\wh c)}\left( \sqrt{\frac{\log(4/\delta)}{2n_u}} + \sqrt{\frac{\log(4/\delta)}{2n_p}}\right) \,,
    \end{align}
    and hence 
    \begin{align}
        \wh c \le c^\prime \,.
    \end{align}

    \textbf{Part 2:} We now establish an upper and lower bound on $\wh \alpha$. We start with upper confidence bound on $\wh \alpha$. By definition of $\wh c$, we have
    \begin{align}
        \frac{\wh q_u(\wh c)}{\wh q_p(\wh c)} + \frac{1 + \gamma}{\wh q_p(\wh c)} & \left( \sqrt{\frac{\log(4/\delta)}{2n_u}} + \sqrt{\frac{\log(4/\delta)}{2n_p}}\right)   \\
        & \le \min_{c \in [0,1]} \left[ \frac{\wh q_u(c)}{\wh q_p(c)} + \frac{1 + \gamma}{\wh q_p(c)}\left( \sqrt{\frac{\log(4/\delta)}{2n_u}} + \sqrt{\frac{\log(4/\delta)}{2n_p}}\right) \right] \\ 
        & \le \, \frac{\wh q_u(c^*)}{\wh q_p(c^*)}  + \frac{1+\gamma}{\wh q_p(c^*)}\left( \sqrt{\frac{\log(4/\delta)}{2n_u}} + \sqrt{\frac{\log(4/\delta)}{2n_p}}\right)  
        \,. \label{eq:bound_step1} \numberthis 
    \end{align}
    Using \lemref{lem:ucb} at $c^*$, we get 
    \begin{align*}
        \frac{\wh q_u(c^*)}{\wh q_p(c^*)} &\le \frac{q_u(c^*)}{q_p(c^*)} + \frac{1}{\wh q_p(c^*)}\left( \sqrt{\frac{\log(4/\delta)}{2n_u}} + \frac{q_u(c^*)}{q_p(c^*)}\sqrt{\frac{\log(4/\delta)}{2n_p}}\right) \\
        &= \alpha^* + \frac{1}{\wh q_p(c^*)}\left( \sqrt{\frac{\log(4/\delta)}{2n_u}} + \alpha^*\sqrt{\frac{\log(4/\delta)}{2n_p}}\right) \,.\numberthis \label{eq:bound_step2}
    \end{align*} 
    Combining \eqref{eq:bound_step1} and \eqref{eq:bound_step2}, we get 
    \begin{align}
        \wh \alpha = \frac{\wh q_u(\wh c)}{\wh q_p(\wh c)} \le \alpha^* + \frac{2+\gamma}{\wh q_p(c^*)}\left( \sqrt{\frac{\log(4/\delta)}{2n_u}} + \sqrt{\frac{\log(4/\delta)}{2n_p}}\right)  \,.
    \end{align}
    
    Using DKW inequality on $\wh q_p(c^*)$, we have $\wh q_p(c^*) \ge q_p(c^*) - \sqrt{\frac{\log(4/\delta)}{2n_p}}$. Assuming $n_p \ge \frac{2\log(4/\delta)}{q_p^2(c^*)}$, we get $\wh q_p(c^*) \le q_p(c^*)/ 2$ and hence, 
    \begin{align}
        \wh \alpha  \le \alpha^* + \frac{4+2\gamma}{q_p(c^*)}\left( \sqrt{\frac{\log(4/\delta)}{2n_u}} + \sqrt{\frac{\log(4/\delta)}{2n_p}}\right)  \,. \label{eq:upper_bound}
    \end{align}
    
    Finally, we now derive a lower bound on $\wh \alpha$. From \lemref{lem:ucb}, we have the following inequality at $\wh c$ 
    \begin{align}
          \frac{q_u(\wh c)}{q_p( \wh c)} \le \frac{\wh q_u(\wh c)}{\wh q_p(\wh c)} + \frac{1}{\wh q_p(\wh c)}\left( \sqrt{\frac{\log(4/\delta)}{2n_u}} + \frac{q_u(\wh c)}{q_p(\wh c)}\sqrt{\frac{\log(4/\delta)}{2n_p}}\right) \,. \label{eq:lem1}
    \end{align}
    Since $\alpha^* \le \frac{q_u(\wh c)}{q_p( \wh c)} $, we have 
    \begin{align}
        \alpha^* \le \frac{q_u(\wh c)}{q_p( \wh c)} \le \frac{\wh q_u(\wh c)}{\wh q_p(\wh c)} + \frac{1}{\wh q_p(\wh c)}\left( \sqrt{\frac{\log(4/\delta)}{2n_u}} + \frac{q_u(\wh c)}{q_p(\wh c)}\sqrt{\frac{\log(4/\delta)}{2n_p}}\right) \,. \label{eq:lower_bound_step1}
    \end{align}
    
    Using \eqref{eq:upper_bound}, we obtain a very loose upper bound on $\frac{\wh q_u(\wh c)}{\wh q_p(\wh c)}$. Assuming $\min(n_p, n_u) \ge \frac{2\log(4/\delta)}{q_p^2(c^*)}$, we have $\frac{\wh q_u(\wh c)}{\wh q_p(\wh c)} \le \alpha^* + 4 + 2\gamma \le 5 + 2\gamma$. Using this in \eqref{eq:lower_bound_step1}, we have 
    \begin{align}
        \alpha^* \le \frac{\wh q_u(\wh c)}{\wh q_p(\wh c)} + \frac{1}{\wh q_p(\wh c)}\left( \sqrt{\frac{\log(4/\delta)}{2n_u}} + (5+2\gamma)\sqrt{\frac{\log(4/\delta)}{2n_p}}\right) \,. 
    \end{align}
    Moreover, as $\wh c \ge c^\prime$,  we have $\wh q_p(\wh c) \ge \frac{\gamma}{2 + \gamma} \wh q_p(c^*)$ and hence, 
    \begin{align}
        \alpha^* - \frac{\gamma + 2}{\gamma \wh q_p(c^*)}\left( \sqrt{\frac{\log(4/\delta)}{2n_u}} + (5+2\gamma)\sqrt{\frac{\log(4/\delta)}{2n_p}}\right) \le  \frac{\wh q_u(\wh c)}{\wh q_p(\wh c)} = \wh \alpha \,. 
    \end{align}
    As we assume $n_p \ge \frac{2\log(4/\delta)}{q_p^2(c^*)}$, we have $\wh q_p(c^*) \le q_p(c^*)/ 2$, which implies the following lower bound on $\alpha$: 
    \begin{align}
        \alpha^* - \frac{2\gamma + 4}{\gamma q_p(c^*)}\left( \sqrt{\frac{\log(4/\delta)}{2n_u}} + (5+2\gamma)\sqrt{\frac{\log(4/\delta)}{2n_p}}\right) \le \wh \alpha \,. 
    \end{align}
\end{proof}
}

\begin{proof}[Proof of \corollaryref{corollary:MPE_final}]
    Note that since $\alpha \le \alpha^*$, the lower bound remains the same as in \thmref{thm:main_MPE}.  
    For upper bound, plugging in $q_u(c) = \alpha q_p(c) + (1-\alpha) q_n(c)$, we have $\alpha^* = \alpha + (1-\alpha) q_n(c^*)/q_p(c^*)$ and hence, the required upper bound.  
\end{proof}
{
\update
\subsection{Note on $\gamma$ in \algoref{alg:MPE_PU}} \label{app:gamma_disc}
We multiply  
the upper bound in \lemref{lem:ucb} 
to establish lower bound on $\wh q_p(\wh c)$.
Otherwise, in an extreme case, 
with $\gamma =0$, 
\algoref{alg:MPE_PU}
can select $\wh c$  
with arbitrarily low $\wh q_p(\wh c)$ 
($\ll q_p(c^*)$)
and hence poor concentration 
guarantee to the 
true mixture proportion. 
However, with a 
small positive $\gamma$,  
we can obtain 
lower bound on $\wh q_p(\wh c)$ 
and hence tight guarantees 
on the ratio estimate 
($\wh q_u(\wh c)/ \wh q_p(\wh c)$)
in \thmref{thm:main_MPE}. 

In our experiments, we choose $\gamma = 0.01$. 
However, we didn't observe 
any (significant) differences 
in mixture proportion 
estimation even with $\gamma=0$. 
implying that 
we never observe $\wh q_p(\wh c)$ 
taking arbitrarily small values 
in our experiments. 
}